\documentclass[english]{article}

\usepackage{geometry}
\usepackage[T1]{fontenc}
\usepackage[latin9]{inputenc}
\usepackage{bm}
\usepackage{amsmath}
\usepackage{amssymb}
\usepackage[unicode=true,
 bookmarks=false,
 breaklinks=false,pdfborder={0 0 1},colorlinks=false]
 {hyperref}
\hypersetup{
 colorlinks,citecolor=blue,filecolor=blue,linkcolor=blue,urlcolor=blue}

\makeatletter
\usepackage{amsthm}
\usepackage{cite}  
\usepackage{comment}
\usepackage{natbib}
\usepackage{booktabs}

\usepackage{graphicx}

\usepackage{float}
\usepackage{multirow}
 
\usepackage{dsfont}

\usepackage{color}
\definecolor{yxc}{RGB}{255,0,0}
\definecolor{yjc}{RGB}{125,0,0}
\definecolor{ytw}{RGB}{0,125,0}

\newcommand{\yxc}[1]{\textcolor{yxc}{[YXC: #1]}}

\allowdisplaybreaks

\DeclareMathOperator{\ind}{\mathds{1}}  

\newcommand{\defn}{:=}
\newcommand{\MDP}{\mathcal{M}}
\newcommand{\gap}{\omega}
\newcommand{\net}[1]{\mathcal{N}_{#1}}

\newcommand{\one}{\bm 1}

\newcommand{\Ind}{\bm I}

\newcommand{\mprob}{\mathbb{P}}
\newcommand{\newnoise}{\varsigma}
\newcommand{\picon}{\widehat{\pi}_{\mathrm{c}}}

\newcommand{\bPi}{{\bf \Pi}}
\newcommand{\Pq}{{\bm P}^\pi}
\newcommand{\e}{{\bm e}}
\newcommand{\real}{\mathbb{R}}
\newcommand{\Var}{\mathsf{Var}}

\newcommand{\cS}{\mathcal{S}}
\newcommand{\cA}{\mathcal{A}}
\newcommand{\ahat}{\widehat{a}}

\newcommand{\Qhat}{\widehat{\bm Q}}
\newcommand{\Qstar}{ {\bm Q}^\star}
\newcommand{\Qhatstar}{\widehat{\bm Q}^\star}

\newcommand{\Vstar}{{\bm V}^\star}
\newcommand{\Vhat}{\widehat{\bm V}}
\newcommand{\Vhatstar}{\widehat{\bm V}^\star}

\newcommand{\Phat}{\widehat{\bm{P}}}
\newcommand{\abPbm}{\bm{P}^{\mathsf{abs}}}
\newcommand{\bP}{\bm{P}}

\newcommand{\Pv}{{\bm P}_{\pi}}

\newcommand{\Phatv}{{\widehat{\bm P}}_{\pi}}

\newcommand{\pistar}{{\pi}^\star}

\newcommand{\pihatstar}{\widehat{\pi}^\star}

\newcommand{\rewabs}{r^{\mathsf{abs}}}

\newcommand{\br}{\bm r}
\newcommand{\mymid}{\,|\,}

\title{Breaking the Sample Size Barrier in \\ Model-Based Reinforcement Learning with a Generative Model}

\author{Gen Li\thanks{Department of Statistics and Data Science, Wharton School, University of Pennsylvania, Philadelphia, PA 19104, USA.} \\
UPenn    \\
	\and
	Yuting Wei\footnotemark[1]\\
	UPenn\\
	\and
	Yuejie Chi\thanks{Department of Electrical and Computer Engineering, Carnegie Mellon University, Pittsburgh, PA 15213, USA.}\\
	CMU\\
	\and
	Yuxin Chen\footnotemark[1] \\
 UPenn  \\
	}
\date{May 2020;~~ Revised: September 2022}

\makeatother

\begin{document}

\theoremstyle{plain} \newtheorem{lemma}{\textbf{Lemma}}\newtheorem{proposition}{\textbf{Proposition}}\newtheorem{theorem}{\textbf{Theorem}}

\theoremstyle{remark}\newtheorem{remark}{\textbf{Remark}}

\maketitle

\begin{abstract}

This paper is concerned with the sample efficiency of reinforcement learning, assuming access to a generative model (or simulator).  
We first consider $\gamma$-discounted infinite-horizon Markov decision processes (MDPs) with state space $\mathcal{S}$ and action space $\mathcal{A}$. Despite a number of prior works tackling this problem, a complete picture of the trade-offs between sample complexity and statistical accuracy is yet to be determined. In particular, all prior results suffer from a severe sample size barrier, in the sense that their claimed statistical guarantees hold only when the sample size exceeds at least $\frac{|\mathcal{S}||\mathcal{A}|}{(1-\gamma)^2}$. 
The current paper overcomes this barrier by certifying the minimax optimality of two algorithms ---  a {\em perturbed} model-based algorithm and a {\em conservative} model-based algorithm --- as soon as the sample size exceeds the order of $\frac{|\mathcal{S}||\mathcal{A}|}{1-\gamma}$ (modulo some log factor). Moving beyond infinite-horizon MDPs, we further study time-inhomogeneous finite-horizon MDPs, and prove that a plain model-based planning algorithm suffices to achieve minimax-optimal sample complexity given any target accuracy level. 
To the best of our knowledge, this work delivers the first minimax-optimal guarantees that accommodate the entire range of sample sizes (beyond which finding a meaningful policy is information theoretically infeasible).  



\end{abstract}


\tableofcontents


\section{Introduction}

Reinforcement learning (RL) \citep{sutton2018reinforcement,szepesvari2010algorithms}, which is frequently modeled as learning and decision making in a Markov decision process (MDP), is garnering growing interest in recent years due to its remarkable success in practice. A core objective of RL is to search for a policy --- based on a collection of noisy data samples --- that approximately maximizes expected cumulative rewards in an MDP, without direct access to a precise description of the underlying model.\footnote{Here and throughout, the ``model'' refers to the transition kernel and the rewards of the MDP taken collectively.}    
In contemporary applications, it is increasingly more common to encounter environments with prohibitively large state and action space, thus exacerbating the challenge of collecting enough samples to learn the model.  To enable faithful policy learning in the sample-starved regime (i.e.~the regime where the model complexity overwhelms the sample size),  it is crucial to obtain a quantitative picture of the fundamental trade-off between sample complexity and statistical accuracy, and to design efficient algorithms that provably achieve the optimal trade-off.

Broadly speaking, there are at least two common algorithmic approaches: a model-based approach and a model-free one. The model-based approach decouples model estimation and policy learning tasks; more specifically, one first estimates the unknown model using the data samples in hand, and then leverages the fitted model to 
perform planning --- a task that can be accomplished by resorting to Bellman's principle of optimality \citep{bellman1952theory}. A notable advantage of model-based algorithms is their flexibility: the learned model can be adapted to perform new ad-hoc tasks without revisiting the data samples.  In comparison, the model-free approach attempts to compute the optimal policy (and the optimal value function) without learning the model explicitly, which lends itself well to scenarios when a model is difficult to estimate or changes on the fly. Characterizing the sample efficiency of both approaches has been the focal point of a large body of recent works, e.g.~\citet{azar2013minimax,kearns1999finite, sidford2018near,sidford2018variance,li2021q,tu2018gap,jin2018q,agarwal2019optimality,wainwright2019stochastic,wainwright2019variance}.   

In this paper, we pursue a comprehensive understanding of model-based RL, given access to a generative model --- that is, a simulator that produces samples based on the transition kernel of the true MDP for each state-action pair \citep{kearns1999finite,kakade2003sample}.   To allow for more precise discussions, we first look at an infinite-horizon discounted MDP with  state space $\cS$, action space $\cA$ and discount factor $0<\gamma<1$,
and pay particular attention to the scenarios where the sizes of the state/action spaces and the effective horizon $\frac{1}{1-\gamma}$ are all quite large. We obtain $N$ samples per state-action pair by querying the generative model.  For an {\em arbitrary} target accuracy level $\varepsilon>0$, a desired model-based planning algorithm should return an $\varepsilon$-optimal policy with a minimal number of calls to the generative model. Particular emphasis is placed on the sub-linear sampling scenario, in which the total sample size is smaller than the total number $|\cS|^2|\cA|$ of model parameters (so that it is in general infeasible to estimate the model accurately).

\paragraph{Motivation: sample size barriers.}  Several prior works were dedicated to investigating model-based RL for $\gamma$-discounted infinite-horizon MDPs with a generative model, which uncovered the minimax optimality of this approach for an already wide regime \citep{azar2013minimax,agarwal2019optimality}. However, the results therein often suffered from  a sample complexity barrier that prevents us from obtaining a complete trade-off curve between sample complexity and statistical accuracy. For instance, the state-of-the-art result \citet{agarwal2019optimality} required the total sample size to at least exceed $\frac{|\cS||\cA|}{(1-\gamma)^2}$ (up to some log factor), thus restricting the validity of the theory for broader contexts. In truth, this is not merely an issue for model-based planning;
the same barrier already showed up when analyzing the simpler task of model-based policy evaluation \citep{pananjady2019value,agarwal2019optimality}. Furthermore, an even more severe barrier emerged in prior theory for model-free methods; for instance, \citet{sidford2018variance,wainwright2019variance} required the sample size to exceed $\frac{|\cS||\cA|}{(1-\gamma)^3}$ modulo some log factor. In stark contrast, however, no lower bounds developed thus far preclude us from attaining reasonable statistical accuracy when going below the aforementioned sample complexity barriers, thus resulting in a gap between upper and lower bounds in this sample-starved regime.  Noteworthily, such a sample size barrier is not only present for discounted infinite-horizon MDPs; the situation is similar for finite-horizon MDPs \citep{yin2021near}.









\paragraph{Our contributions.} 
The current paper seeks to achieve optimal sample complexity even below the aforementioned sample size barrier. For $\gamma$-discounted infinite-horizon MDPs, we propose two model-based algorithms: 
(i) {\em perturbed model-based planning}: which performs planning based on an empirical MDP learned from samples with mild random {\em reward perturbation}; 
and (ii) {\em conservative model-based planning}: which computes approximately optimal policies for the empirical MDP without reward perturbation. These two proposed algorithms provably find an $\varepsilon$-optimal policy with an order of $\frac{|\cS||\cA|}{(1-\gamma)^3\varepsilon^2}$ samples (up to log factor), thereby matching the minimax lower bound \citep{azar2013minimax}. Our result accommodates the full range of accuracy level $\varepsilon$ (namely, $\varepsilon\in (0,\frac{1}{1-\gamma}]$), thus unveiling the minimaxity of our algorithms as soon as the sample size exceeds $\frac{|\cS||\cA|}{1-\gamma}$ (modulo some log factor). Encouragingly, this covers the {\em full} range of sample sizes that enable one to find a policy strictly better than a random guess.  See Table~\ref{tab:prior-work} for detailed comparisons with prior literature. Along the way, we also derive minimax-optimal statistical guarantees for policy evaluation, which strengthen state-of-the-art results by broadening the applicable sample size range. 

Moving beyond discounted infinite-horizon MDPs, we further characterize the sample efficiency of model-based planning for time-inhomogeneous finite-horizon MDPs, 
which provably achieves minimax-optimal sample complexity as well for the full range of target accuracy levels \citep{domingues2021episodic}. No reward perturbation or conservative action selection is needed for this finite-horizon scenario. See Table~\ref{tab:prior-work-finite} for detailed comparisons with prior literature. 

On the technical side, 
our theory for infinite-horizon MDPs is established upon a novel combination of several key ideas: (1) a  high-order expansion of the estimation error for value functions, coupled with fine-grained analysis for each term in the expansion; (2) the construction of auxiliary leave-one-out type (state-action-absorbing) MDPs --- motivated by \cite{agarwal2019optimality} --- that help decouple the complicated statistical dependency between the empirically optimal policy (as opposed to value functions) and data samples; (3) a tie-breaking argument guaranteeing that the empirically optimal policy is sufficiently separated from all other policies under reward perturbation. 
The case with finite-horizon MDPs is also established based on certain high-order expansion of the value estimation errors, in addition to careful variance control for the terms in the expansion.

\newcommand{\topsepremove}{\aboverulesep = 0mm \belowrulesep = 0mm} \topsepremove
 
\begin{table}[H]

\begin{center}

\begin{tabular}{c|c|c|c}
\toprule
Algorithm & Sample size range & \multicolumn{1}{c|}{Sample complexity} & $\varepsilon$-range $\vphantom{\frac{1^{7}}{1^{7^3}}}$ \tabularnewline
\toprule
Phased Q-learning & \multirow{2}{*}{$\big[\frac{|\mathcal{S}||\mathcal{A}|}{(1-\gamma)^{5}},\infty)$} & \multirow{2}{*}{$\frac{|\mathcal{S}||\mathcal{A}|}{(1-\gamma)^{7}\varepsilon^{2}}$\vphantom{$\frac{1^{7^{7}}}{1^{7^{7^{7^{7}}}}}$}\hspace{-0.4em}} & \multirow{2}{*}{$(0,\frac{1}{1-\gamma}]$}\tabularnewline
\cite{kearns1999finite} &  &  & \tabularnewline
\hline 
Empirical QVI & \multirow{2}{*}{$\big[\frac{|\mathcal{S}|^{2}|\mathcal{A}|}{(1-\gamma)^{2}},\infty)$} & \multirow{2}{*}{$\frac{|\mathcal{S}||\mathcal{A}|}{(1-\gamma)^{3}\varepsilon^{2}}$
\vphantom{$\frac{1^{7^{7}}}{1^{7^{7^{7^{7}}}}}$}\hspace{-0.4em}} & \multirow{2}{*}{$(0,\frac{1}{\sqrt{(1-\gamma)|\mathcal{S}|}}]$}\tabularnewline
\cite{azar2013minimax} &  &  & \tabularnewline
\hline 
Sublinear randomized value iteration & \multirow{2}{*}{$\big[\frac{|\mathcal{S}||\mathcal{A}|}{(1-\gamma)^{2}},\infty)$} & \multirow{2}{*}{$\frac{|\mathcal{S}||\mathcal{A}|}{(1-\gamma)^{4}\varepsilon^{2}}$
	\vphantom{$\frac{1^{7^{7}}}{1^{7^{7^{7^{7}}}}}$}\hspace{-0.4em}} & \multirow{2}{*}{$\big(0, \frac{1}{1-\gamma}\big]$}\tabularnewline
\cite{sidford2018variance} &  &  & \tabularnewline
\hline 
Variance-reduced QVI & \multirow{2}{*}{$\big[\frac{|\mathcal{S}||\mathcal{A}|}{(1-\gamma)^{3}},\infty)$} & \multirow{2}{*}{$\frac{|\mathcal{S}||\mathcal{A}|}{(1-\gamma)^{3}\varepsilon^{2}}$
\vphantom{$\frac{1^{7^{7}}}{1^{7^{7^{7^{7}}}}}$}\hspace{-0.4em}} & \multirow{2}{*}{$(0,1]$}\tabularnewline
\cite{sidford2018near} &  &  & \tabularnewline
\hline 
Randomized primal-dual method & \multirow{2}{*}{$\big[\frac{|\mathcal{S}||\mathcal{A}|}{(1-\gamma)^{2}},\infty)$} & \multirow{2}{*}{$\frac{|\mathcal{S}||\mathcal{A}|}{(1-\gamma)^{4}\varepsilon^{2}}$
\vphantom{$\frac{1^{7^{7}}}{1^{7^{7^{7^{7}}}}}$}\hspace{-0.4em}} & \multirow{2}{*}{$(0,\frac{1}{1-\gamma}]$}\tabularnewline
\cite{wang2017randomized} &  &  & \tabularnewline
\hline 
Empirical MDP + planning & \multirow{2}{*}{$\big[\frac{|\mathcal{S}||\mathcal{A}|}{(1-\gamma)^{2}},\infty)$} & \multirow{2}{*}{$\frac{|\mathcal{S}||\mathcal{A}|}{(1-\gamma)^{3}\varepsilon^{2}}$
\vphantom{$\frac{1^{7^{7}}}{1^{7^{7^{7^{7}}}}}$}\hspace{-0.4em}} & \multirow{2}{*}{$(0,\frac{1}{\sqrt{1-\gamma}}]$}\tabularnewline
\cite{agarwal2019optimality} &  &  & \tabularnewline
\hline 
{\em Perturbed} empirical MDP + planning & \multirow{2}{*}{$\big[\frac{|\mathcal{S}||\mathcal{A}|}{1-\gamma},\infty)$} & \multirow{2}{*}{$\frac{|\mathcal{S}||\mathcal{A}|}{(1-\gamma)^{3}\varepsilon^{2}}$
\vphantom{$\frac{1^{7^{7}}}{1^{7^{7^{7^{7}}}}}$}\hspace{-0.4em}} & \multirow{2}{*}{$(0,\frac{1}{1-\gamma}]$}\tabularnewline
\textbf{This paper} &  &  & \tabularnewline
\hline 
	Empirical MDP + {\em conservative} planning & \multirow{2}{*}{$\big[\frac{|\mathcal{S}||\mathcal{A}|}{1-\gamma},\infty)$} & \multirow{2}{*}{$\frac{|\mathcal{S}||\mathcal{A}|}{(1-\gamma)^{3}\varepsilon^{2}}$
\vphantom{$\frac{1^{7^{7}}}{1^{7^{7^{7^{7}}}}}$}\hspace{-0.4em}} & \multirow{2}{*}{$(0,\frac{1}{1-\gamma}]$}\tabularnewline
\textbf{This paper} &  &  & \tabularnewline
\toprule
\end{tabular}
\end{center}
	\caption{Comparisons with prior results (up to log factors) regarding finding an $\varepsilon$-optimal policy in a $\gamma$-discounted infinite-horizon MDP with a generative model. The sample size range and the $\varepsilon$-range  stand for the range of sample size and optimality gap (e.g.~$\varepsilon$-accuracy) for the claimed sample complexity to hold. Note that the results in \citet{kearns1999finite,wang2017randomized} only hold for a restricted family of MDPs satisfying certain ergodicity assumptions.  In addition, \citet{azar2013minimax} (resp.~\citet{wainwright2019variance}) showed that empirical QVI (resp.~variance-reduced Q-learning) finds an $\varepsilon$-optimal Q-function estimate with sample complexity $\frac{|\mathcal{S}||\mathcal{A}|}{(1-\gamma)^{3}\varepsilon^{2}}$ ($\varepsilon\in (0,1]$) in a sample size range $[\frac{|\cS||\cA|}{(1-\gamma)^3},\infty)$, which did not translate directly to an $\varepsilon$-optimal policy. 
	\label{tab:prior-work}}  
\end{table}

\medskip

\begin{table}[H]

\begin{center}

\begin{tabular}{c|c|c|c}
\toprule
Algorithm & Sample size range & \multicolumn{1}{c|}{Sample complexity} & $\varepsilon$-range $\vphantom{\frac{1^{7}}{1^{7^3}}}$ \tabularnewline
\toprule 
Sublinear randomized value iteration & \multirow{2}{*}{$\big[ |\mathcal{S}||\mathcal{A}|H^3,\infty)$} & \multirow{2}{*}{$ \frac{|\mathcal{S}||\mathcal{A}|H^{5}}{\varepsilon^{2}}$
	\vphantom{$\frac{1^{7^{7}}}{1^{7^{7^{7^{7}}}}}$}\hspace{-0.4em}} & \multirow{2}{*}{$\big(0, H \big]$}\tabularnewline
\cite{sidford2018variance} &  &  & \tabularnewline
\hline 
Variance-reduced QVI & \multirow{2}{*}{$\big[ |\mathcal{S}||\mathcal{A}| H^4,\infty)$} & \multirow{2}{*}{$ \frac{|\mathcal{S}||\mathcal{A}|H^{4}}{\varepsilon^{2}}$
\vphantom{$\frac{1^{7^{7}}}{1^{7^{7^{7^{7}}}}}$}\hspace{-0.4em}} & \multirow{2}{*}{$(0,1]$}\tabularnewline
\cite{sidford2018near} &  &  & \tabularnewline
\hline 
Empirical MDP + planning & \multirow{2}{*}{$\big[ |\mathcal{S}||\mathcal{A}| H^{3} ,\infty)$} & \multirow{2}{*}{$ \frac{|\mathcal{S}||\mathcal{A}|H^{4}}{\varepsilon^{2}}$
\vphantom{$\frac{1^{7^{7}}}{1^{7^{7^{7^{7}}}}}$}\hspace{-0.4em}} & \multirow{2}{*}{$(0,\sqrt{H}]$}\tabularnewline
\cite{yin2021near} &  &  & \tabularnewline
\hline 
Empirical MDP + planning & \multirow{2}{*}{$\big[ |\mathcal{S}||\mathcal{A}| H^2 ,\infty)$} & \multirow{2}{*}{$\frac{|\mathcal{S}||\mathcal{A}| H^{4}}{\varepsilon^{2}}$
\vphantom{$\frac{1^{7^{7}}}{1^{7^{7^{7^{7}}}}}$}\hspace{-0.4em}} & \multirow{2}{*}{$(0, H]$}\tabularnewline
\textbf{This paper} &  &  & \tabularnewline
\toprule
\end{tabular}
\end{center}
	\caption{Comparisons with prior results (up to log factors) regarding finding an $\varepsilon$-optimal policy in a time-inhomogeneous finite-horizon MDP with a generative model. The sample size range and the $\varepsilon$-range  stand for the range of sample size and optimality gap (e.g.~$\varepsilon$-accuracy) for the claimed sample complexity to hold. The results in \citet{sidford2018variance,sidford2018near} were originally stated for the time-homogeneous case; we translate them into the time-inhomogeneous case with an additional factor of $H$. In addition, \citet{li2021tightening} proved that Q-learning finds an $\varepsilon$-optimal Q-function estimate with sample complexity $\frac{|\mathcal{S}||\mathcal{A}|H^4}{\varepsilon^{2}}$ ($\varepsilon\in (0,1]$) in a sample size range $[|\cS||\cA|H^4,\infty)$, which did not translate directly to an $\varepsilon$-optimal policy. 
	\label{tab:prior-work-finite}}  
\end{table}

\section{Problem formulation}

The current paper studies both discounted infinite-horizon MDPs and finite-horizon MDPs, which will be introduced separately in the sequel. Here and  throughout, we adopt the standard notation $[H]:=\{1,\cdots,H\}$.

\subsection{Discounted infinite-horizon Markov decision processes}

\paragraph{Models and background.} 
Consider a discounted infinite-horizon MDP represented by a quintuple $\mathcal{M} = (\cS,\cA, P, r,\gamma)$, where $\cS:=\{1,2,\ldots, |\cS|\}$ denotes a finite set of states, $\cA:=\{1,2,\ldots,|\cA|\}$ is a finite set of actions, $\gamma\in (0,1)$ stands for the discount factor, and $r:  \cS\times\cA \rightarrow [0,1]$ represents the reward function, namely, $r(s,a)$ is the immediate reward received upon executing action $a$ while in state $s$ (here and throughout, we consider the normalized setting where the rewards lie within $[0,1]$).  In addition, $P:\cS\times\cA \rightarrow \Delta(\cS)$ represents the probability transition kernel of the MDP, where $P(s^\prime \mymid s,a)$ denotes the probability of transiting from state $s$ to state $s^\prime$  when action $a$ is executed, and $\Delta(\cS)$ denotes the probability simplex over $\cS$.   


A deterministic policy (or action selection rule) is a mapping $\pi: \cS \rightarrow \cA$ that maps a state to an action. The value function $V^{\pi}: \cS \rightarrow \real$ of a policy $\pi$ is defined by
\begin{equation}
 \forall s\in \cS: \qquad V^{\pi}(s) \defn \mathbb{E} \left[ \sum_{t=0}^{\infty} \gamma^t r(s^t,a^t ) \,\big|\, s^0 =s \right],
\end{equation} 
which is the expected discounted total reward starting from the initial state $s^0=s$;
here,  the sample trajectory $\{(s^t, a^t)\}_{t\geq 0}$ is generated based on the transition kernel (namely, $s^{t+1}\sim P(\cdot \,|\, s^t, a^t)$), with the actions taken according to policy $\pi$ (namely, $a^t= \pi(s^t)$ for all $t\geq 0$). It is easily seen that $0\leq V^{\pi}(s)\leq \frac{1}{1-\gamma} $.
The corresponding action-value function (or Q-function) $Q^{\pi}: \cS \times \cA \rightarrow \real$ of a policy $\pi$ is defined by
\begin{equation}
 \forall (s,a)\in \cS \times \cA: \qquad Q^{\pi}(s,a) \defn \mathbb{E} \left[ \sum_{t=0}^{\infty} \gamma^t r(s^t,a^t ) \,\big|\, s^0 =s, a^0 = a \right],
\end{equation} 
where the actions are taken according to the policy $\pi$ after the initial action (i.e.~$a^t= \pi(s^t)$ for all $t\geq 1$).  
It is well-known that there exists an optimal policy, denoted by $\pi^{\star}$,  
that simultaneously maximizes $V^{\pi}(s)$ (resp.~$Q^{\pi}(s,a)$) for all states $s\in\cS$ (resp.~state-action pairs $(s,a)\in (\cS\times\cA)$) \citep{sutton2018reinforcement}.
The corresponding value function  $V^{\star} :=V^{\pi^{\star}}$ (resp.~action-value function $Q^{\star} := Q^{\pi^{\star}}$) is called the optimal value function (resp.~optimal action-value function).

\paragraph{A generative model and an empirical MDP.} 
The current paper focuses on a stylized generative model (also called a simulator) as studied in \cite{kearns2002sparse,kakade2003sample}. Assuming access to this generative model, we collect $N$ independent samples 
$$
	s_{s,a}^{i} \overset{\mathrm{i.i.d.}}{\sim} P(\cdot \mymid s,a), \qquad i=1,\ldots, N
$$ 
for each state-action pair $(s,a)\in \cS\times \cA$, which allows us to construct an empirical transition kernel $\widehat{P}$ as follows
\begin{equation}
	\forall s' \in \mathcal{S}, \qquad \widehat{P}(s^\prime \mymid s, a) = \frac{1}{N}\sum\nolimits_{i=1}^N \ind \{s_{s,a}^{i} = s^{\prime} \},
	\label{eq:defn-empirical-P}
\end{equation}
where $\ind\{\cdot\}$ is the indicator function. In words, $\widehat{P}(s^\prime \mymid s,a)$  counts the empirical frequency of transitions from $(s,a)$ to state $s^\prime$. The total sample size should therefore be understood as $N^{\mathsf{total}} := N |\cS| |\cA|$.  
 This leads to an empirical MDP $\widehat{\mathcal{M}} = (\cS,\cA, \widehat{P}, r,\gamma)$ 
constructed from the data samples. We can define the value function and the action-value function of a policy $\pi$ for $\widehat{\mathcal{M}} $ analogously, which we shall denote by $\widehat{V}^{\pi}$ and $\widehat{Q}^{\pi}$, respectively. The optimal policy of $\widehat{\mathcal{M}}$ is denoted by $\widehat{\pi}^{\star}$, with the optimal 
value function and Q-function  denoted  by $\widehat{V}^{\star}:=\widehat{V}^{\widehat{\pi}^\star}$ and $\widehat{Q}^{\star}:=\widehat{Q}^{\widehat{\pi}^\star}$, respectively.

\paragraph{Learning the optimal policy via model-based planning.} 
Given a few data samples in hand, the task of policy learning seeks to identify a policy that (approximately) maximizes the expected discounted reward given the data samples. Specifically, for any target level $\varepsilon>0$, the aim is to compute an $\varepsilon$-accurate policy ${\pi}_{\mathsf{est}}$ obeying
\begin{align}
	\forall (s,a)\in\cS\times \cA: \qquad V^{{\pi}_{\mathsf{est}}}(s) \geq V^{\star}(s) -\varepsilon, \quad  Q^{{\pi}_{\mathsf{est}}}(s,a) \geq Q^{\star}(s,a) -\varepsilon.
\end{align}
Naturally, one would hope to accomplish these tasks with as few samples as possible. Recall that for the normalized reward setting with $0\leq r\leq 1$, the value function and Q-function fall within the range $[0,\frac{1}{1-\gamma}]$; this means that the  range of the target accuracy level $\varepsilon$ should be set to $\varepsilon \in [0, \frac{1}{1-\gamma}]$. The model-based approach typically starts by constructing an empirical MDP $\widehat{\mathcal{M}}$ based on all collected samples, and then ``plugs in'' this empirical model directly into the Bellman recursion to perform policy evaluation or planning, with prominent examples including Q-value iteration (QVI) and policy iteration (PI) \citep{bertsekas2017dynamic}. 

\paragraph{Aside: policy evaluation.} A related task is {policy evaluation}, which aims to compute or approximate the value function $V^{\pi}$ under a given policy $\pi$. To be precise, for any target level $\varepsilon>0$, the goal is to find an $\varepsilon$-accurate estimate $V^{\pi}_\mathsf{est}$ such that  
\begin{align}
	\forall s\in\cS: \qquad \big| {V}^{\pi}_{\mathsf{est}}(s) - {V}^{\pi}(s) \big| \leq \varepsilon. 
\end{align}

\subsection{Finite-horizon Markov decision processes} 
\paragraph{Models and background.}
Another type of models considered in this paper is a finite-horizon MDP, which can be represented and denoted by $\mathcal{M} = (\mathcal{S}, \mathcal{A}, \{P_h\}_{h=1}^H, \{r_h\}_{h=1}^H, H)$. 
Here, $\cS$ and $\cA$ denote respectively the state space and the action space as before, and $H$ represents the  horizon length of the MDP. 
For any $1\leq h\leq H$, we let $P_h : \cS \times \cA \rightarrow \Delta(\cS)$ denote the probability transition kernel at step $h$, that is, $P_h(s' \mymid s,a)$ is the probability of transiting to $s'$ from $(s,a)$ at step $h$;  
$r_h: \cS \times \cA \rightarrow [0,1]$ indicates the reward function at step $h$, namely, $r_h(s,a)$ is the immediate reward gained at step $h$ in response to $(s,a)$. As before, we assume normalized rewards throughout the paper, so that all the $r_h(s,a)$'s reside within the interval $[0,1]$. 

Let  $\pi=\{\pi_h\}_{1\leq h\leq H}$ represent a deterministic policy, such that for any $1\leq h\leq H$ and any $s\in \cS$,  $\pi_h(s)$ specifies the action selected at step $h$ in state $s$. Note that $\pi$ could be non-stationary, meaning that the $\pi_h$'s might be different across different time steps $h$.  The value function and the Q-function associated with policy $\pi$ are defined respectively by
 \begin{align*}
   V^{\pi}_h (s) \defn \mathbb{E} \Big[ \sum_{k=h}^{H} r_k(s_k,a_k ) \,\Big|\, s_h =s \Big]
\end{align*} 
 for all $s\in \cS$ and all $1\leq h \leq H$, and
 \begin{equation*}
Q^{\pi}_h(s,a) \defn \mathbb{E} \Big[ \sum_{k=h}^{H} r_k(s_k,a_k ) \,\Big|\, s_h =s, a_h = a \Big] 
\end{equation*} 
for all $(s,a)\in \cS \times \cA$  and all $1\leq h \leq H$. 
As usual,  the expectations are taken over the randomness of the MDP trajectory $\{(s_k,a_k)\}_{1\leq k\leq H}$ induced by the transition kernel $\{ P_h \}_{h=1}^H$ when policy $\pi$ is adopted.
With slight abuse of notation, we let $Q^{\pi}_{H+1}(s,a) = 0$ for every $(s,a) \in \cS \times \cA$ and $V^{\pi}_{H+1} (s) = 0$ for every $s\in \cS.$
In view of the assumed range of the immediate rewards, it is easily seen that 
$$
	0\leq V^{\pi}_h (s) \leq H \qquad \text{and} \qquad 0 \leq Q^{\pi}_h (s, a) \leq H
$$
for any $\pi$, any state-action pair $(s, a)$, and any step $h$. 
Akin to the infinite-horizon counterpart, the optimal value functions $\{V_h^{\star}\}_{1\leq h\leq H}$ and optimal Q-functions $\{Q_h^{\star}\}_{1\leq h\leq H}$ are defined respectively by
\begin{equation*}
	V^{\star}_h(s) \defn \max_{\pi} V^{\pi}_h(s)
	\qquad \text{and} \qquad
	Q^{\star}_h(s,a) \defn \max_{\pi} Q^{\pi}_h(s,a)
\end{equation*}
for any state-action pair $(s,a)\in \cS\times \cA$ and any $1\leq h \leq H$. 
It is well known that there exists at least one policy that allows one to simultaneously achieve the optimal value function and optimal Q-functions for all state-action pairs and all time steps. 
Throughout this paper, we shall denote by $\pi^{\star} = \{\pi_h^{\star}\}_{1\leq h\leq H}$ an optimal policy.   

\paragraph{A generative model and an empirical MDP.} 
Similar to the infinite-horizon setting, 
we assume access to a generative model, which is able to generate $N$ independent samples for each triple $(s,a,h)\in \cS\times \cA\times [H]$ as follows
\[
	s_{h}^i(s,a) \overset{\mathrm{i.i.d.}}{\sim} P_h(\cdot \mymid s,a), \qquad i=1,\ldots, N.
\]
The empirical transition kernel $\{\widehat{P}_h\}_{h=1}^H$ is thus given by
\begin{equation}
	\forall s' \in \mathcal{S}, \qquad \widehat{P}_h(s^\prime \mymid s, a) = \frac{1}{N}\sum\nolimits_{i=1}^N \ind \{s_{h}^i(s,a) = s^{\prime} \}, 
	\label{eq:defn-empirical-P-finite}
\end{equation}
which records the empirical frequency of transitions from $(s,a)$ to state $s^\prime$ at step $h$. This gives rise to a total sample size $N^{\mathsf{total}} := N H |\cS| |\cA|$.  
We shall let  $\widehat{\mathcal{M}} = (\cS,\cA, \{\widehat{P}_h\}_{h=1}^H, \{r_h\}_{h=1}^H, H)$ 
represent the empirical MDP constructed from the data samples. The value function and the Q-function of a policy $\pi$ for $\widehat{\mathcal{M}} $ can be defined analogously, which shall be denoted by $\{\widehat{V}^{\pi}_h\}_{h=1}^H$ and $\{\widehat{Q}^{\pi}_h\}_{h=1}^H$, respectively. We denote by $\widehat{\pi}^{\star}$ the optimal policy of $\widehat{\mathcal{M}}$, and the resulting optimal 
value function and Q-function are denoted by $\widehat{V}^{\star}_h:=\widehat{V}^{\widehat{\pi}^\star}_h$ and $\widehat{Q}^{\star}_h:=\widehat{Q}^{\widehat{\pi}^\star}_h$, respectively.

\paragraph{Learning the optimal policy via model-based planning.} 
Given the data samples in hand, the task of policy learning in the finite-horizon case can be defined similarly as the infinite-horizon counterpart.  
Specifically, for any target level $\varepsilon>0$, the aim of policy learning is to compute an $\varepsilon$-accurate policy ${\pi}_{\mathsf{est}}$ obeying
\begin{align}
	 \forall (s,a,h)\in\cS\times \cA \times [H]: \qquad V^{{\pi}_{\mathsf{est}}}_h (s) \geq V_h^{\star}(s) -\varepsilon, \quad  Q^{{\pi}_{\mathsf{est}}}_h(s,a) \geq Q_h^{\star}(s,a) -\varepsilon.
\end{align}
With the normalized range of the reward function, it is easily seen that the value function and the Q-function reside within the interval $[0,H]$, 
thus implying that the range of the target accuracy level should be $\varepsilon \in [0,H]$.

\subsection{Notation}
\label{sec:notation}

Let $\mathcal{X}:= \big( |\mathcal{S}|,|\mathcal{A}|, \frac{1}{1-\gamma}, \frac{1}{\varepsilon} \big)$.  The notation $f(\mathcal{X}) = O(g(\mathcal{X}))$   means there exists a universal constant $C_1>0$ such that $f\leq C_1 g$, whereas the notation $f(\mathcal{X}) = \Omega(g(\mathcal{X}))$ means $g(\mathcal{X}) = O(f(\mathcal{X}))$. In addition, the notation $\widetilde{O}(\cdot)$ (resp.~$\widetilde{\Omega}(\cdot)$) is defined in the same way as ${O}(\cdot)$ (resp.~$\Omega(\cdot)$) except that it ignores logarithmic factors. 

For any vector $\bm{a}=[a_i]_{1\leq i\leq n}\in \mathbb{R}^n$, we overload the notation $\sqrt{\cdot}$ and $|\cdot|$ in an entry-wise manner such that $\sqrt{\bm{a}}:=[\sqrt{a_i}]_{1\leq i\leq n}$ and $|\bm{a}|:=[|a_i|]_{1\leq i\leq n}$. For any vectors $\bm{a}=[a_i]_{1\leq i\leq n}$ and $\bm{b}=[b_i]_{1\leq i\leq n}$, the notation $\bm{a} \geq \bm{b}$ (resp.~$\bm{a} \leq \bm{b}$) means $a_i \geq b_i$ (resp.~$a_i\leq b_i$) for all $1\leq i\leq n$, and we let $\bm{a} \circ \bm{b} := [a_ib_i]_{1\leq i\leq n}$ represent the Hadamard product. 
Additionally, we denote by $\bm{1}$ the all-one vector, and $\bm{I}$ the identity matrix. For any matrix $\bm{A}$, we define the norm $\|\bm{A}\|_1:=\max_i \sum_j |A_{i,j}|$.

\section{Model-based planning in discounted infinite-horizon MDPs}
\label{sec:main-results}


As summarized in Table~\ref{tab:prior-work}, the theory of all prior works required the sample size per state-action pair to at least exceed $N \geq  \Omega\big( \frac{1}{(1-\gamma)^2} \big)$.   In order to break this sample size barrier, we develop two model-based algorithms that provably overcome such a sample size barrier.

\subsection{Model-based reinforcement learning: two algorithms}
\label{sec:algorithm}

\paragraph{Algorithm 1: perturbed model-based planning.} 
The first algorithm applies model-based planning to an empirical MDP with {\em randomly perturbed rewards}. 
Specifically, for each state-action pair $(s,a)\in\mathcal{S}\times\mathcal{A}$, we randomly perturb the immediate reward by 
\begin{equation}
	r_{\mathrm{p}} (s,a) = r(s,a) + \zeta(s,a), \qquad \zeta(s,a) \overset{\mathrm{i.i.d.}}{\sim} \mathsf{Unif}(0,\xi),
\label{eq:perturbed-reward-gone}
\end{equation}
where $\mathsf{Unif}(0,\xi)$ denotes the uniform distribution between 0 and some parameter $\xi>0$ (to be specified momentarily).\footnote{{Note that perturbation is only invoked when running the planning algorithms
and does not require collecting new samples. }} 
%
%
%
%
For any policy $\pi$, we denote by $\widehat{{V}}^{\pi}_{{\mathrm{p}}}$ the corresponding value function of the perturbed empirical MDP $\widehat{\mathcal{M}}_{\mathrm{p}}=(\cS,\cA, \widehat{P}, {r}_{\mathrm{p}}, \gamma)$ with the probability transition kernel $\widehat{P}$ (cf.~\eqref{eq:defn-empirical-P}) and the perturbed reward function ${r}_{\mathrm{p}}$. Let $\widehat{\pi}_{\mathrm{p}}^{\star}$ represent the optimal policy of $\widehat{\mathcal{M}}_{\mathrm{p}}$, i.e.
\begin{align}
	\widehat{\pi}_{\mathrm{p}}^{\star} := \arg\max_\pi \widehat{{V}}^{\pi}_{{\mathrm{p}}}.
	\label{defn:pi-p-star-perturb}
\end{align}

\paragraph{Algorithm 2: conservative model-based planning.}

An alternative approach that eliminates the need of reward perturbation is to  
select {\em approximately optimal} actions for the empirical MDP instead of the absolute optimal actions.  
To be precise, denote by $\widehat{Q}^{\star}$ (resp.~$\widehat{V}^{\star}$)  the optimal action-value (resp.~value) function of the empirical MDP $\widehat{\mathcal{M}}=(\cS,\cA, \widehat{P}, r, \gamma)$ with the probability transition kernel $\widehat{P}$ (cf.~\eqref{eq:defn-empirical-P}) and the original reward function $r$.
By producing a random draw from $\newnoise \sim \mathsf{Unif}(0,\xi)$ (with $\xi$ specified shortly), we can generate the following policy $\picon$: 
\begin{align}
	\label{eq:policy-Q-approx-MDP}
	\forall s \in \mathcal{S}: \qquad 
	\picon(s) \defn \min \big\{a \in \mathcal{A}: \widehat{Q}^{\star}(s,a) > \widehat{V}^{\star}(s) - \newnoise \big\}.
\end{align}
Note that there is an index assigned to each action as $\mathcal{A}=\{1,\cdots, |\cA|\}$ which induces a natural order for all actions. 
In words, this approach is more conservative and does not stick to the optimal actions w.r.t. the empirical MDP; instead, the policy $\picon$ picks out --- for each state $s\in \cS$ --- the smallest indexed action that is within a gap of $\newnoise$ from optimal.


\subsection{Theoretical guarantees}

Indeed, both the above-mentioned approaches result in a value function (resp.~Q-function) that well approximates the true optimal value function $V^{\star}$ (resp.~optimal Q-function $Q^{\star}$). We start by presenting our results for the perturbed model-based approach. 
%
%
\begin{theorem}[Perturbed model-based planning]
\label{Thm:sample-compl-main}
	There exist some universal constants $c_0,c_1>0$ such that: for any $\delta > 0$ and any $0<\varepsilon \leq \frac{1}{1-\gamma}$, the policy $\widehat{\pi}_{\mathrm{p}}^{\star}$ defined in \eqref{defn:pi-p-star-perturb} obeys
	\begin{align}
		\forall (s,a) \in \mathcal{S} \times \mathcal{A}, \qquad 
		V^{\pihatstar_{\mathrm{p}}}(s)  \geq  V^\star(s) - \varepsilon
		\qquad \text{and} \qquad
		 Q^{\pihatstar_{\mathrm{p}}}(s,a)  \geq  Q^\star(s,a)  - \gamma \varepsilon,  
	\end{align}
	with probability at least $1-\delta$, provided that the perturbation size is $\xi = \frac{c_1(1-\gamma)\varepsilon}{|\mathcal{S}|^5|\mathcal{A}|^5}$ and that 
	the sample size per state-action pair exceeds 
	\begin{align}
	\label{EqnSamples-main}
		N \geq \frac{c_{0}\log\big(\frac{|\mathcal{S}||\mathcal{A}|}{(1-\gamma)\varepsilon 
		\delta}\big)}{(1-\gamma)^{3}\varepsilon^{2}}  . 
	\end{align}
	In addition, both the empirical QVI and PI algorithms w.r.t.~$\widehat{\mathcal{M}}_{\mathrm{p}}$ (cf.~\cite[Algorithms 1-2]{azar2013minimax}) are able to recover $\widehat{\pi}_{\mathrm{p}}^{\star}$ perfectly within $O\big(\frac{1}{1-\gamma}\log(\frac{|\cS||\cA|}{(1-\gamma)\varepsilon\delta}) \big)$ iterations.  
\end{theorem}
\begin{remark}
	Theorem~\ref{Thm:sample-compl-main} holds unchanged if  $\xi$ is taken to be $\frac{c_1(1-\gamma)\varepsilon}{|\mathcal{S}|^{\alpha}|\mathcal{A}|^{\alpha}}$ for any fixed constant $\alpha\geq 1$. This paper picks the specific choice $\alpha= 5$ merely to convey that a very small degree of perturbation suffices for our purpose. 
\end{remark}
\begin{remark}
	Perturbation brings a side benefit: one can recover the optimal policy $\pihatstar_{\mathrm{p}}$ of the perturbed empirical MDP $\widehat{\mathcal{M}}_{\mathrm{p}}$ {\em exactly} in a small number of iterations without incurring further optimization errors.  To give a flavor of the overall computational complexity, let us take QVI for example \cite{azar2013minimax}. Recall that each iteration of  QVI takes time proportional to the time taken to read $\widehat{P}$ (which is a matrix with at most $N |\cS||\cA|$ nonzeros), hence the resulting computational complexity can be as low as ${O}\big(\frac{|\cS||\cA|}{(1-\gamma)^4\varepsilon^2}\log^2(\frac{|\cS||\cA|}{(1-\gamma)\varepsilon\delta})  \big)$.  
\end{remark}
Further,  similar performance guarantees can be established for the conservative model-based approach without reward perturbation, as stated below. 
\begin{theorem}[Conservative model-based planning]
\label{Thm:sample-compl-conservative}
	Under the same assumptions of Theorem~\ref{Thm:sample-compl-main} (including both the sample size and the choice of $\xi$), the policy $\picon$ defined in \eqref{eq:policy-Q-approx-MDP} achieves
	\begin{align}
		\forall (s,a) \in \mathcal{S} \times \mathcal{A}, \qquad 
		V^{\picon}(s)  \geq  V^\star(s) - \varepsilon
		\qquad \text{and} \qquad
		 Q^{\picon}(s,a)  \geq  Q^\star(s,a)  - \gamma \varepsilon 
	\end{align}
	with probability at least $1-\delta$.
\end{theorem}
%

In a nutshell, the above theorems demonstrate that: both model-based algorithms we introduce succeed in finding an $\varepsilon$-optimal policy  
as soon as the total sample complexity exceeds the order of $\frac{|\mathcal{S}||\mathcal{A}|}{(1-\gamma)^{3}\varepsilon^2}$ (modulo some log factor). It is worth emphasizing that, compared to prior literature, our result imposes no restriction on the range of $\varepsilon$ and, in particular, we allow the accuracy level $\varepsilon$ to go all the way up to $\frac{1}{1-\gamma}$. 
Our result is particularly useful in the regime with small-to-moderate sample sizes, since its validity is guaranteed as long as 
\begin{align}
	N \geq \widetilde{\Omega}\Big(\frac{1}{1-\gamma}\Big) . 
	\label{eq:sample-size-range-planning}
\end{align}
Tackling the sample-limited regime (in particular, the scenario when $N \in [\frac{1}{1-\gamma},\frac{1}{(1-\gamma)^2}]$)  requires us to develop new analysis frameworks beyond prior theory, which we shall discuss in detail momentarily.

We remark that the work \cite{azar2013minimax} established a minimax lower bound of the same order as \eqref{EqnSamples-main} (up  to some log factor) in  the  regime $\varepsilon = O(1)$. 
A closer inspection of their analysis, however, reveals that their argument and bound hold true as long as $\varepsilon = O(\frac{1}{1-\gamma})$. This in turn corroborates the {\em minimax optimality} of our perturbed model-based approach for the full $\varepsilon$-range (which is previously unavailable), and demonstrates the information-theoretic infeasibility to learn a policy strictly better than a random guess if  $N\leq \widetilde{O}\big( \frac{1}{1-\gamma} \big)$. Put another way, the condition \eqref{eq:sample-size-range-planning} contains the full range of  ``meaningful'' sample sizes.


Finally, we single out an intermediate result in the analysis of our theorems concerning model-based policy evaluation, which might be of interest on its own. Specifically, for any fixed policy $\pi$ independent of the data, this task concerns value function estimation via the plug-in estimate $\widehat{V}^{\pi}$ (i.e.~the value function of the empirical $\mathcal{M}$ under this policy). However simple as this might seem, existing theoretical underpinnings of this approach remain suboptimal, unless the sample size is already sufficiently large. Our result is the following, which does not require enforcing reward perturbation. 
\begin{theorem}[Model-based policy evaluation] 
\label{Thm:very-wise-evaluation} 
Fix any policy $\pi$.  There exists some universal constant $c_{0}>0$ such that: for any $0<\delta<1$ and any $0<\varepsilon \leq \frac{1}{1-\gamma}$, one has
	\begin{align} 
		\forall s\in \cS:\qquad  \big| \widehat{V}^{\pi}(s)  - V^\pi(s) \big| \leq \varepsilon
	\end{align}
	with probability at least $1-\delta$, provided that the sample size per state-action pair exceeds 
	\begin{align}
	\label{EqnSamples-evaluation}
		N \geq  c_0 \frac{\log\Big(\frac{|\mathcal{S}|\log\frac{e}{1-\gamma}}{\delta}\Big)}{(1-\gamma)^3 \varepsilon^2}. 
	\end{align}
\end{theorem}
In words, this theorem reveals that $\widehat{V}^{\pi}$ begins to outperform a random guess as soon as $N \geq  \widetilde{\Omega} \Big( \frac{1}{1-\gamma } \Big)$. 
%
%
The sample complexity bound \eqref{EqnSamples-evaluation} enjoys {\em full} coverage of the $\varepsilon$-range  $(0, \frac{1}{1-\gamma}]$, and matches the minimax lower bound derived in \cite[Theorem 2(b)]{pananjady2019value} up to only a $\log\log\frac{1}{1-\gamma}$ factor.   In addition, a recent line of work investigated instance-dependent guarantees for policy evaluation (\cite{pananjady2019value,khamaru2020temporal}). 
While this is not our focus, our analysis does uncover an instance-dependent  bound with a broadened sample size range. See Lemma~\ref{lemma:fixed-policy-error} and the discussion thereafter.

\subsection{Comparisons with prior works and implications}

In order to discuss the novelty of our results in context,  we take a moment to compare them with prior theory.  
See Table~\ref{tab:prior-work} for a more complete list of comparisons.

\paragraph{Prior bounds for planning and policy learning.} None of the prior results with a generative model (including both model-based or model-free approaches) was capable of efficiently finding the desired policy while accommodating the full sample size range \eqref{eq:sample-size-range-planning}. For instance,  the state-of-the-art analysis for the model-based approach  \cite{agarwal2019optimality} required the sample size to at least exceed
		\begin{align}
			\label{eq:prior-sample-size-barrier-model-based}
			N \geq \widetilde{\Omega}\Big( \frac{1}{(1-\gamma)^2} \Big) ,
		\end{align}
whereas the theory for the variance-reduced model-free approach \cite{sidford2018near,wainwright2019variance} imposed the sample size requirement 
		\begin{align}
			\label{eq:prior-sample-size-barrier-model-free}
			N \geq \widetilde{\Omega}\Big( \frac{1}{(1-\gamma)^3} \Big) .
		\end{align}
In fact, it was previously unknown what is achievable in the sample size range $N\in [\frac{1}{1-\gamma},\frac{1}{(1-\gamma)^2}]$.  In contrast, our results confirm the minimax-optimal statistical performance of the model-based approach with full coverage of the $\varepsilon$-range and the sample size range. 

\begin{remark}
	{
	We briefly point out why the sample size barrier \eqref{eq:prior-sample-size-barrier-model-based} appeared in the analysis of \cite{agarwal2019optimality}.  Take  \cite{agarwal2019optimality} Section 4.3 for example:  the contraction factor $\gamma\sqrt{\frac{8\log(|\mathcal{S}||\mathcal{A}|/(1-\gamma)\delta)}{N}}\frac{1}{1-\gamma}$  therein needs to be smaller than 1, thereby requiring $N\geq \widetilde{\Omega}\big( (1-\gamma)^{-2}\big)$. }
\end{remark}

\paragraph{Prior bounds for policy evaluation.}

Regarding value function estimation for any fixed policy $\pi$, the prior results \cite{azar2013minimax,agarwal2019optimality,pananjady2019value} 
 for the plug-in approach all operated under the assumption that
$N \geq \widetilde{\Omega}\big( \frac{1}{(1-\gamma)^2} \big) $,
which is more stringent than our result by a factor of at least $\frac{1}{{1-\gamma}}$. 
In addition, our sample complexity matches the state-of-the-art guarantees  in the regime where $\varepsilon \leq \frac{1}{\sqrt{1-\gamma}}$  \citep{agarwal2019optimality,pananjady2019value}, while extending them to the range $\varepsilon\in \big[\frac{1}{\sqrt{1-\gamma}}, \frac{1}{1-\gamma}\big]$  uncovered in these previous papers.

\section{Model-based planning in finite-horizon MDPs}
\label{sec:main-result-finite-MDPs}




Moving beyond discounted infinite-horizon MDPs, 
our theoretical framework is also able to accommodate finite-horizon MDPs, which we detail in this section.

\subsection{Algorithm: model-based planning} 
The algorithm considered in this section is model-based planning (without reward perturbation). 
Specifically, this model-based approach returns a policy $\widehat{\pi}^{\star}=\{\widehat{\pi}^{\star}_h\}_{1\leq h\leq H}$ by means of the following two steps:
\begin{itemize}
	\item[1)] Construct the empirical MDP $\widehat{\mathcal{M}}=(\cS,\cA, \{\widehat{P}_h\}_{1\le h\le H}, \{r_h\}_{1\le h\le H}, H)$ based on the data samples in hand (see \eqref{eq:defn-empirical-P-finite} for the computation of the empirical transition kernel $\widehat{P}_h$);

	\item[2)] Run a classical dynamic programming algorithm \citep{bertsekas2017dynamic} to find an optimal policy $\widehat{\pi}^{\star}$ of the empirical MDP $\widehat{\mathcal{M}}$. 
\end{itemize}
%

\noindent
Note that $\widehat{\pi}^{\star}_h$ is an optimal policy of $\widehat{\mathcal{M}}$ at step $h$, computed by the dynamic programming algorithm calculated backward from $h=H$. Since 
 $\widehat{\pi}^{\star}_h$ is calculated solely based on what happens after step $h$, $\widehat{\pi}^{\star}_h$ is independent of the empirical transitions $\{\widehat{P}_j\}_{1\leq j< h}$.

It is noteworthy that, in contrast to the infinite-horizon counterpart in Section~\ref{sec:main-results},  we do not need to enforce random reward perturbation for this finite-horizon case.

\subsection{Theoretical guarantees and implications} 
The model-based algorithm described above turns out to be nearly minimax optimal, as asserted by the following theorem. 
\begin{theorem}[Model-based planning]
\label{Thm:sample-compl-finite}
	There exist some universal constants $c_0,c_1>0$ such that: for any $\delta > 0$ and any $0<\varepsilon \leq H$, the aforementioned policy $\widehat{\pi}^{\star}$ returned by model-based planning  obeys
	\begin{align}
		\forall (s,a,h) \in \mathcal{S} \times \mathcal{A} \times [H]: \qquad 
		V_h^{\widehat{\pi}^{\star}}(s)  \geq  V_h^\star(s) - \varepsilon
		\quad \text{and} \quad
		 Q_h^{\widehat{\pi}^{\star}}(s,a)  \geq  Q_h^\star(s,a)  - \varepsilon 
	\end{align}
	with probability at least $1-\delta$, provided that 
	the sample size for every triple $(s,a,h)$ exceeds 
	\begin{align}
	\label{EqnSamples-finite}
		N \geq \frac{c_{0}H^3\log\big(\frac{H|\mathcal{S}||\mathcal{A}|}{ 
		\delta}\big)}{\varepsilon^{2}}  . 
	\end{align}
\end{theorem}

Akin to the discounted infinite-horizon scenario, the model-based approach manages to achieve $\varepsilon$-accuracy as long as the sample size per $(s,a,h)$ exceeds the order of
\[
	\frac{H^3}{\varepsilon^2} \quad (\text{up to some log factor}).
\]
This result, which is valid for the full $\varepsilon$ range $(0,H]$, is reminiscent of the bound \eqref{EqnSamples-main}, except that the effective horizon $\frac{1}{1-\gamma}$ needs to be replaced by the horizon length $H$.  
Given that there are in total $|\cS||\cA|H$ different combinations of $(s,a,h)$, the total sample complexity is on the order of $\widetilde{O} \bigg( \frac{|\cS||\cA|H^4}{\varepsilon^2} \bigg).$
%
	
%
The quadruple scaling $H^4$ of this total sample complexity --- as opposed to the cubic scaling in the discounted infinite-horizon case --- is due to time inhomogeneity; that is, the $P_h$'s might be different across $h$, resulting in an additional $H$ factor. Again, our result kicks in as soon as the sample size satisfies
\begin{equation}
	N \geq \widetilde{\Omega} (H) ,
	\label{eq:sample-size-range-planning-finite}
\end{equation}
improving upon the sample size requirement
		\begin{equation}
			\label{eq:prior-sample-size-barrier-model-based-finite}
			N \geq \widetilde{\Omega} (H^2 ) 
		\end{equation}
 in the state-of-the-art analysis for the model-based approach \cite{yin2021near}.

\section{Other related works} 
\label{sec:related-work}

Classical analyses of reinforcement learning algorithms have largely focused on asymptotic performance (e.g.~\cite{tsitsiklis1997analysis,szepesvari1998asymptotic,tsitsiklis1994asynchronous,jaakkola1994convergence}). Leveraging the toolkit of concentration inequalities, a number of recent papers have shifted attention towards understanding the performance in the non-asymptotic and finite-time settings. A highly incomplete list includes \cite{bhandari2018finite,dalal2018finite,lakshminarayanan2018linear,srikant2019finite,xu2019two,mou2020linear,khamaru2020temporal,kaledin2020finite,kearns1999finite,bradtke1996linear,beck2012error,strehl2006pac,wainwright2019variance,gupta2019finite,even2003learning,xu2019finite,cai2019neural,azar2017minimax,jin2018q,shah2018q,yang2019theoretical,sidford2018near,chen2020finite,li2020sample,li2021breaking,qu2020finite,li2021q,yan2022efficacy,shi2022pessimistic}, a large fraction of which is concerned with model-free algorithms.

The generative model (or simulator) adopted in this paper was first proposed in \cite{kearns1999finite}, which has been invoked in \cite{kearns2002sparse,kakade2003sample,azar2013minimax,kearns1999finite,sidford2018near,azar2012sample,sidford2018variance,wang2017randomized,wainwright2019variance, khamaru2020temporal,pananjady2019value,agarwal2019optimality,yang2019sample,lattimore2012pac,li2022minimax,wang2021sample}, to name just a few. In particular, \cite{azar2013minimax} developed the minimax lower bound on the sample complexity $N= {\Omega}\big(\frac{|\mathcal{S}||\mathcal{A}|\log (|\cS||\cA|)}{(1-\gamma)^{3}\varepsilon^{2}}  \big)$ necessary for finding an $\varepsilon$-optimal policy, and  showed that, for any $\varepsilon\in(0,1)$, a model-based approach (e.g.~applying QVI or PI to the empirical MDP) can estimate the optimal Q-function to within an $\varepsilon$-accuracy given near-minimal samples. 
Note, however, that directly translating this result to the policy guarantees leads to an additional factor of $\frac{1}{1-\gamma}$ in estimation accuracy and of $\frac{1}{(1-\gamma)^2}$ in sample complexity. In light of this, \cite{azar2013minimax} further showed that a near-optimal sample complexity is possible for policy learning if the sample size is at least on the order of $ \frac{|\cS|^2|\cA|}{(1-\gamma)^2}$ which, however, is no longer sub-linear in the model complexity. A recent breakthrough  \cite{agarwal2019optimality} substantially improved the model-based guarantee with the aid of auxiliary state-absorbing MDPs,  extending the range of sample complexity to $[\frac{|\cS||\cA|\log(|\cS||\cA|)}{(1-\gamma)^2},\infty)$. Our analysis is motivated in part by  \cite{agarwal2019optimality}, but also relies on several other novel techniques to complete the picture.  



Finally, we remark that the construction of state-absorbing MDPs or state-action-absorbing MDPs falls under the category of ``leave-one-out'' type analysis, which is particularly effective in decoupling complicated statistical dependency in various statistical estimation problems \cite{el2015impact,agarwal2019optimality,chen2019spectral, chen2021spectral, Chen22931, ma2017implicit, yan2021inference, pananjady2019value}. The application of such an analysis framework to MDPs should be attributed to \cite{agarwal2019optimality}. Other applications to Markov chains include \cite{chen2019spectral,pananjady2019value}. 
More recently, several follow-up works have further generalized the leave-one-out analysis idea to accommodate broader RL settings including
offline RL \citep{li2022settling}, RL with linear function approximation \citep{wang2021sample}, and Markov games \citep{cui2021minimax,yan2022model}, and so on.


\section{Analysis: infinite-horizon MDPs}
\label{sec:analysis-all}

This section presents the key ideas for proving our main results, following an introduction of some convenient matrix notation. 

\subsection{Matrix notation and Bellman equations} 
\label{subsec:matrix-notation}

It is convenient to present our proof based on some matrix notation for MDPs.  Denoting by $\bm{e}_1,\cdots,\bm{e}_{|\mathcal{S}|}\in \mathbb{R}^{|\mathcal{S}|}$ the standard basis vectors, we can define: 
\begin{itemize}
	\item $\bm{r} \in \mathbb{R}^{|\mathcal{S}||\mathcal{A}|}$: a vector representing the reward function $r$ (so that $r_{(s,a)} =r(s,a)$ for all $(s,a)\in \cS\times \cA$). 
	\item $\bm{V}^{\pi} \in \mathbb{R}^{|\mathcal{S}|}$: a vector representing the value function $V^{\pi}$ (so that ${V}^{\pi}_s = {V}^{\pi}(s)$ for all $s\in \cS$).
	\item $\bm{Q}^{\pi} \in \mathbb{R}^{|\mathcal{S}||\mathcal{A}|}$: a vector representing the Q-function $Q^{\pi}$ (so that ${Q}^{\pi}_{(s,a)} = {Q}^{\pi}(s,a)$ for all $(s,a)\in \cS\times \cA$).
	\item  $\bm{V}^{\star} \in \mathbb{R}^{|\mathcal{S}|}$ and $\bm{Q}^{\star} \in \mathbb{R}^{|\mathcal{S}||\mathcal{A}|}$: representing the optimal value function $V^{\star}$ and optimal Q-function $Q^{\star}$.
	\item $\bm{P}\in \mathbb{R}^{|\cS||\cA|\times |\cS|}$: a matrix representing the probability transition kernel $P$, where the $(s,a)$-th row of $\bm{P}$ is a probability vector representing $P(\cdot|s,a)$.  
	Denote $\bm P_{s,a}$ as the $(s,a)$-th row of the transition matrix $\bm P$. 
	\item $\bm{\Pi}^{\pi} \in \{0,1\}^{|\cS| \times |\cS||\cA|}$: a projection matrix associated with a given policy $\pi$ taking the following form
\begin{align}
\label{eqn:bigpi}
	\bPi^{\pi} = {\scriptsize
	\begin{pmatrix}
		\e_{\pi(1)}^\top &       \textbf{0}^{\top}     &  \cdots & \textbf{0}^{\top} \\
		       \textbf{0}^{\top}     & \e_{\pi(2)}^\top &  \cdots & \textbf{0}^{\top} \\
			  \vdots  &        \vdots    & \ddots & \vdots \\	
		        \textbf{0}^{\top}    &     \textbf{0}^{\top}        &    \cdots     & \e_{\pi(|\cS|)}^\top
	\end{pmatrix}  }.
\end{align}
	\item $\bm{P}^{\pi} \in \mathbb{R}^{|\cS| |\cA|\times |\cS| |\cA|}$ and $\bm{P}_{\pi} \in \mathbb{R}^{|\cS| \times |\cS| }$: two {\em square} probability transition matrices induced by the policy $\pi$ over the state-action pairs and the states respectively, defined by
	\begin{equation}
	\label{eqn:ppivq}
		\Pq \defn \bP \bPi^{\pi} \qquad \text{and} \qquad \Pv \defn \bPi^{\pi} \bP .
	\end{equation}

\item $\bm{r}_{\pi} \in \mathbb{R}^{|\cS|}$: a reward vector restricted to the actions chosen by the policy $\pi$, namely,  $r_{\pi}(s) = r(s,\pi(s))$ for all $s\in \cS$ (or simply, $\bm{r}_{\pi}=\bm{\Pi}^{\pi}\bm{r}$). 
\end{itemize}

\noindent Armed with the above matrix notation, we can write, for any policy $\pi$, the {\em Bellman consistency equation} as
\begin{equation} \label{eq:bellman_consistency_q}
\bm{Q}^{\pi}  = \bm{r} + \gamma \bm{P} \bm{V}^{\pi} = \bm{r} + \gamma \Pq  \bm{Q}^{\pi} ,
\end{equation}
which implies that
\begin{align}
	& \bm{Q}^{\pi} = (\bm{I} -  \gamma \Pq)^{-1} \bm{r} ; \\
	\bm{V}^{\pi} =  \bm{r}_{\pi} + \gamma \Pv & \bm{V}^{\pi} 
	\qquad \text{and} \qquad
	\bm{V}^{\pi}  = (\bm{I} -  \gamma \Pv)^{-1}    \bm{r}_{\pi} .  \label{eq:bellman_consistency_v}
\end{align}
%
For a vector $\bm{V}=[V_i]_{1\leq i\leq |\mathcal{S}|}\in\real^{|\cS|}$, we define the vector $\Var_{\bP}(\bm{V}) \in \real^{|\cS||\cA|}$ whose entries are given by 
\begin{align*}
	\forall (s,a)\in \cS \times \cA, \qquad \big[\Var_{\bP}(\bm{V})\big]_{(s,a)}: = \sum_{s'\in \cS} P(s'| s,a) V_{s'}^2 
	-  \Big( \sum_{s'\in \cS}P(s'| s,a) V_{s'} \Big)^2,
\end{align*}
i.e.~the variance of $\bm{V}$ w.r.t.~$P(\cdot| s,a)$. This can be expressed using our matrix notation as follows
\begin{equation}
\Var_{\bP}(\bm{V}) = \bm{P} (\bm{V}\circ\bm{V})-(\bm{P} \bm{V})\circ(\bm{P} \bm{V}).
	\label{eq:matrix-VarP-V-expression}
\end{equation} 
Similarly, for any given policy $\pi$ we define 
\begin{equation}
	\Var_{\bP_{\pi}}(\bm{V}) =  \bm{\Pi}^{\pi} \Var_{\bP}(\bm{V}) = \bm{P}_{\pi} (\bm{V}\circ\bm{V})-(\bm{P}_{\pi} \bm{V})\circ(\bm{P}_{\pi} \bm{V}) \in \mathbb{R}^{|\mathcal{S}|}.
	\label{eq:matrix-VarPpi-V-expression}
\end{equation} 

We shall also define $\widehat{\bm{V}}^{\pi}$, $\widehat{\bm{Q}}^{\pi}$, $\widehat{\bm{V}}^{\star}$, $\widehat{\bm{Q}}^{\star}$, $\widehat{\bm{P}}$, $\widehat{\bm{P}}^{\pi}$, $\widehat{\bm{P}}_{\pi}$,  $\Var_{\widehat{\bP}}(\bm{V})$, $\Var_{\widehat{\bP}_{\pi}}(\bm{V})$ w.r.t.~the empirical MDP $\widehat{\mathcal{M}}$ in an analogous fashion.

\subsection{Analysis: model-based policy evaluation}
\label{sec:estimation-error-fixed-policy}

We start with the simpler task of policy evaluation, which also plays a crucial role in the analysis of planning. 
To establish our guarantees in Theorem~\ref{Thm:very-wise-evaluation}, we aim to prove the following result. Here, we recall that the true value function under a policy $\pi$ and the model-based empirical estimate are given respectively by
\begin{align}
	\bm{V}^{\pi}=(\bm{I}-\gamma\Pv)^{-1}\bm{r}_{\pi} \qquad \text{and} \qquad \widehat{\bm{V}}^{\pi}=(\bm{I}-\gamma\Phatv)^{-1}\bm{r}_{\pi}.
	\label{eq:Vpi-Vhatpi-notation}
\end{align}
\begin{lemma} 
\label{lemma:fixed-policy-error} 
	Fix any policy $\pi$. Consider any $0<\delta<1$, and suppose $N\geq\frac{32e^2}{1-\gamma} \log\big(\frac{4|\mathcal{S}|\log (\frac{e}{1-\gamma})}{\delta}\big) $.
	Then with probability at least $1-\delta$, the vectors defined in \eqref{eq:Vpi-Vhatpi-notation} obey
\begin{align}
	\big\|\widehat{\bm{V}}^{\pi}-\bm{V}^{\pi}\big\|_{\infty} & \leq 
 	4 \gamma \sqrt{\frac{2\log\big(\frac{4|\mathcal{S}|\log(\frac{e}{1-\gamma})}{\delta}\big)}{N}} 
 	  \Big\|(\bm{I}-\gamma\Pv)^{-1}\sqrt{\mathsf{Var}_{\Pv}\big[\bm{V}^{\pi}\big]}\Big\|_{\infty} + 
	\frac{2\gamma \log\big(\frac{4|\mathcal{S}|\log(\frac{e}{1-\gamma})}{\delta}\big)}{(1-\gamma)N}\big\|\bm{V}^{\pi} \big\|_{\infty} 
	\nonumber\\ 
	& \leq 6  \sqrt{\frac{2\log\big(\frac{4|\mathcal{S}|\log(\frac{e}{1-\gamma})}{\delta}\big)}{N(1-\gamma)^3}}. \label{eq:minimax-V}
\end{align}
\end{lemma}
\begin{proof}
	The key proof idea is to resort to a high-order successive expansion of  $\widehat{\bm{V}}^{\pi}-\bm{V}^{\pi}$, followed by fine-grained analysis of each term up to a certain logarithmic order. 
	See Appendix~\ref{sec:proof-lemma-fixed-policy-error}. 
\end{proof}
Clearly, Theorem~\ref{Thm:very-wise-evaluation} is a straightforward consequence of  Lemma~\ref{lemma:fixed-policy-error}. 
Further, we strengthen the result by providing an additional instance-dependent bound (see the first line of \eqref{eq:minimax-V} that depends on the true instance $\bm{P}_{\pi},\bm{V}^{\pi}$), which is often tighter than the worst-case bound stated in the second line of \eqref{eq:minimax-V}. 
Our contribution can be better understood when compared with \cite{pananjady2019value}.
Assuming that there is no noise in the rewards, our instance-dependent guarantee  matches \citet[Theorem~1(a)]{pananjady2019value} up to some $\log\log\frac{1}{1-\gamma}$ factor, while being capable of covering the full sample size range $N\geq \widetilde{\Omega}(\frac{1}{1-\gamma})$. In contrast, \citet[Theorem 1]{pananjady2019value} is only valid when $N\geq \widetilde{\Omega}(\frac{1}{(1-\gamma)^2}$).

\paragraph{Proof ideas.} We now briefly and informally describe the key proof ideas. 
As a starting point, the elementary identities \eqref{eq:Vpi-Vhatpi-notation}
allow us to obtain
\begin{align}
\widehat{\bm{V}}^{\pi}-\bm{V}^{\pi} & =\big(\bm{I}-\gamma\widehat{\bm{P}}_{\pi}\big)^{-1}\bm{r}_{\pi}-\bm{V}^{\pi} \nonumber\\
 & =\big(\bm{I}-\gamma\widehat{\bm{P}}_{\pi}\big)^{-1}\big(\bm{I}-\gamma\bm{P}_{\pi}\big)\bm{V}^{\pi}-\big(\bm{I}-\gamma\widehat{\bm{P}}_{\pi}\big)^{-1}\big(\bm{I}-\gamma\widehat{\bm{P}}_{\pi}\big)\bm{V}^{\pi} \nonumber\\
 & =\gamma\big(\bm{I}-\gamma\widehat{\bm{P}}_{\pi}\big)^{-1}\big(\widehat{\bm{P}}_{\pi}-\bm{P}_{\pi}\big)\bm{V}^{\pi}.
	\label{eq:V-basic-expansion}
\end{align}
Due to the complicated dependency between $(\bm{I}-\gamma\widehat{\bm{P}}_{\pi})^{-1}$ and $\big(\widehat{\bm{P}}_{\pi}-\bm{P}_{\pi}\big)\bm{V}^{\pi}$, a natural strategy is to control these two terms separately and then to combine bounds;  see \citet[Lemma 5]{agarwal2019optimality} for an introduction. This simple approach, however, leads to sub-optimal statistical guarantees.

In order to refine the statistical analysis, we propose to further expand \eqref{eq:V-basic-expansion} in a similar way to deduce
\begin{align}
\eqref{eq:V-basic-expansion} & =\gamma\big(\bm{I}-\gamma\bm{P}_{\pi}\big)^{-1}\big(\widehat{\bm{P}}_{\pi}-\bm{P}_{\pi}\big)\bm{V}^{\pi}+\gamma\left\{ \big(\bm{I}-\gamma\widehat{\bm{P}}_{\pi}\big)^{-1}-\big(\bm{I}-\gamma\bm{P}_{\pi}\big)^{-1}\right\} \big(\widehat{\bm{P}}_{\pi}-\bm{P}_{\pi}\big)\bm{V}^{\pi}\nonumber \\
 & =\gamma\big(\bm{I}-\gamma\bm{P}_{\pi}\big)^{-1}\big(\widehat{\bm{P}}_{\pi}-\bm{P}_{\pi}\big)\bm{V}^{\pi}+
	\gamma^{2}\big(\bm{I}-\gamma\widehat{\bm{P}}_{\pi}\big)^{-1} \big(\widehat{\bm{P}}_{\pi}-\bm{P}_{\pi}\big) \big(\bm{I} - \gamma \bm{P}_{\pi} \big)^{-1} \big(\widehat{\bm{P}}_{\pi}-\bm{P}_{\pi}\big) \bm{V}^{\pi},\label{eq:eq:V-basic-expansion-2nd}
\end{align}
where the last line holds due to the same reason as \eqref{eq:V-basic-expansion} (basically it can be seen by replacing $\bm{r}_{\pi}$ with $\big(\widehat{\bm{P}}_{\pi}-\bm{P}_{\pi}\big)\bm{V}^{\pi}$ in \eqref{eq:V-basic-expansion}). This can be viewed as a ``second-order'' expansion, with \eqref{eq:V-basic-expansion} being a ``first-order'' counterpart. 
The advantage is that: the first term in \eqref{eq:eq:V-basic-expansion-2nd} becomes easier to cope with than its counterpart \eqref{eq:V-basic-expansion}, owing to the independence between $(\bm{I}-\gamma{\bm{P}}_{\pi})^{-1}$ and $\big(\widehat{\bm{P}}_{\pi}-\bm{P}_{\pi}\big)\bm{V}^{\pi} $. However, the second term in \eqref{eq:eq:V-basic-expansion-2nd} remains difficult to control optimally. To remedy this issue, we shall continue to expand it to higher order (up to some logarithmic order), which eventually allows for optimal control of the estimation error. 

Another crucial issue is that: in order to obtain fine-grained analyses on each term in the expansion (except for the first-order term), 
a common approach is to combine the Bernstein inequality with a classical entrywise bound on a quantity taking the form $(\bm{I}-\gamma\bm{P}_{\pi})^{-1}\sqrt{\mathsf{Var}_{\bm{P}_{\pi}}(\bm {V})}$ (which dates back to \cite{azar2013minimax}). Such a classical bound in prior literature, however, is not sufficiently tight for our purpose, which calls for refinement; see Lemma~\ref{lemma:VarP-V-bound}. Details are deferred to Appendix~\ref{sec:proof-lemma-fixed-policy-error}.

%
%
%
%
%

\subsection{Analysis: perturbed model-based planning}
\label{sec:analysis}


This subsection moves on to establishing our theory for model-based planning (cf.~Theorem~\ref{Thm:sample-compl-main}) and 
outlines the key ideas. In what follows, we shall start by analyzing the unperturbed version, which will elucidate the role of reward perturbation in our analysis.

We first make note of the following elementary decomposition:
\begin{align}
	\Vstar - \bm{V}^{\pihatstar} &= \big( \Vhat^{\pihatstar} - \bm{V}^{\pihatstar}\big) + \big( \Vhat^{\pistar}- \Vhat^{\pihatstar} \big)
	+ \big( \Vstar - \Vhat^{\pistar} \big) \nonumber\\
	&\leq  \big(\Vhat^{\pihatstar} - \bm{V}^{\pihatstar} \big) + \big( \bm{V}^{\pistar} - \Vhat^{\pistar}  \big), 
	\label{eq:Vhat-Vstar-pistar-tot}
\end{align}
where the inequality follows from the optimality of $\pihatstar$ w.r.t.~$\Vhat$ (so that $\Vhat^{\pistar} \leq \Vhat^{\pihatstar}$) and the definition $\bm{V}^{\star}=\bm{V}^{\pistar}$. This leaves us with two terms to control.

\paragraph{Step 1: bounding $\|\bm{V}^{\pistar} - \Vhat^{\pistar}\|_{\infty}$.} Given that $\pi^{\star}$ is independent of the data, we can carry out this step using Lemma~\ref{lemma:fixed-policy-error}. Specifically, taking $\pi=\pi^{\star}$ in Lemma~\ref{lemma:fixed-policy-error}  yields that, with probability at least $1-\delta$, 
\begin{align}
	\big\|\widehat{\bm{V}}^{\pi^{\star}}-\bm{V}^{\pi^{\star}}\big\|_{\infty} & \leq 6  \sqrt{\frac{2\log\big(\frac{4|\mathcal{S}|\log\frac{e}{1-\gamma}}{\delta}\big)}{N(1-\gamma)^3}}.
	\label{eq:optimal-Vpistar-bound}
\end{align}
%


\paragraph{Step 2: bounding $\|\Vhat^{\pihatstar} - \bm{V}^{\pihatstar}\|_{\infty}$.} 
Extending the result in Step 1 to $\|\Vhat^{\pihatstar} - \bm{V}^{\pihatstar}\|_{\infty}$ is considerably more challenging, primarily due to the complicated statistical dependency between $(\bm{V}^{\pihatstar},\Vhat^{\pihatstar})$ and the data matrix $\widehat{\bm{P}}$. The recent work \cite{agarwal2019optimality} developed a clever ``leave-one-out'' type argument by constructing some auxiliary state-absorbing MDPs to decouple the statistical dependency when $\varepsilon <  1/\sqrt{1-\gamma}$. However, their argument falls short of accommodating the full range of $\varepsilon$.  To address this challenge, our analysis consists of the following two steps, both of which require new ideas beyond \cite{agarwal2019optimality}. 
\begin{itemize}
	\item {\em Decoupling statistical dependency between $\pihatstar$ and $\widehat{\bm{P}}$.} Instead of attempting to decouple the statistical dependency between $\Vhat^{\pihatstar}$ and $\Phat$ as in \cite{agarwal2019optimality}, we focus on decoupling the statistical dependency between the policy $\pihatstar$ and $\widehat{\bm{P}}$. If this can be achieved, then the proof strategy adopted in Step 1 for a fixed policy becomes applicable (see Section~\ref{sec:value-estimation-Bernstein}). A key ingredient of this step lies in the construction of a collection of auxiliary state-action-absorbing MDPs (motivated by \cite{agarwal2019optimality}), which allows us to get hold of  $\|\bm{V}^{\pihatstar} - \Vhat^{\pihatstar}\|_{\infty}$. See Section~\ref{sec:unique-policy} for details, with a formal bound delivered in Lemma~\ref{lem:gap-separation-condition}.  

	\item {\em Tie-breaking via reward perturbation.} A shortcoming of the above-mentioned approach, however, is that it relies crucially on the separability of $\pihatstar$ from other policies; in other words, the proof might fail if $\pihatstar$ is non-unique or not sufficiently distinguishable from others. Consequently, it remains to ensure that the optimal policy $\pihatstar$ stands out from all the rest for all MDPs of interest. As it turns out, this can be guaranteed with high probability by slightly perturbing the reward function so as to break the ties. See Section~\ref{sec:tie-breaking} for details. 

\end{itemize}
In the sequel, we shall flesh out these key ideas.

\subsubsection{Value function estimation for a policy obeying Bernstein-type conditions}
\label{sec:value-estimation-Bernstein}

Before discussing how to decouple statistical dependency, we record a useful result that plays an important role in the analysis. Specifically, 
Lemma~\ref{lemma:fixed-policy-error} can be generalized beyond the family of fixed policies (namely, those independent of $\widehat{\bm{P}}$), as long as a certain Bernstein-type condition --- to be formalized in \eqref{eq:Bernstein-Vl-general} --- is satisfied. To make it precise, we need to introduce a set of auxiliary vectors as follows
\begin{align}
\label{defn-r-v-iterate}
\begin{array}{ll}
	\bm{r}^{(0)}:=\bm{r}_{\pi}, & \bm{V}^{(0)}:=(\bm{I}-\gamma\Pv)^{-1}\bm{r}^{(0)}, \\
\bm{r}^{(l)}:=\sqrt{\mathsf{Var}_{\Pv}\big[\bm{V}^{(l-1)}\big]}, \quad & \bm{V}^{(l)}:=(\bm{I}-\gamma\Pv)^{-1}\bm{r}^{(l)},  \qquad   l\geq 1.
\end{array}
\end{align}
Our generalization of Lemma~\ref{lemma:fixed-policy-error} is as follows, which does \textit{not} require statistical independence between the policy $\pi$ and the data  $\widehat{\bm{P}}$. Here, we remind the reader of the notation  $|\bm{z}|:=[|z_1|,\cdots,|z_n|]^{\top}$ and $\sqrt{\bm{z}}:=[\sqrt{z_1},\cdots,\sqrt{z_n}]^{\top}$  for any vector $\bm{z}\in \mathbb{R}^n$.  

\begin{lemma}
\label{lemma:dependent-policy-error}
	Suppose that there exists some quantity $\beta_1>0$ such that $\{\bm{V}^{(l)}\}$ (cf.~\eqref{defn-r-v-iterate}) obeys
\begin{equation}
	\left|\big(\Phatv-\Pv\big)\bm{V}^{(l)}\right|\leq\sqrt{\frac{\beta_1}{N}}\sqrt{\mathsf{Var}_{\Pv}\big[\bm{V}^{(l)}\big]}+\frac{\beta_1 \left\|\bm{V}^{(l)}\right\|_{\infty}}{N}\bm{1}, \qquad \text{for all }0\leq l\leq \log \Big(\frac{e}{1-\gamma}\Big) .
	\label{eq:Bernstein-Vl-general}
\end{equation}
Suppose that $N>\frac{16e^{2}}{1-\gamma}\beta_{1}$. Then the vectors $\bm{V}^{\pi}=(\bm{I}-\gamma\Pv)^{-1}\bm{r}_{\pi}$ and $\widehat{\bm{V}}^{\pi}=(\bm{I}-\gamma\Phatv)^{-1}\bm{r}_{\pi}$ satisfy  
\begin{align}
	\big\|\widehat{\bm{V}}^{\pi}-\bm{V}^{\pi}\big\|_{\infty} 
	\leq \frac{6}{1-\gamma} \sqrt{\frac{\beta_{1}}{N(1-\gamma)}}. 
\end{align}
\end{lemma}
While the Bernstein-type condition \eqref{eq:Bernstein-Vl-general} clearly holds for some reasonably small $\beta_1$ if $\pi$ is independent of $\widehat{\bm{P}}$, it might remain valid if $\pi$  exhibits fairly ``weak'' statistical dependency on the data samples. This is a key step that paves the way for our subsequent analysis of $\pihatstar$.

\subsubsection{Decoupling statistical dependency via $(s,a)$-absorbing MDPs}
\label{sec:unique-policy}

We are now positioned to demonstrate how to control $\big\|\widehat{\bm{V}}^{\pihatstar}-\bm{V}^{\pihatstar}\big\|_{\infty}$ 
w.r.t.~the optimal policy $\pihatstar$ to $\widehat{\bm V}.$ A crucial technical challenge lies in how to decouple the complicated statistical dependency between the optimal policy $\pihatstar$ and the $\widehat{\bm V}^{\star}$
(which heavily relies on the data samples). 
Towards this, we resort to a leave-one-row-out argument built upon a collection of auxiliary MDPs, largely motivated by  the novel construction in \cite[Section 4.2]{agarwal2019optimality}. In comparison to \cite{agarwal2019optimality} that introduces state-absorbing MDPs (so that a state $s$ is absorbing regardless of the subsequent actions chosen),  our construction is a set of state-action-absorbing MDPs, in which a state $s$ is absorbing only when a designated action $a$ is always executed at the state $s$.

\paragraph{Construction of $(s,a)$-absorbing MDPs.} 
For each state-action pair $(s,a)$ and each scalar  $u$ with $|u| \leq 1/(1-\gamma)$, we construct an auxiliary MDP $\MDP_{s,a,u}$ --- it is identical to the original $\MDP$ except that it is absorbing in state $s$ if we always choose action $a$ in state $s$.  
More specifically, the probability transition kernel associated with $\MDP_{s,a,u}$ (denoted by $P_{\MDP_{s,a,u}}$) can be specified by
\begin{align}
\label{eqn:brahms-awesome}
	\begin{array}{rll}
		P_{\mathcal{M}_{s,a,u}}(s\mid s,a)&=1, \qquad  \\
		P_{\mathcal{M}_{s,a,u}}(s'\mid s,a)&=0, \qquad  &\text{for all }s' \neq s, \\
		P_{\mathcal{M}_{s,a,u}}(\cdot\mid s',a')&=P_{\mathcal{M}}(\cdot\mid s',a'), \qquad &\text{for all }(s',a') \neq (s,a),
	\end{array}
\end{align}
where $P_{\MDP}$ is the probability transition kernel w.r.t.~the original $\MDP.$
Meanwhile, the instant reward received at $(s,a)$ in $\MDP_{s,a,u}$ is set to be $u$, while the rewards at all other state-action pairs stay unchanged. We can define $\widehat{\MDP}_{s,a,u}$ analogously (so that its probability transition matrix is identical to $\widehat{\bm{P}}$ except that the $(s,a)$-th row becomes absorbing). The main advantage of this construction is that: for any fixed $u$, the MDP $\widehat{\MDP}_{s,a,u}$ is statistically independent of $\widehat{\bm{P}}_{s,a}$ (the row of $\widehat{\bm{P}}$ corresponding to the state-action pair $(s,a)$, determined by the samples collected for the $(s,a)$ pair).


To streamline notation, we let ${\bm Q}^{\pi}_{s,a,u}$ represent the Q-function of $\MDP_{s,a,u}$ under a policy $\pi$,  denote by $\pistar_{s,a,u}$ the optimal policy associated with $\MDP_{s,a,u}$, and let $\Qstar_{s,a,u}$ be the Q-function under this optimal policy $\pistar_{s,a,u}$. The notation ${\bm V}^{\pi}_{s,a,u}$ and $\Vstar_{s,a,u}$ regarding value functions, as well as their counterparts (i.e.~$\widehat{\bm Q}^{\pi}_{s,a,u}$, $\widehat{\bm Q}^{\star}_{s,a,u}$, $\widehat{\bm V}^{\pi}_{s,a,u}$, $\widehat{\bm V}^{\star}_{s,a,u}$, $\widehat{\pi}^{\star}_{s,a,u}$) in the empirical MDP $\widehat{\mathcal{M}}$, can be defined in an analogous fashion.

\begin{remark}
	The careful reader will remark that the instant reward $u$ is constrained to reside within $[-\frac{1}{1-\gamma},\frac{1}{1-\gamma}]$ rather than the usual range $[0,1]$. Fortunately, none of the subsequent steps that involve $u$ requires $u$ to lie within $[0,1]$. 
\end{remark}

\paragraph{Intimate connections between the auxiliary MDPs and the original MDP.}  
In the following, we introduce a result that connects the Q-function and the value function of the absorbing MDP with those of the original MDP. The idea is motivated by \citet[Lemma 7]{agarwal2019optimality} and its proof is deferred to Appendix~\ref{sec:proof-LemAbsToNormalQV}.

\begin{lemma}
\label{LemAbsToNormalQV}
	Setting $u^{\star} \defn r(s,a) + \gamma (\bP \Vstar)_{s,a} - \gamma V^\star(s)$, one has 
\begin{align}
\label{eqn:relation}
 	\Qstar_{s,a,u^{\star}} = \Qstar \qquad \text{and} \qquad \Vstar_{s,a,u^{\star}} = \Vstar.
\end{align} 
\end{lemma}
\begin{remark}
Lemma~\ref{LemAbsToNormalQV} does not rely on the particular form of $\bP$, and can be directly generalized 
	to the empirical model $\widehat{\bm P}$ and the auxiliary MDPs built upon $\widehat{\bm P}.$ 
\end{remark}

In words, by properly setting the instant reward $u=u^{\star}$ (which can be easily shown to reside within $[-\frac{1}{1-\gamma},\frac{1}{1-\gamma}]$), one guarantees that the $(s,a)$-absorbing MDP and the original MDP have the same Q-function and value function under the respective optimal policies.



\paragraph{Representing $\widehat{\pi}^{\star}$ via a small set of  policies independent of $\widehat{\bm{P}}_{s,a}$.}

With Lemma~\ref{LemAbsToNormalQV} in place, it is tempting to use $\widehat{\mathcal{M}}_{s,a,\widehat{u}^{\star}}$ with $\widehat{u}^{\star} \defn r(s,a) + \gamma (\Phat \Vhatstar)_{s,a} - \gamma \widehat{V}^\star(s)$ to replace the original  $\widehat{\mathcal{M}}$. The rationale is simple: given that the probability transition matrix of $\widehat{\mathcal{M}}_{s,a,\widehat{u}^{\star}}$ does not rely upon $\widehat{\bm{P}}_{s,a}$,  the statistical dependency between $\widehat{\mathcal{M}}_{s,a,\widehat{u}^{\star}}$ and $\widehat{\bm{P}}_{s,a}$ is now fully embedded into a single parameter $\widehat{u}^{\star}$.  This motivates us to decouple the statistical dependency effectively by constructing an epsilon-net (see, e.g., \cite{vershynin2018high}) w.r.t.~this single parameter. The aim is to locate a point $u_0$ over a small fixed set such that \emph{(i)} it is close to $\widehat{u}^{\star}$, and 
\emph{(ii)} its associated optimal policy is identical  to the original optimal policy $\widehat{\pi}^{\star}$.

It turns out that this aim can be accomplished as long as the original Q-function $\widehat{\bm{Q}}^{\star}$ satisfies a sort of separation condition (which indicates that there is no tie when it comes to the optimal policy). 
To make it precise, given any $0<\gap <1$, our separation condition is characterized through  the following event
\begin{align}
	\label{eq:defn-separation-event}
	\mathcal{B}_{\omega} := 
	\Big\{ 
		\widehat{Q}^\star(s, \widehat{\pi}^\star(s)) - \max_{a: a \neq \widehat{\pi}^\star(s)} \widehat{Q}^\star(s,a) \geq \gap ~~ \text{for all } s\in \mathcal{S} 
	\Big\}. 
\end{align}
Clearly, on the event $\mathcal{B}_{\omega}$, the optimal policy $\widehat{\pi}^\star$ is unique, since for each $s$ the action $\widehat{\pi}^\star(s)$ results in a strictly higher Q-value compared to any other action.
With this separation condition in mind, our result is stated below. Here and throughout, we define an epsilon-net of the interval 
$[-\frac{1}{1 - \gamma}, \frac{1}{1 - \gamma}]$ as follows 
\begin{align}
\label{eqn:net}
	\net{\epsilon} \defn \big\{
	-n_{\epsilon} \epsilon, \ldots, -\epsilon, 0,
	\epsilon, \ldots, n_{\epsilon} \epsilon \big\}, 
	\qquad \text{for the largest integer } n_\epsilon \text{ obeying } n_\epsilon \epsilon < \frac{1}{1-\gamma},
\end{align}
which has cardinality at most $\frac{2}{(1-\gamma)\epsilon}$. 
\begin{lemma}
\label{LemUnique}
	Consider any $\gap > 0$, and suppose the event $\mathcal{B}_{\omega}$ (cf.~\eqref{eq:defn-separation-event}) holds. 
	%
	Then for any pair $(s, a) \in \mathcal{S} \times \mathcal{A}$, there exists a point $u_0 \in \net{(1-\gamma)\gap/4}$, such that 
	\begin{align}
		\pihatstar = \pihatstar_{s,a,u_0}.
	\end{align}
\end{lemma}
\begin{proof} See Appendix~\ref{sec:proof-LemUnique}. \end{proof}

\paragraph{Deriving an optimal error bound under the separation condition.} Armed with the above bounds, we are ready to derive the desired error bound by combining Lemma~\ref{lemma:dependent-policy-error} and Lemma~\ref{LemUnique}.  
\begin{lemma}
	\label{lem:gap-separation-condition}
	Given $0<\gap <1$ and $\delta>0$, suppose that $\mathcal{B}_{\omega}$ (defined in \eqref{eq:defn-separation-event}) occurs with probability at least $1-\delta$. 
	Then with probability at least $1 - 3\delta$, 
\begin{align}
	\big\|\widehat{\bm{V}}^{\pihatstar}-\bm{V}^{\pihatstar}\big\|_{\infty} 
	&\leq 6 \sqrt{\frac{2\log\big(\frac{32 |\mathcal{S}||\mathcal{A}|}{(1-\gamma)^3\omega\delta} \big)}{N(1-\gamma)^3}}
	\quad \text{and} \quad
	\Vstar - \bm{V}^{\pihatstar} \leq 12 \sqrt{\frac{2\log\big(\frac{32 |\mathcal{S}||\mathcal{A}|}{(1-\gamma)^3\omega\delta} \big)}{N(1-\gamma)^3}} \, \one, \label{eqn:V-diff-fabulous}
\end{align}
	provided that $N \geq \frac{c_{0}\log\big(\frac{|\mathcal{S}||\mathcal{A}|}{(1-\gamma) 
		\delta\omega}\big)}{1-\gamma}$ for some sufficiently large constant $c_0>0$. 

\end{lemma}

\begin{proof}
See Appendix~\ref{sec:proof-lem:gap-separation-condition}.
\end{proof}

\subsubsection{A tie-breaking argument}
\label{sec:tie-breaking}

Unfortunately,  the separation condition specified in $\mathcal{B}_{\omega}$ (cf.~\eqref{eq:defn-separation-event}) does not always hold.  In order to accommodate all possible MDPs of interest without imposing such a special separation condition,  we put forward a perturbation argument  allowing one to generate a new MDP that  \textit{(i)} satisfies the separation condition, and that \textit{(ii)} is sufficiently close to the original MDP.  

Specifically, let us represent the proposed reward perturbation \eqref{eq:perturbed-reward-gone} in a vector form as follows
\begin{align}
\bm{r}_{\mathrm{p}}:=\bm{r}+\bm{\zeta}, \label{eq:perturbed-reward}
\end{align}
where $\bm{\zeta}=[\,\zeta(s,a)\,]_{(s,a)\in \mathcal{S}\times \mathcal{A}}$
is an $|\mathcal{S}||\mathcal{A}|$-dimensional vector composed of
independent entries with each $\zeta(s,a) \overset{\mathrm{i.i.d.}}{\sim}\mathsf{Unif}(0,\xi)$.
We aim to show that: by randomly perturbing the reward function, we
can ``break the tie'' in the Q-function and ensure sufficient separation
of Q-values associated with different actions. 

To formalize our result, we find it convenient to introduce 
additional notation. Denote by $\pi_{\mathrm{p}}^{\star}$ the optimal policy of the MDP $\mathcal{M}_{\mathrm{p}}=(\cS,\cA, \bm{P}, \bm{r}_{\mathrm{p}}, \gamma)$, and 
$Q_{\mathrm{p}}^{\star}$ its optimal state-action value function. We can define
$\widehat{Q}_{\mathrm{p}}^{\star}$ and $\widehat{\pi}_{\mathrm{p}}^{\star}$ analogously for the MDP $\widehat{\mathcal{M}}_{\mathrm{p}}=(\cS,\cA, \widehat{\bm{P}}, \bm{r}_{\mathrm{p}}, \gamma)$. Our result is phrased as follows. 

\begin{lemma}
\label{LemTieBreaking}
Consider the perturbed reward vector defined in expression \eqref{eq:perturbed-reward}. With probability at least $1-\delta$,
\begin{align}
	\forall (s,a)\in \cS \times \cA \text{ with }a\neq \pi_{\mathrm{p}}^\star(s) : 
	\qquad 
	Q_{{\mathrm{p}}}^{\star}(s,\pi_{\mathrm{p}}^\star(s))-Q_{{\mathrm{p}}}^{\star}(s,a)  
> \frac{\xi\delta(1-\gamma)}{3|\mathcal{S}||\mathcal{A}|^{2}}	.
\end{align}
This result holds unchanged if $(Q_{{\mathrm{p}}}^{\star},\pi_{\mathrm{p}}^\star)$ is replaced by $(\widehat{Q}_{{\mathrm{p}}}^{\star},\widehat{\pi}_{\mathrm{p}}^\star)$.
\end{lemma}

\begin{proof}
See Appendix~\ref{sec:proof-lem-tie-breaking}.
\end{proof}

Lemma~\ref{LemTieBreaking} reveals that at least a polynomially small degree of separation ($\omega =\frac{\xi\delta(1-\gamma)}{3|\mathcal{S}||\mathcal{A}|^{2}} $) arises upon random perturbation (with size $\xi$) of the reward function. As we shall see momentarily, this level of separation suffices for our purpose.

\subsubsection{Proof of Theorem~\ref{Thm:sample-compl-main}}



Let us consider the randomly perturbed reward function as in \eqref{eq:perturbed-reward}. For any policy $\pi$, we denote by $\bm{V}^{\pi}_{{\mathrm{p}}}$ (resp.~$\widehat{\bm{V}}^{\pi}_{{\mathrm{p}}}$) the corresponding value function vector in the MDP with probability transition matrix $\bm{P}$ (resp.~$\widehat{\bm{P}}$) and reward vector $\bm{r}_{\mathrm{p}}$. Note that ${\pi}_{\mathrm{p}}^{\star}$ (resp.~$\widehat{\pi}_{\mathrm{p}}^{\star}$) denotes the optimal policy that maximizes  ${\bm{V}}^{\pi}_{{\mathrm{p}}}$ (resp.~$\widehat{\bm{V}}^{\pi}_{{\mathrm{p}}}$).

In view of Lemma~\ref{LemTieBreaking},  with probability at least $1-\delta$ one has the separation
\begin{align}
\label{eqn:enough-perturbation}
\Big|\widehat{Q}_{{\mathrm{p}}}^{\star}(s,\widehat{\pi}_{\mathrm{p}}^\star(s))-\widehat{Q}_{{\mathrm{p}}}^{\star}(s,a)\Big| 
> \frac{\xi\delta(1-\gamma)}{3|\mathcal{S}||\mathcal{A}|^{2}}	
\end{align}
uniformly over all $s$ and $a\neq \widehat{\pi}_{\mathrm{p}}^\star(s)$. 
With this separation in place, taking $\omega \defn \frac{\xi\delta(1-\gamma)}{3|\mathcal{S}||\mathcal{A}|^{2}}$ in Lemma~\ref{lem:gap-separation-condition} yields 
\begin{align}
\label{eqn:cat}
		\left\| \bm{V}_{\mathrm{p}}^{{\pi}_{\mathrm{p}}^{\star}} - \bm{V}_{\mathrm{p}}^{\widehat{\pi}_{\mathrm{p}}^{\star}} \right\|_{\infty} \leq 12\sqrt{\frac{2\log\big(\frac{96|\mathcal{S}|^{2}|\mathcal{A}|^{3}}{(1-\gamma)^{4}\xi\delta^{2}}\big)}{N(1-\gamma)^{3}}}.
\end{align}
In addition, the value functions under any policy $\pi$ obeys
\[
\bm{V}^{\pi}-\bm{V}_{\mathrm{p}}^{\pi}=\bPi^{\pi}\Big((\Ind-\bP^\pi)^{-1}\br-(\Ind-\bP^\pi)^{-1}\bm{r}_{\mathrm{p}}\Big),
\]
which taken collectively with the facts $\|\bm{r}-\bm{r}_{\mathrm{p}}\|_{\infty}\leq\xi$ and $\|(\bm{I}-\gamma\bm{P}^\pi)^{-1}\|_1 \leq \frac{1}{1-\gamma}$ gives
\begin{align*}
	 \big\|\bm{V}^{\pi}-\bm{V}_{\mathrm{p}}^{\pi}\big\|_{\infty} &\leq\|(\bm{I}-\gamma\bm{P}^\pi)^{-1}\|_{1}\|\bm{r}-\bm{r}_{\mathrm{p}}\|_{\infty}\leq\frac{1}{1-\gamma}\xi .
\end{align*}
Specializing the above relation to $\pi^{\star}$ and $\pihatstar_{\mathrm{p}}$ gives
\begin{align}
	& \big\|\bm{V}^{\pi^{\star}}-\bm{V}_{\mathrm{p}}^{\pi^{\star}}\big\|_{\infty} 
	\leq\frac{1}{1-\gamma}\xi \qquad {\text{and}} \qquad
\big\|\bm V^{\pihatstar_{\mathrm{p}}} - \bm V^{\pihatstar_{\mathrm{p}}}_{\mathrm{p}}\big\|_{\infty} \leq\frac{1}{1-\gamma}\xi. \label{eqn:bearcalm}
\end{align}

Now let us consider the following decomposition
\begin{align*}
	\bm V^{\pihatstar_{\mathrm{p}}} - \Vstar 
	&= \big( \bm V^{\pihatstar_{\mathrm{p}}} - \bm V^{\pihatstar_{\mathrm{p}}}_{\mathrm{p}} \big) + 
	  \big(\bm V^{\pihatstar_{\mathrm{p}}}_{\mathrm{p}} - \bm V^{\pistar_{\mathrm{p}}}_{\mathrm{p}} \big)
	+ \big(\bm V^{\pistar_{\mathrm{p}}}_{\mathrm{p}} -  \bm V^{\pistar}_{\mathrm{p}} \big)
	+ \big(\bm V^{\pistar}_{\mathrm{p}} - \Vstar \big) \\
	&\geq  \big(\bm V^{\pihatstar_{\mathrm{p}}} - \bm V^{\pihatstar_{\mathrm{p}}}_{\mathrm{p}} \big) + 
	       \big(\bm V^{\pihatstar_{\mathrm{p}}}_{\mathrm{p}} - \bm V^{\pistar_{\mathrm{p}}}_{\mathrm{p}} \big)
	+ \big(\bm V^{\pistar}_{\mathrm{p}} - \Vstar \big),
\end{align*}
where the last step follows from the optimality of $\pistar_{\mathrm{p}}$ w.r.t.~$\bm V_{\mathrm{p}}$.
Taking this collectively with the inequalities~\eqref{eqn:cat} and \eqref{eqn:bearcalm}, 
one shows that with probability greater than $1-3\delta$,
\begin{align*}
	 \bm V^{\pihatstar_{\mathrm{p}}} - \Vstar  \geq - \Bigg( \frac{2}{1 - \gamma} \xi  +
	12 \sqrt{\frac{2\log\big(\frac{96 |\mathcal{S}|^2|\mathcal{A}|^3}{\xi (1-\gamma)^4\delta^2} \big)}{N(1-\gamma)^3}} \Bigg) \,\one
\end{align*}
%
By taking $\xi = \frac{(1 - \gamma)\varepsilon}{3|\mathcal{S}|^5|\mathcal{A}|^5}$ and 
$N \geq \frac{ c_0\log\left(\frac{|\mathcal{S}||\mathcal{A}|}{(1-\gamma)\delta\varepsilon} \right)}{(1 - \gamma)^3\varepsilon^2} 
	$ 
for some constant $c_{0}>0$ large enough, we can ensure that $\bm{0}\geq \bm V^{\pihatstar_{\mathrm{p}}} - \Vstar \geq - \varepsilon \one$ as claimed. 
Regarding the Q-functions,   the Bellman equation gives 
\begin{align*}
	\bm Q^{\pihatstar_{\mathrm{p}}} - \Qstar  = \bm{r} + \gamma \bm{P} \bm{V} ^{\pihatstar_{\mathrm{p}}} - \big(\bm{r} + \gamma \bm{P} \bm{V} ^{\star}\big) = \gamma \bP (\bm V^{\pihatstar_{\mathrm{p}}} - \Vstar).
\end{align*}
Consequently, one has 
\begin{align*}
	\bm Q^{\pihatstar_{\mathrm{p}}} - \Qstar \geq - \big( \gamma \|\bP\|_1 \|\bm{V}^{\pihatstar_{\mathrm{p}}} - \Vstar\|_\infty \big) \,\one
	\geq  - \gamma \varepsilon \one.
\end{align*}


Finally, we demonstrate that both the empirical QVI and PI w.r.t.~$\widehat{\mathcal{M}}_{\mathrm{p}}$ are guaranteed to find $\pihatstar_{\mathrm{p}}$ in a few iterations. 
Suppose for the moment that we can obtain a policy $\pi_{k}$ obeying 
\begin{align}
	\|\widehat{\bm{Q}}_{\mathrm{p}}^{\pi_{k}}-\widehat{\bm{Q}}_{\mathrm{p}}^{\star}\|_{\infty}<\frac{\xi\delta(1-\gamma)}{8|\mathcal{S}||\mathcal{A}|^{2}}.
	\label{eq:Qpik-opt-accuracy}
\end{align}
Then for any $s\in\mathcal{S}$ and any action $a\neq\widehat{\pi}_{\mathrm{p}}^{\star}(s)$
one has
\begin{align*}
\widehat{\bm{Q}}_{\mathrm{p}}^{\pi_{k}}\big(s,\widehat{\pi}_{\mathrm{p}}^{\star}(s)\big)&-\widehat{\bm{Q}}_{\mathrm{p}}^{\pi_{k}} (s,a) \\
& =\widehat{\bm{Q}}_{\mathrm{p}}^{\star}\big(s,\widehat{\pi}_{\mathrm{p}}^{\star}(s)\big)-\widehat{\bm{Q}}_{\mathrm{p}}^{\star} (s,a)+\Big(\widehat{\bm{Q}}_{\mathrm{p}}^{\pi_{k}}\big(s,\widehat{\pi}_{\mathrm{p}}^{\star}(s)\big)-\widehat{\bm{Q}}_{\mathrm{p}}^{\star}\big(s,\widehat{\pi}_{\mathrm{p}}^{\star}(s)\big)\Big)-\Big(\widehat{\bm{Q}}_{\mathrm{p}}^{\pi_{k}}(s,a)-\widehat{\bm{Q}}_{\mathrm{p}}^{\star} (s,a)\Big)\\
 & \geq\widehat{\bm{Q}}_{\mathrm{p}}^{\star}\big(s,\widehat{\pi}_{\mathrm{p}}^{\star}(s)\big)-\widehat{\bm{Q}}_{\mathrm{p}}^{\star}(s,a)-2\big\|\widehat{\bm{Q}}_{\mathrm{p}}^{\pi_{k}}-\widehat{\bm{Q}}_{\mathrm{p}}^{\star}\big\|_{\infty}\\
 & >\frac{\xi\delta(1-\gamma)}{4|\mathcal{S}||\mathcal{A}|^{2}}-2\cdot\frac{\xi\delta(1-\gamma)}{8|\mathcal{S}||\mathcal{A}|^{2}}=0,
\end{align*}
where the last line results from \eqref{eqn:enough-perturbation} and \eqref{eq:Qpik-opt-accuracy}. 
In other words, we can perfectly recover the policy $\widehat{\pi}_{\mathrm{p}}^{\star}$
from the estimate $\widehat{\bm{Q}}_{\mathrm{p}}^{\pi_{k}}$, provided that \eqref{eq:Qpik-opt-accuracy} is satisfied. In addition, it has been shown that \cite[Lemma 2]{azar2013minimax} the greedy policy induced by $k$-th iteration of both algorithms --- denoted by $\pi_k$ --- satisfies 
	$\big\|\Qhat_{{\mathrm{p}}}^{\pi_k} - \Qhat_{{\mathrm{p}}}^{\star}\big\|_{\infty} \leq \frac{2\gamma^{k+1}}{(1-\gamma)^2}$.
Taking $\xi = \frac{c_1(1-\gamma)\varepsilon}{|\mathcal{S}|^5|\mathcal{A}|^5}$ and $k=\frac{c_2}{1-\gamma}\log(\frac{|\cS||\cA|}{(1-\gamma)\varepsilon\delta}) $ for some constant $c_2>0$ large enough, 
one guarantees that $\pi_k$ satisfies \eqref{eq:Qpik-opt-accuracy}, which in turn ensures 
perfect recovery of $\widehat{\pi}^{\star}_{\mathrm{p}}$. 

\subsection{Analysis: conservative model-based planning}
\label{sec:decouple-no-perturb}


In view of the conservative model-based planning~\eqref{eq:policy-Q-approx-MDP}, 
we begin with the following decomposition:
\begin{align}
	\Vstar - \bm{V}^{\picon} &= 
	\big( \bm{V}^{\pistar}  - \Vhat^{\pistar} \big) + \big( \Vhat^{\pistar}- \Vhat^{\picon} \big) + 
	\big( \Vhat^{\picon} - \bm{V}^{\picon}\big). 
	\label{eq:basic-decomposition-123}
\end{align}
In order to control the second term on the right-hand side of the above identity, we resort to the following lemma, whose proof is postponed to Section~\ref{sec:proof-inequality-upper-bound-picon-Vhat-Vstar}.
\begin{lemma}
\label{lem:upper-bound-picon-Vhat-Vstar}
It holds that 
\begin{align}
	 \Vhat^{\pistar}- \Vhat^{\picon}  \leq \frac{\xi }{1-\gamma} \one. 
	\label{eq:upper-bound-picon-Vhat-Vstar}
\end{align}
\end{lemma}
Combining Lemma~\ref{lem:upper-bound-picon-Vhat-Vstar} with \eqref{eq:basic-decomposition-123}, we arrive at
\begin{align}
	\Vstar - \bm{V}^{\picon} &= 
	\big( \bm{V}^{\pistar}  - \Vhat^{\pistar} \big) + \big( \Vhat^{\pistar}- \Vhat^{\picon} \big) + 
	\big( \Vhat^{\picon} - \bm{V}^{\picon}\big) 
	 \nonumber\\
	&\leq \big( \bm{V}^{\pistar} - \Vhat^{\pistar}  \big) +  \frac{\xi }{1-\gamma} \one + \big(\Vhat^{\picon} - \bm{V}^{\picon} \big) .
	\label{eq:decomposition-V-pic-hat}
\end{align}
%
%
Clearly, the first term of \eqref{eq:decomposition-V-pic-hat} has already been controlled in Section~\ref{sec:estimation-error-fixed-policy}, 
while the second term of \eqref{eq:decomposition-V-pic-hat} is extremely small
when we take $\xi =O\big(\frac{(1-\gamma)\varepsilon}{|\cS|^5|\cA|^5}\big)$. 
It thus suffices to bound the third term of \eqref{eq:decomposition-V-pic-hat}, which again requires decoupling the statistical dependence between $\picon$ and $\widehat{\bm P}$.

%
%
%
%

\paragraph{Representing $\picon$ via a small set of policies independent of $\widehat{\bm{P}}_{s,a}$.}

Akin to our analysis for the perturbed model-based planning algorithm in Section~\ref{sec:analysis}, a key step lies in demonstrating the connection between $\picon$ and a reasonably small collection of leave-one-out auxiliary MDPs.  
Towards this, we are in need of the following lemma,  
which characterizes certain ``stability'' of our conservative model-based strategy and lies at the core of our analysis. The proof is deferred to Section~\ref{sec:proof-lem:picon-is-stable}. 
\begin{lemma}
\label{lem:picon-is-stable}
	Consider any given Q-function $\widehat{Q}: \cS\times \cA \rightarrow \mathbb{R}$ and its associated value function $\widehat{V}: \cS \rightarrow \mathbb{R}$ (i.e.~$\widehat{V}(s)=\max_a \widehat{Q}(s,a)$ for all $s$).  Generate an independent random variable $\newnoise \sim \mathsf{Unif}(0,\xi)$. Then with probability at least $1-\delta$, 
\begin{align}
\label{eqn:pi-stable-nice} 
	\forall s \in \mathcal{S}, \qquad 
	\big\{a \in \mathcal{A}: \widehat{Q}(s,a) > \widehat{V}(s) - \newnoise \big\} 
	= \big\{ a \in \mathcal{A}: Q(s,a) > V(s) - \newnoise \big\}
\end{align}
holds simultaneously for all Q-function  $Q: \cS\times \cA \rightarrow \mathbb{R}$ (and its associated value function $V: \cS \rightarrow \mathbb{R}$) obeying 
\begin{equation}
	\max_{s,a} \big| Q(s,a) - \widehat{Q}(s,a) \big| \leq \frac{\xi\delta}{8|\mathcal{S}||\mathcal{A}|}.
	\label{eq:Q-Qhat-distance-bound}
\end{equation}
\end{lemma}
As an important implication of this lemma, the policy $\picon$ computed in \eqref{eq:policy-Q-approx-MDP} remains unchanged upon slight perturbation of the Q-function estimates.  
Armed with Lemma~\ref{lem:picon-is-stable} and the leave-one-out auxiliary MDPs $\{\widehat{\mathcal{M}}_{s,a,u}\}$ constructed in Section~\ref{sec:analysis}, our analysis proceeds as follows. 
\begin{itemize}
	\item For each $(s,a)\in \cS\times \cA$, there exists a point $u_0 \in \mathcal{N}_{(1-\gamma)\omega}$ (cf.~\eqref{eqn:net}) such that 
		the optimal Q-function of $\widehat{\bm{Q}}^{\star}_{s,a,u_0}$ of $\widehat{\mathcal{M}}_{s,a,u_0}$ obeys
		\begin{equation}
			\big\| \widehat{\bm{Q}}^{\star}_{s,a,u_0} - \widehat{\bm{Q}}^{\star} \big\|_{\infty} \leq \omega. 
			\label{eq:Qstar-Qhat-star-omega-dist}
		\end{equation}
		This is a fact that has already been established in the proof of Lemma~\ref{LemUnique}; see \eqref{eqn:watermelon}. 

	\item   Define a conservative policy for $\widehat{\mathcal{M}}_{s,a,u_0}$ as follows: 
		\[
			\forall s'\in \cS: \qquad 
			\widehat{\pi}_{s,a,u_0,\mathrm{c}} (s') 
			= \min \big\{ a'\in \cA: \widehat{Q}^{\star}_{s,a,u_0}(s',a') > \widehat{V}^{\star}_{s,a,u_0}(s') - \newnoise  \big\},
		\]
		where $\widehat{Q}^{\star}_{s,a,u_0}$ and $\widehat{V}^{\star}_{s,a,u_0}$ denote the optimal Q-function and optimal value function of $\widehat{\mathcal{M}}_{s,a,u_0}$, respectively. 
		Taking $\omega= \frac{\xi\delta}{8|\mathcal{S}||\mathcal{A}|}$ and invoking Lemma~\ref{lem:picon-is-stable} and \eqref{eq:Qstar-Qhat-star-omega-dist}, we arrive at
		\begin{equation}
			\picon = \widehat{\pi}_{s,a,u_0,\mathrm{c}}. 
			\label{eq:picon-equivalence-pi-sau}
		\end{equation}
		%
\end{itemize}
This result \eqref{eq:picon-equivalence-pi-sau} parallels Lemma~\ref{LemUnique} for the perturbed model-based planning algorithm, 
revealing that $\picon$ is representable using a policy independent of the randomness associated with $(s,a)$. 
The remaining proof of Theorem~\ref{Thm:sample-compl-conservative} then follows from an identical argument as in the proof of Theorem~\ref{Thm:sample-compl-main}, and is hence omitted here.

%
%
%
%
%
%
%

\section{Analysis: finite-horizon MDPs}

In this section, we outline the proof of Theorem~\ref{Thm:sample-compl-finite}. 
We shall start by introducing a set of convenient matrix notation before embarking on the main proof.

\subsection{Matrix notation and Bellman equations}

Akin to the infinite-horizon case, we introduce some matrix notation for finite-horizon MDPs. Analogous to Section~\ref{subsec:matrix-notation},  we introduce the following set of notation. 
\begin{itemize}
\item
$\bm{r}_h\in\mathbb{R}^{|\mathcal{S}||\mathcal{A}|}$: a vector representing the reward function $r_h$ at step $h$. 
\item 
$\bm{V}_h^{\pi}\in\mathbb{R}^{|\mathcal{S}|}$: a vector representing the value function $V_h^{\pi}$ of $\pi$ at step $h$.
\item
$\bm{V}_h^{\star}\in\mathbb{R}^{|\mathcal{S}|}$: a vector representing the optimal value function $V_h^{\star}$ at step $h$.
\item
$\bm{Q}_h^{\pi}\in\mathbb{R}^{|\mathcal{S}||\mathcal{A}|}$: a vector representing the Q-function $Q_h^{\pi}$ of $\pi$ at step $h$.
\item
$\bm{Q}_h^{\star}\in\mathbb{R}^{|\mathcal{S}||\mathcal{A}|}$: a vector representing the optimal Q-function ${Q}_h^{\star}$ at step $h$. 
\item
$\bm{P}_h\in\mathbb{R}^{|\mathcal{S}||\mathcal{A}|\times|\mathcal{S}|}$: a matrix representing the probability transition kernel $P_h$ at step $h$.
\item
$\bm P_{h,\pi}\in\mathbb{R}^{|\mathcal{S}|\times|\mathcal{S}|}$: a submatrix of $\bm{P}_h$,  which consists of the rows with indices coming from $\{(s,\pi_h(s))\mid s\in \cS\}$. 
\item
$\bm r_h^{\pi} \in\mathbb{R}^{|\mathcal{S}|}$: a subvector of $\bm r_h$, which consists of the rows with indices coming from $\{(s,\pi_h(s))\mid s\in \cS\}$.
\end{itemize}
Armed with the above notation, the Bellman equation here is given by
\begin{align}
\bm Q^{\pi}_h = \bm r_h + \bm P_h \bm V^{\pi}_{h+1}, \qquad 1\leq h\leq H, 
\end{align}
where we recall that for all $s \in \cS$
\begin{align}
V^{\pi}_h(s) = Q^{\pi}_h\big( s, \pi_h(s) \big) \quad \text{and} \quad
V^{\pi}_{H+1}(s) = 0.
\end{align}
%
%
This also allows one to derive
\begin{align}
\bm V^{\pi}_h = \bm r_h^{\pi} + \bm P_{h,\pi} \bm V^{\pi}_{h+1}. 
\end{align}
We shall also define $\widehat{\bm{V}}_h^{\pi}$, $\widehat{\bm{Q}}_h^{\pi}$, $\widehat{\bm{V}}_h^{\star}$, $\widehat{\bm{Q}}_h^{\star}$, $\widehat{\bm{P}}_h$, $\widehat{\bm{P}}_{h,\pi}$ w.r.t.~the empirical MDP $\widehat{\mathcal{M}}$ in an analogous fashion.

\subsection{An auxiliary value function sequence obeying Bernstein-type conditions}

Similar to the infinite-horizon case (in particular, Section~\ref{sec:value-estimation-Bernstein}), 
we find it convenient to introduce a collection of auxiliary vectors as follows.
For any $l\geq 0$, define
\begin{align}
	\bm{V}_{H+1}^{(l)} :=\bm 0 \qquad \text{and} \qquad \widehat{\bm{V}}_{H+1}^{(l)}:=\bm 0; 
	\label{defn-r-v-iterate-finite-Hplus1}
\end{align}
and for any $1 \le h \le H$ and any policy $\pi$, define the following sequences recursively: 
\begin{align}
\label{defn-r-v-iterate-finite}
\begin{array}{lll}
	\bm{r}_h^{(0)}:=\bm{r}_h^{\pi}, & \bm{V}_h^{(0)}:=\bm r_h^{(0)} + \bm P_{h,\pi} \bm V^{(0)}_{h+1}, & \widehat{\bm{V}}_h^{(0)}:=\bm r_h^{(0)} + \widehat{\bm P}_{h,\pi} \widehat{\bm V}^{(0)}_{h+1}, \\
	\bm{r}_h^{(l)}:=\sqrt{\mathsf{Var}_{\bm P_{h,\pi}}\big[\bm{V}^{(l-1)}_{h+1}\big]}, \quad & \bm{V}^{(l)}_h:=\bm r_h^{(l)} + \bm P_{h,\pi} \bm V^{(l)}_{h+1}, \quad & \widehat{\bm{V}}^{(l)}_h:=\bm r_h^{(l)} + \widehat{\bm P}_{h,\pi} \widehat{\bm V}^{(l)}_{h+1}, \qquad   l\geq 1.
\end{array}
\end{align}
As can be easily verified, $\{\bm{V}_h^{(0)}\}$ coincides with the value function of policy $\pi$ in the true MDP $\mathcal{M}$, 
while $\{\widehat{\bm{V}}_h^{(0)}\}$ corresponds to the value function of policy $\pi$ in the empirical MDP $\widehat{\mathcal{M}}$.

As it turns out, if the above auxiliary sequence satisfies certain Bernstein-type conditions, then 
we can establish a useful upper bound on the entrywise difference between $\bm{V}_h^{(0)}$ and $\widehat{\bm{V}}_h^{(0)}$, as stated below. 
The proof of this lemma is deferred to Section~\ref{sec:proof-lemma:dependent-policy-error-finite}. 
\begin{lemma}
\label{lemma:dependent-policy-error-finite}
Suppose that there exists some quantity $\beta_1>0$ such that the sequence $\{\bm{V}^{(l)}_{h+1}\}$ constructed in \eqref{defn-r-v-iterate-finite} obeys
\begin{equation}
	\left|\big(\widehat{\bm P}_{h,\pi} - \bm P_{h,\pi} \big)\bm{V}^{(l)}_{h+1}\right|\leq\sqrt{\frac{\beta_1}{N}}\sqrt{\mathsf{Var}_{\bm P_{h,\pi}}\big[\bm{V}^{(l)}_{h+1}\big]}+\frac{\beta_1 \big\|\bm{V}^{(l)}_{h+1}\big\|_{\infty}}{N}\bm{1}, \quad \forall 0\leq l\leq \log_2 H,\, 1 \le h \le H-1.
	\label{eq:Bernstein-Vl-general-finite}
\end{equation}
In addition, assume that $N>12H\beta_{1}$. Then we have  
\begin{align}
	\big\|\widehat{\bm{V}}^{(0)}_h - \bm{V}^{(0)}_h\big\|_{\infty} 
	\leq 6H \sqrt{\frac{3\beta_{1}H}{N}}
\end{align}
for all $1 \le h \le H$. 
\end{lemma} 
%

\subsection{Proof of Theorem \ref{Thm:sample-compl-finite}}

Let us begin with the following elementary decomposition:
\begin{align}
\bm{V}_{h}^{\star}-\bm{V}_{h}^{\widehat{\pi}^{\star}} & =\big(\Vhat_{h}^{\widehat{\pi}^{\star}}-\bm{V}_{h}^{\widehat{\pi}^{\star}}\big)+\big(\Vhat_{h}^{\pistar}-\Vhat_{h}^{\widehat{\pi}^{\star}}\big)+\big(\Vstar_{h}-\Vhat_{h}^{\pistar}\big)\nonumber\\
 & \leq\big(\Vhat_{h}^{\widehat{\pi}^{\star}}-\bm{V}_{h}^{\widehat{\pi}^{\star}}\big)+\big(\bm{V}_{h}^{\pistar}-\Vhat_{h}^{\pistar}\big). 
	\label{eq:Vhat-Vstar-pistar-tot-finite}
\end{align}
Here, the inequality follows from the definition $\bm{V}^{\star}_h=\bm{V}^{\pistar}_h$, as well as the fact that  $\Vhat^{\pistar}_h \leq \Vhat^{\widehat{\pi}^{\star}}_h$ (since $\widehat{\pi}^{\star}$ is the optimal policy of the empirical MDP). 
%
%
In light of \eqref{eq:Vhat-Vstar-pistar-tot-finite}, there are two terms that need to be controlled.

We intend to bound both $\Vhat^{\widehat{\pi}^{\star}}_h - \bm{V}^{\widehat{\pi}^{\star}}_h$ and $\bm{V}^{\pistar}_h - \Vhat^{\pistar}_h$ by means of Lemma \ref{lemma:dependent-policy-error-finite}. 
Towards this, we first note that for any policy $\pi$, the associated value functions of the MDP and the empirical MDP obey the following Bellman equations:
\begin{align*}
 & \bm{V}_{h}^{\pi}:=\bm{r}_{h}^{\pi}+\bm{P}_{h,\pi}\bm{V}_{h+1}^{\pi},\qquad\widehat{\bm{V}}_{h}^{\pi}:=\bm{r}_{h}^{\pi}+\widehat{\bm{P}}_{h,\pi}\widehat{\bm{V}}_{h+1}^{\pi},
\end{align*}
along with the boundary conditions
$\bm{V}_{H+1}^{\pi}=\widehat{\bm{V}}_{H+1}^{\pi}=\bm{0}$. 
This indicates that the vector $\bm{V}_h^{(0)}$ (resp.~$\widehat{\bm{V}}_h^{(0)}$) constructed in \eqref{defn-r-v-iterate-finite} is precisely the value function of policy $\pi$ at step $h$ in the true MDP (resp.~empirical MDP). 
As a result, in order to invoke Lemma \ref{lemma:dependent-policy-error-finite}, it is sufficient to verify the Bernstein-type condition \eqref{eq:Bernstein-Vl-general-finite} w.r.t.~policies $\pistar$ and $\widehat{\pi}^{\star}$ for some sufficiently small quantity $\beta_1$. 
\begin{itemize}
	\item Let us begin with the optimal policy $\pistar$, which is fixed and statistically independent of the data samples. 
		As a result, if we take $\pi=\pistar$ during the construction \eqref{defn-r-v-iterate-finite}, then it is clearly seen that 
		$\Phat_h$ is statistically independent of $\bm{V}_{h+1}^{(l)}$. 
		Applying the Bernstein inequality together with the union bound then guarantees that: with probability exceeding $1 - \delta$, 
\begin{align}
	\Big|(\Phat_h - \bP_h) \bm{V}_{h+1}^{(l)}\Big| \leq 
	\sqrt{\frac{\beta_1}{N}}\sqrt{  \mathsf{Var}_{\bP_h}\big[\bm{V}^{(l)}_{h+1}\big]  }+\frac{\beta_1\big\|\bm{V}_{h+1}^{(l)}\big\|_{\infty}}{N}\bm 1
	\label{eq:bernstein-pi-star-finite}
\end{align}
holds uniformly over all $0\leq l\leq  \log_2 H$, $1 \le h \le H-1$, $(s,a) \in \mathcal{S} \times \mathcal{A}$,
where $\beta_1$ is given by
\begin{align}
	\beta_{1}:=4\log\Big(\frac{H |\mathcal{S}||\mathcal{A}|}{\delta}\Big).
	\label{eq:size-beta1-finite-star}
\end{align}
Armed with Condition~\eqref{eq:bernstein-pi-star-finite}, we can readily invoke Lemma \ref{lemma:dependent-policy-error-finite} to reach 
\begin{align*}
	\big\|  \widehat{\bm{V}}_{h}^{\pi^{\star}} - \bm{V}_{h}^{\pi^{\star}} \big\|_{\infty} &= \big\|  \widehat{\bm{V}}_{h}^{(0)} - \bm{V}_{h}^{(0)} \big\|_{\infty}  \leq 6H \sqrt{\frac{3\beta_1 H}{N}} 
	 \leq 6H \sqrt{\frac{12 H \log\big(\frac{H |\mathcal{S}||\mathcal{A}|}{\delta}\big) }{N}}
\end{align*}
with probability at least $1-\delta$.

\item Next, we move on to the policy $\widehat{\pi}^{\star}$ by taking $\pi=\widehat{\pi}^{\star}$ during the construction \eqref{defn-r-v-iterate-finite}. 
Note that $\bm{V}_{h+1}^{(l)}$ depends only on $\widehat{\pi}_j^{\star}$ $(j\geq h+1)$. 
		In view of our assumption on $\widehat{\pi}^{\star}$ (i.e., it is computed backward via dynamic programming), $\widehat{\pi}_i^{\star}$ is independent of any $\widehat{\bm P}_j$ with $j<i$, 
		and hence    $\bm{V}_{h+1}^{(l)}$ is statistically independent from $\Phat_h$. Consequently, the preceding bounds \eqref{eq:bernstein-pi-star-finite} and \eqref{eq:size-beta1-finite-star} continue to hold. 
		All of this immediately results in
\begin{align*}
	\big\|  \widehat{\bm{V}}_{h}^{\widehat{\pi}^{\star}} - \bm{V}_{h}^{\widehat{\pi}^{\star}} \big\|_{\infty}  
	 \leq 6H \sqrt{\frac{12 H \log\big(\frac{H |\mathcal{S}||\mathcal{A}|}{\delta}\big) }{N}}
\end{align*}
with probability exceeding $1-\delta$.
\end{itemize}

Substituting the above bounds into \eqref{eq:Vhat-Vstar-pistar-tot-finite}, we arrive at
\begin{align}
	\bm{0} \leq \bm{V}_{h}^{\star}-\bm{V}_{h}^{\widehat{\pi}^{\star}} 
	\leq 12H \sqrt{\frac{12 H \log\big(\frac{H |\mathcal{S}||\mathcal{A}|}{\delta}\big) }{N}} \,\bm{1} 
	\label{eq:RHS-537}
\end{align}
with probability greater than $1-2\delta$, provided that $N\geq 48 H\log\big(\frac{H |\mathcal{S}||\mathcal{A}|}{\delta}\big)$. 
By taking the right-hand side of \eqref{eq:RHS-537} to be smaller than $\varepsilon \bm 1$, we immediately conclude the proof.

\section{Discussion}
\label{sec:discussion}

This paper has demonstrated that (some variants of) model-based planning algorithms achieve the minimax sample complexity in the presence of a generative model, as soon as
the sample size exceeds the order of $\frac{|\cS||\cA|}{1-\gamma}$ for $\gamma$-discounted infinite horizon MDPs and $|\cS||\cA|H^2$ for time-inhomogeneous finite-horizon MDPs (modulo some log factor). Compared to prior literature, our result has considerably broadened the sample size range,  allowing us to pin down a complete trade-off curve between sample complexity and statistical accuracy.


The present work opens up several directions for future investigation, which we discuss in passing below. 

\begin{itemize}
	\item {\em Is perturbation or conservative action selection necessary for infinite-horizon MDPs?} The planning algorithm analyzed here for infinite-horizon MDPs is either applied to a perturbed variant of the empirical MDP (as in perturbed model-based planning) or run in a conservative manner (as in conservative model-based planning).
		This, however, gives rise to a natural question regarding the necessity of perturbation or conservative action selection: can we achieve optimal performance directly using {\em plain} model-based planning on the empirical MDP? While we conjecture that the answer is affirmative, settling this conjecture requires new techniques beyond the analysis framework of this paper.

\item {\em Improved analysis for model-free algorithms.} As mentioned previously, a even more severe sample complexity barrier is present in all prior theory regarding model-free approaches (e.g.~\citet{wainwright2019variance,sidford2018near,li2020sample}). Our analysis might shed light on how to overcome such barriers for model-free approaches.

\item {\em Time-homogeneous finite-horizon MDPs.} 
When it comes to finite-horizon MDPs, the present work concentrates on time-inhomogeneous MDPs where the probability transition kernels may vary across time steps. 
		Another important scenario is concerned with time-homogeneous MDPs, where $P_1=P_2=\cdots=P_H$. It remains unclear how to develop tight sample analysis for time-homogeneous MDPs due to the lack of statistical independence across time steps (namely, we shall use all samples to estimate the kernels across time steps as they are identical). 

\item {\em Markovian sample trajectories.} 
Going beyond the generative model,  another common form of data samples takes the form of a Markovian sample trajectory, which is generated by taking actions according to a stationary behavior policy in the MDP. 
This is also referred to as the asynchronous setting in the context of Q-learning \citep{tsitsiklis1994asynchronous}. 
While the sample complexity of several RL algorithms under this data-generating mechanism has been studied in prior literature (e.g.~\citet{li2020sample,li2021q,qu2020finite}), 
it remains unclear how to achieve  minimax optimality for the full $\varepsilon$-range, due to the complicated statistical dependency across time. 
The recent work \citet{li2022settling} demonstrated the plausibility of converting a finite-horizon Markovian trajectory into independent samples via two-fold sample splitting in the context of offline RL. 
It would be interesting to investigate whether one could employ a similar idea --- in conjunction with a proper leave-one-out analysis framework --- to settle the sample complexity in the presence of Markovian samples.

\item {\em Online exploratory RL.} 
	In practice, there is no shortage of applications where the learner acquires data samples by executing the MDP in real time. 
		This corresponds to an important setting, called online RL, that requires careful managing of the exploration-exploitation tradeoff  \citep{jin2018q,bai2019provably,li2021breaking}. 
		Interestingly, the model-based approach --- with proper modification to implement optimism in the face of uncertainty --- achieves minimax-optimal regret asymptotically  \citep{azar2017minimax}, 
		although its performance in the sample-starved regime remains largely unknown. 
		It would be of great interest to see whether the analysis ideas developed herein 
		could help characterize the sample efficiency of model-based online RL for the entire $\varepsilon$-range.

\item {\em Beyond the tabular setting.} The current paper focuses on the tabular setting with finite state and action spaces.
While we improve the sample size range, the sample complexities might still be prohibitively large when $|\cS|$ and $|\cA|$ are enormous. 
Therefore, it is desirable to further investigate settings where low-complexity function approximation is employed to improve the efficiency (e.g.~\citet{yang2019sample,jin2020provably,li2021sample}).

\end{itemize}

\section*{Acknowledgements}

Y.~Wei is supported in part by the the NSF grants CCF-2106778, DMS-2147546/2015447 and  CAREER award DMS-2143215.
Y.~Chi is supported in part by the ONR grants N00014-18-1-2142 and N00014-19-1-2404, 
the NSF grants CCF-1806154, CCF-2007911 and CCF-2106778. 
Y.~Chen is supported in part by the Alfred P.~Sloan Research Fellowship, the Google Research Scholar Award, the AFOSR grants FA9550-19-1-0030 and FA9550-22-1-0198, 
the ONR grant N00014-22-1-2354,  
and the NSF grants CCF-2221009, CCF-1907661, DMS-2014279, IIS-2218713 and IIS-2218773.  
We thank Qiwen Cui for pointing out an issue in Section~\ref{sec:proof-lem-tie-breaking} in an early version of this paper, and 
thank Shicong Cen, Chen Cheng and Cong Ma for numerous discussions. 
Part of this work was done while G.~Li, Y.~Wei and Y.~Chen were visiting the Simons Institute for the Theory of Computing.



\appendix 

\section{Preliminary facts}

We begin by recording a few elementary facts about $\bm{P}^{\pi}$ and $\bm{P}_{\pi}$ (see definitions in \eqref{eqn:ppivq}). 
These are standard results and we omit the proofs for brevity. 

\begin{lemma}\label{lemma:basic-properties-I-gammaP}For any policy $\pi$, any probability
	transition matrix $\bm{P}\in \mathbb{R}^{|\mathcal{S}||\mathcal{A}|\times |\mathcal{S}|}$ and any $0<\gamma<1$, one has 

\begin{itemize}

	\item[(a)] $(\bm{I}-\gamma\bm{P}_{\pi})^{-1}=\sum_{i=0}^{\infty}(\gamma\bm{P}_{\pi})^{i}$;

\item[(b)] All entries of the matrix $(\bm{I}-\gamma\bm{P}_{\pi})^{-1}$
are non-negative;

\item[(c)] $\|(\bm{I}-\gamma\bm{P}_{\pi})^{-1}\|_{1}\leq1/(1-\gamma)$; 

\item[(d)] $(1-\gamma)(\bm{I}-\gamma\bm{P}_{\pi})^{-1}\bm{1}=\bm{1}$; 

\item[(e)] For any non-negative vectors $\bm{0}\leq\bm{r}_{1}\leq\bm{r}_{2}$ of compatible dimension,
one has $\bm{0}\leq(\bm{I}-\gamma\bm{P}_{\pi})^{-1}\bm{r}_{1}\leq(\bm{I}-\gamma\bm{P}_{\pi})^{-1}\bm{r}_{2}$. 

\end{itemize}
	The above results continue to hold if $\bm{P}_{\pi}$ is replaced by $\bm{P}^{\pi}$. 
\end{lemma}

\section{Proofs of auxiliary lemmas: infinite-horizon MDPs}

\subsection{Proofs of Lemma~\ref{lemma:fixed-policy-error} and Lemma~\ref{lemma:dependent-policy-error}}
\label{sec:proof-lemma-fixed-policy-error}

\paragraph{Auxiliary notation and preliminaries.} 

Before proceeding, we define several $|\mathcal{S}|$-dimensional auxiliary vectors $\bm{r}^{(i)},\bm{V}^{(i)},\widehat{\bm{V}}^{(i)}$
($1\leq i\leq m$) recursively as follows
\begin{align}
\label{eqn:superV}
\begin{array}{lll}
	\bm{r}^{(0)}:=\bm{r}_{\pi}, & \bm{V}^{(0)}:=(\bm{I}-\gamma\Pv)^{-1}\bm{r}^{(0)}, & \widehat{\bm{V}}^{(0)}:=(\bm{I}-\gamma\Phatv)^{-1}\bm{r}^{(0)},\\
\bm{r}^{(1)}:=\sqrt{\mathsf{Var}_{\Pv}\big[\bm{V}^{(0)}\big]}, & \bm{V}^{(1)}:=(\bm{I}-\gamma\Pv)^{-1}\bm{r}^{(1)}, & \widehat{\bm{V}}^{(2)}:=(\bm{I}-\gamma\Phatv)^{-1}\bm{r}^{(1)},\\
\quad\quad\quad\vdots & \quad\quad\quad\vdots & \quad\quad\quad\vdots\\
\bm{r}^{(m)}:=\sqrt{\mathsf{Var}_{\Pv}\big[\bm{V}^{(m-1)}\big]}, & \bm{V}^{(m)}:=(\bm{I}-\gamma\Pv)^{-1}\bm{r}^{(m)}, & \widehat{\bm{V}}^{(m)}:=(\bm{I}-\gamma\Phatv)^{-1}\bm{r}^{(m)},
\end{array}
\end{align}
where $m$ will be specified momentarily.

A crucial quantity that appears repeatedly in analyzing the above terms is $\|(\bm{I}-\gamma\bm{P}_{\pi})^{-1}\sqrt{\mathsf{Var}_{\bm{P}_{\pi}}(\bm{V})}\|_{\infty}$,
whose importance was already made apparent in the work \cite{azar2013minimax}.
A widely used upper bound on this quantity, originally due to
\cite[Lemma 8]{azar2013minimax}, is given by
\begin{equation}
\left\Vert (\bm{I}-\gamma\bm{P}_{\pi})^{-1}\sqrt{\mathsf{Var}_{\bm{P}_{\pi}}(\bm{V})}\right\Vert _{\infty}\leq\frac{2\log2}{\gamma(1-\gamma)^{1.5}}\left\Vert \bm{r}\right\Vert _{\infty}.
\label{eq:Azar-bound}
\end{equation}
This bound turns out to be loose for our purpose, and
we develop an improved bound as follows, whose proof is deferred to Appendix~\ref{sec:proof_lemma_varp-vbound}. 


\begin{lemma}
\label{lemma:VarP-V-bound}
Consider any policy $\pi$ and any probability transition matrix $\bm{P}\in \mathbb{R}^{|\mathcal{S}||\mathcal{A}|\times |\mathcal{S}|}$. Let 
 $\bm{V}$ be a vector
obeying $\bm{V}=(\bm{I}-\gamma\bm{P}_{\pi})^{-1}\bm{r}_{\pi}$ for some $ |\mathcal{S}|$-dimensional vector
$\bm{r}_{\pi}\geq\bm{0}$.
For any $0<\gamma<1$, one has 
\begin{equation}
\left\Vert (\bm{I}-\gamma\bm{P}_{\pi})^{-1}\sqrt{\mathsf{Var}_{\bm{P}_{\pi}}(\bm{V})}\right\Vert _{\infty}\leq\frac{4}{\gamma\sqrt{1-\gamma}}\left\Vert \bm{V}\right\Vert _{\infty}.\label{eq:lemma-VarP-V-bound}
\end{equation}
\end{lemma}

\begin{remark}In comparison to the bound \eqref{eq:Azar-bound}
derived in \cite[Lemma 8]{azar2013minimax}, Lemma \ref{lemma:VarP-V-bound}
offers an improved upper bound stated directly in terms of
the properties of $\bm{V}$ rather than those of $\bm{r}$. \end{remark}

As it turns out, Lemma \ref{lemma:VarP-V-bound}
allows us to obtain an entrywise bound for $\bm{V}^{(l)}$ $(1\leq l\leq m)$.
To begin with, the first term $\bm{V}^{(1)}$ satisfies 
\begin{equation}
	\big\|\bm{V}^{(1)}\big\|_{\infty}=\Big\|(\bm{I}-\gamma\Pv)^{-1}\sqrt{\mathsf{Var}_{\Pv}\big[\bm{V}^{(0)}\big]}\Big\|_{\infty} 
	\label{eq:V1-inf-bound}
\end{equation}
since $\bm{r}^{(1)}=\sqrt{\mathsf{Var}_{\Pv}\big[\bm{V}^{(0)}\big]}$. 
Next, for any $l>1$ one has 
\begin{align*}
	\big\|\bm{V}^{(l)}\big\|_{\infty} & = \big\|(\bm{I}-\gamma\Pv)^{-1}\bm{r}^{(l)}\big\|_{\infty}  =\Big\|(\bm{I}-\gamma\Pv)^{-1}\sqrt{\mathsf{Var}_{\Pv}\big[\bm{V}^{(l-1)}\big]}\Big\|_{\infty}\\
 & \leq\frac{4}{\gamma\sqrt{1-\gamma}}\big\|\bm{V}^{(l-1)}\big\|_{\infty},
\end{align*}
where the second identity results from the definition of $\bm{r}^{(l)}$,
and the last inequality comes from Lemma \ref{lemma:VarP-V-bound}.
As a consequence, applying this inequality recursively gives 
\begin{align}
\big\|\bm{V}^{(l)}\big\|_{\infty}\leq\Big(\frac{4}{\gamma\sqrt{1-\gamma}}\Big)^{l-1}\big\|\bm{V}^{(1)}\big\|_{\infty} .
 \label{eq:Vl-inf-norm-bound}
\end{align}


\paragraph*{Main proof.}
Equipped with the above facts, we are now in a position to prove the
lemmas, for which we start with the more general one --- Lemma~\ref{lemma:dependent-policy-error}. Consider any $0\leq l \leq m$. We first observe that
\begin{align}
\widehat{\bm{V}}^{(l)}-\bm{V}^{(l)} & =(\bm{I}-\gamma\Phatv)^{-1}\bm{r}^{(l)}-\bm{V}^{(l)}\nonumber \\
 & =\big(\bm{I}-\gamma\Phatv\big)^{-1} \big(\bm{I}-\gamma\Pv\big)\bm{V}^{(l)} -\big(\bm{I}-\gamma\Phatv\big)^{-1}\big(\bm{I}-\gamma\Phatv\big)\bm{V}^{(l)}\nonumber \\
 & =\gamma\big(\bm{I}-\gamma\Phatv\big)^{-1}\big(\Phatv-\Pv\big)\bm{V}^{(l)},\label{eq:Vl-Vhat-l-perturbation}
\end{align}
where the second line follows since, by definition, $\big(\bm{I}-\gamma\Pv\big)\bm{V}^{(l)}=\bm{r}^{(l)}$.
Suppose that there exists some quantity $\beta_1>0$ such that the following condition 
\begin{equation}
	\Big|\big(\Phatv-\Pv\big)\bm{V}^{(l)}\Big|\leq\sqrt{\frac{\beta_1}{N}}\sqrt{\mathsf{Var}_{\Pv}\big[\bm{V}^{(l)}\big]}+\frac{\big\|\bm{V}^{(l)}\big\|_{\infty}\beta_1}{N}\bm{1}
	\label{eq:Bernstein-Vl}
\end{equation}
holds uniformly for all $0\leq l\leq m$. Then this combined with (\ref{eq:Vl-Vhat-l-perturbation}) gives 
\begin{align*}
\big\|\widehat{\bm{V}}^{(l)}-\bm{V}^{(l)}\big\|_{\infty} & =\gamma\Big\|\big(\bm{I}-\gamma\Phatv\big)^{-1}\big(\Phatv-\Pv\big)\bm{V}^{(l)}\Big\|_{\infty}\nonumber\\
 & \overset{(\mathrm{i})}{\leq}\gamma\Big\|\big(\bm{I}-\gamma\Phatv\big)^{-1}\Big|\big(\Phatv-\Pv\big)\bm{V}^{(l)}\Big|\Big\|_{\infty}\nonumber\\
 & \overset{(\mathrm{ii})}{\leq}\gamma\sqrt{\frac{\beta_{1}}{N}}\Big\|\big(\bm{I}-\gamma\Phatv\big)^{-1}\sqrt{\mathsf{Var}_{\Pv}\big[\bm{V}^{(l)}\big]}\Big\|_{\infty}+\frac{\gamma\big\|\bm{V}^{(l)}\big\|_{\infty}\beta_{1}}{N}\Big\|\big(\bm{I}-\gamma\Phatv\big)^{-1}\bm{1}\Big\|_{\infty} . \nonumber
\end{align*}
Here, (i) follows since $\big(\bm{I}-\gamma\Phatv\big)^{-1}$
is a non-negative matrix, (ii) comes from (\ref{eq:Bernstein-Vl})
and the triangle inequality.
Recalling the definition of $\bm{r}^{(l)}$
and $\widehat{\bm{V}}^{(l)}$ and invoking Lemma~\ref{lemma:basic-properties-I-gammaP}(d), we can further bound the above as 
\begin{align}
\big\|\widehat{\bm{V}}^{(l)}-\bm{V}^{(l)}\big\|_{\infty} & \leq\gamma\sqrt{\frac{\beta_{1}}{N}}\big\|\widehat{\bm{V}}^{(l+1)}\big\|_{\infty}+\frac{\gamma\big\|\bm{V}^{(l)}\big\|_{\infty}\beta_{1}}{(1-\gamma)N}\nonumber\\
 & \leq\gamma\sqrt{\frac{\beta_{1}}{N}}\big\|\widehat{\bm{V}}^{(l+1)}-\bm{V}^{(l+1)}\big\|_{\infty}
   +\gamma\sqrt{\frac{\beta_{1}}{N}}\big\|\bm{V}^{(l+1)}\big\|_{\infty}+\frac{\gamma\beta_{1}}{(1-\gamma)N}\big\|\bm{V}^{(l)}\big\|_{\infty}, 
   \label{eq:upper-bound-Vl-Vhatl}
\end{align}
where the last inequality is a consequence of the triangle inequality.

The above inequality \eqref{eq:upper-bound-Vl-Vhatl} provides a recursive relation that in turn allows
for an effective upper bound. Specifically, combining the inequalities (\ref{eq:Vl-inf-norm-bound}) and (\ref{eq:upper-bound-Vl-Vhatl}) leads to 
\begin{align*}
  \big\|\widehat{\bm{V}}^{(0)}-\bm{V}^{(0)}\big\|_{\infty} 
	&\leq 
  \gamma\sqrt{\frac{\beta_{1}}{N}}\big\|\widehat{\bm{V}}^{(1)}-\bm{V}^{(1)}\big\|_{\infty}
   + \gamma\sqrt{\frac{\beta_{1}}{N}}\big\|\bm{V}^{(1)}\big\|_{\infty}
   + \frac{\gamma\beta_{1}}{(1-\gamma)N}\big\|\bm{V}^{(0)}\big\|_{\infty} \\
	& =: 
  	b_1 \big\|\widehat{\bm{V}}^{(1)}-\bm{V}^{(1)}\big\|_{\infty}
   + b_1 \big\|\bm{V}^{(1)}\big\|_{\infty}
   + \frac{\gamma\beta_{1}}{(1-\gamma)N}\big\|\bm{V}^{(0)}\big\|_{\infty}	,
   \end{align*}
and for $l \geq 1$,
\begin{align*}  
  \big\|\widehat{\bm{V}}^{(l)}-\bm{V}^{(l)}\big\|_{\infty}
  & \leq\gamma\sqrt{\frac{\beta_{1}}{N}}\big\|\widehat{\bm{V}}^{(l+1)}-\bm{V}^{(l+1)}\big\|_{\infty}
  +
  \left(4\sqrt{\frac{\beta_1}{(1-\gamma)N}} + \frac{\gamma\beta_1}{(1-\gamma)N}\right)
  \left(\frac{4}{\gamma\sqrt{1-\gamma}}\right)^{l-1}\|\bm V^{(1)}\|_{\infty}\\
 & =: b_{1}\big\|\widehat{\bm{V}}^{(l+1)}-\bm{V}^{(l+1)}\big\|_{\infty}+b_{2}b_{3}^{l-1}\|\bm V^{(1)}\|_{\infty}.
\end{align*}
Here for notational simplicity, we introduce 
\begin{align*}
b_{1}:=\gamma\sqrt{\frac{\beta_{1}}{N}},
\qquad 
b_{2}:= 4\sqrt{\frac{\beta_1}{(1-\gamma)N}} + \frac{\gamma\beta_1}{(1-\gamma)N},
\qquad b_{3}:=\frac{4}{\gamma\sqrt{1-\gamma}}.
\end{align*}
Invoking the above recursive relation, we can arrange terms to reach
\begin{align}
\big\|\widehat{\bm{V}}^{(0)}-\bm{V}^{(0)}\big\|_{\infty} & \leq\underbrace{b_{1}^{m}\big\|\widehat{\bm{V}}^{(m)}-\bm{V}^{(m)}\big\|_{\infty}}_{=:\alpha_{1}}+
\underbrace{\Big( b_1b_{2}\sum_{l=0}^{m-2}(b_{1}b_{3})^{l} + b_1 \Big)}_{=:\alpha_{2}}\big\|\bm{V}^{(1)}\big\|_{\infty}
   + \frac{\gamma\beta_{1}}{(1-\gamma)N}\big\|\bm{V}^{(0)}\big\|_{\infty}.
\label{eqn:Chopin}
\end{align}

\medskip
\noindent \textit{Controlling the quantity $\alpha_{2}$.}
Now it suffices to control the two terms on the right-hand side of the inequality~\eqref{eqn:Chopin} separately, towards which we shall
start with the quantity $\alpha_{2}.$ Assuming that $N\geq 64\beta_1/(1-\gamma)$,
one can easily verify $b_{1}b_{3}\leq1/2$. The summation of the geometric
sequence thus gives 
\begin{align}
	\notag \alpha_{2} &:= b_1 b_{2}\sum_{l=0}^{m-2}(b_{1}b_{3})^{l} + b_1 
\leq \frac{b_1b_{2}}{1-b_{1}b_{3}} + b_1 
\leq 2b_1b_{2} + b_1
	\\ &= \gamma \sqrt{\frac{\beta_1}{N}}
	\left\{1 + 8\sqrt{\frac{\beta_1}{(1-\gamma)N}} + \frac{2\gamma\beta_1}{(1-\gamma)N} \right\}
  \leq 3\gamma \sqrt{\frac{\beta_1}{N}}, \label{eq:T2-bound}
\end{align}
where the last step holds with the assumption $N\geq\frac{64\beta_{1}}{1-\gamma}$.

\medskip
\noindent \textit{Controlling the quantity $\alpha_{1}$.} 
Next, we proceed to the quantity $\alpha_{1}$, which requires the
control of \mbox{$\big\|\widehat{\bm{V}}^{(m)}-\bm{V}^{(m)}\big\|_{\infty}$}.
In view of the identity \eqref{eq:Vl-Vhat-l-perturbation}, we obtain
\begin{align*}
\big\|\widehat{\bm{V}}^{(m)}-\bm{V}^{(m)}\big\|_{\infty} & =\gamma\big\|(\bm{I}-\gamma\Phatv)^{-1}(\Pv-\Phatv)\bm{V}^{(m)}\big\|_{\infty}\\
 & \leq\gamma\Bigg\|(\bm{I}-\gamma\Phatv)^{-1}\Bigg(\sqrt{\frac{\beta_{1}}{N}}\sqrt{\mathsf{Var}_{\Pv}\big[\bm{V}^{(m)}\big]}+\frac{\big\|\bm{V}^{(m)}\big\|_{\infty}\beta_{1}}{N}\bm{1}\Bigg)\Bigg\|_{\infty},
\end{align*}
where the last inequality follows from the Bernstein-type condition~\eqref{eq:Bernstein-Vl}
and the fact that $\big(\bm{I}-\gamma\Phatv\big)^{-1}$
has non-negative entries. By virtue of the simple relation $\sqrt{\mathsf{Var}_{\Pv}\big[\bm{V}^{(m)}\big]}\leq\big\|\bm{V}^{(m)}\big\|_{\infty}$
and the fact that $\|\big(\bm{I}-\gamma\Phatv\big)^{-1}\|_{1}\leq\frac{1}{1-\gamma}$ (cf.~Lemma~\ref{lemma:basic-properties-I-gammaP}(c)),
it is further guaranteed that 
\begin{align*}
\big\|\widehat{\bm{V}}^{(m)}-\bm{V}^{(m)}\big\|_{\infty} & \leq\gamma\|(\bm{I}-\gamma\Phatv)^{-1}\|_{1}\cdot\Bigg\|\sqrt{\frac{\beta_{1}}{N}}\sqrt{\mathsf{Var}_{\Pv}\big[\bm{V}^{(m)}\big]}+\frac{\big\|\bm{V}^{(m)}\big\|_{\infty}\beta_{1}}{N}\bm{1}\Bigg\|_{\infty},\\
 & \leq\frac{\gamma}{1-\gamma}\Bigg(\sqrt{\frac{\beta_{1}}{N}}+\frac{\beta_{1}}{N}\Bigg)\big\|\bm{V}^{(m)}\big\|_{\infty},
\end{align*}
which combined with the bound~\eqref{eq:Vl-inf-norm-bound} yields
\begin{align*}
\big\|\widehat{\bm{V}}^{(m)}-\bm{V}^{(m)}\big\|_{\infty} & \leq \frac{\gamma}{1-\gamma}
\Bigg(\sqrt{\frac{\beta_{1}}{N}}+\frac{\beta_{1}}{N}\Bigg)
\Big(\frac{4}{\gamma\sqrt{1-\gamma}}\Big)^{m-1}\|\bm V^{(1)}\|_{\infty} .
\end{align*}
Putting the above bounds together yields 
\begin{align}
\alpha_{1}:=b_{1}^{m}\big\|\widehat{\bm{V}}^{(m)}-\bm{V}^{(m)}\big\|_{\infty} & \leq \left(\gamma\sqrt{\frac{\beta_{1}}{N}}\right)^{m}
\frac{\gamma}{1-\gamma} \Bigg(\sqrt{\frac{\beta_{1}}{N}}+\frac{\beta_{1}}{N}\Bigg) \Big(\frac{4}{\gamma\sqrt{1-\gamma}}\Big)^{m-1}\|\bm V^{(1)}\|_{\infty}
\nonumber\\
 & =\left(\sqrt{\frac{16\beta_{1}}{N(1-\gamma)}}\right)^{m-1}\Bigg(\sqrt{\frac{\beta_{1}}{N}}+1\Bigg)
 \frac{\gamma^2\beta_{1}}{(1-\gamma)N}\|\bm V^{(1)}\|_{\infty} \nonumber\\
 & \leq \left(\frac{1}{e}\right)^{m-1} \frac{1.1 \gamma^2\beta_{1}}{(1-\gamma)N} \|\bm V^{(1)}\|_{\infty},
 \label{eq:T1-bound}
\end{align}
where the last inequality holds provided that $N>\frac{16e^{2}}{1-\gamma}\beta_{1}$.

\medskip
\noindent \textit{Putting all this together.}
Combining the inequalities~\eqref{eqn:Chopin}, \eqref{eq:T2-bound} and \eqref{eq:T1-bound} gives 
\begin{align*}
  \big\|\widehat{\bm{V}}^{(0)}-\bm{V}^{(0)}\big\|_{\infty}
	\leq
	\frac{\gamma\beta_{1}}{(1-\gamma)N}\big\|\bm{V}^{(0)}\big\|_{\infty} +
	\left \{3 \gamma \sqrt{\frac{\beta_1}{N}}
	+
	\left(\frac{1}{e}\right)^{m-1} \frac{1.1\gamma^2\beta_{1}}{(1-\gamma)N} \right\} \|\bm V^{(1)}\|_{\infty} . 
\end{align*}
To finish up, set $m=\log (\frac{e}{1-\gamma})$ and assume that $N>\frac{16e^{2}}{1-\gamma}\beta_{1}$. Recognizing that $\bm{V}^{(0)}=\bm{V}^{\pi}$ and $\widehat{\bm{V}}^{(0)}=\widehat{\bm{V}}^{\pi}$, 
we arrive at
\begin{align}
	\notag \big\|\widehat{\bm{V}}^{\pi}-\bm{V}^{\pi}\big\|_{\infty}  =
	\big\|\widehat{\bm{V}}^{(0)}-\bm{V}^{(0)}\big\|_{\infty}
	&\leq 
	\frac{\gamma\beta_{1}}{(1-\gamma)N}\big\|\bm{V}^{(0)}\big\|_{\infty} + \left \{3 \gamma \sqrt{\frac{\beta_1}{N}}
	+
	\left(\frac{1}{e}\right)^{m-1} \frac{1.1\gamma^2\beta_{1}}{(1-\gamma)N} \right\} \|\bm V^{(1)}\|_{\infty} \\
 	& \leq \frac{\gamma\beta_{1}}{(1-\gamma)N}\big\|\bm{V}^{\pi} \big\|_{\infty} + 
 	  4 \gamma \sqrt{\frac{\beta_1}{N}} 
 	  \Big\|(\bm{I}-\gamma\Pv)^{-1}\sqrt{\mathsf{Var}_{\Pv}\big[\bm{V}^{\pi}\big]}\Big\|_{\infty},
 	  \label{eqn:ice-cream}
\end{align}
provided that $N\geq\frac{16e^2 \beta_{1}}{1-\gamma}$.

\paragraph{Proof of Lemma~\ref{lemma:dependent-policy-error}.} 
Invoking the inequality~\eqref{eq:Azar-bound} to bound the second term of \eqref{eqn:ice-cream}, we reach 
\begin{align}
\label{eqn:for-lemma-3}
	\big\|\widehat{\bm{V}}^{\pi}-\bm{V}^{\pi}\big\|_{\infty} \leq 
	\frac{\gamma\beta_{1}}{(1-\gamma)N}\big\|\bm{V}^{\pi} \big\|_{\infty} + 
 	 8 \log 2 \sqrt{\frac{\beta_1}{N(1-\gamma)^3}} \|\bm r\|_\infty 
 	 \leq 
 	 6 \sqrt{\frac{\beta_1}{N(1-\gamma)^3}} \|\bm r\|_\infty,
\end{align}
where the last inequality uses the elementary fact that $\big\|\bm{V}^{\pi} \big\|_{\infty} \leq \frac{1}{1 - \gamma}\|\bm r\|_\infty$ and the assumption that $N>\frac{16e^{2}}{1-\gamma}\beta_{1}$.
We complete the proof of Lemma~\ref{lemma:dependent-policy-error}.

\paragraph{Proof of Lemma~\ref{lemma:fixed-policy-error}.} 
Finally, to establish Lemma~\ref{lemma:fixed-policy-error}, we observe that: for any fixed policy $\pi$, the vector  $\bm{V}^{(l)}$ is independent of $\Phatv$. 
The Bernstein inequality (e.g.~\cite[Lemma 6]{agarwal2019optimality})
then reveals that with probability at least $1-\delta$, 
\begin{equation}
\Big|\big(\Phatv-\Pv\big)\bm{V}^{(l)}\Big|\leq\sqrt{\frac{2\log\big(\frac{4m|\mathcal{S}|}{\delta}\big)}{N}}\sqrt{\mathsf{Var}_{\Pv}\big[\bm{V}^{(l)}\big]}+\frac{\big\|\bm{V}^{(l)}\big\|_{\infty}\log\big(\frac{4m|\mathcal{S}|}{\delta}\big)}{N}\bm{1}\label{eq:Bernstein-Vl-special}
\end{equation}
holds uniformly for all $0\leq l\leq m$. This means that we can take $\beta_1 := 2\log\big(\frac{4m|\mathcal{S}|}{\delta}\big)$ with $m=\log (\frac{e}{1-\gamma})$ for this case. Combining this with the inequality~\eqref{eqn:ice-cream}, we derive the advertised instance-dependent bound  
\begin{align*}
	\big\|\widehat{\bm{V}}^{\pi}-\bm{V}^{\pi}\big\|_{\infty} 
	&\leq 
	\frac{2\gamma \log\big(\frac{4|\mathcal{S}| \log \frac{e}{1-\gamma}}{\delta}\big)}{(1-\gamma)N}\big\|\bm{V}^{\pi} \big\|_{\infty} 
	+ 
 	4 \gamma \sqrt{\frac{2 \log\big(\frac{4|\mathcal{S}| \log \frac{e}{1-\gamma}}{\delta}\big)}{(1-\gamma)N}} 
 	  \, \Big\|(\bm{I}-\gamma\Pv)^{-1}\sqrt{\mathsf{Var}_{\Pv}\big[\bm{V}^{\pi}\big]}\Big\|_{\infty}.
\end{align*}
Further,  this taken collectively with \eqref{eq:Azar-bound} and the crude bound $\|\bm{V}^{\pi}\|_{\infty} \leq \frac{1}{1-\gamma} \|\bm{r}\|_{\infty}$ gives
\begin{align*}
\notag\big\|\widehat{\bm{V}}^{\pi}-\bm{V}^{\pi}\big\|_{\infty} & \leq\frac{2\gamma\log\big(\frac{4|\mathcal{S}|\log\frac{e}{1-\gamma}}{\delta}\big)}{(1-\gamma)^{2}N}\big\|\bm{r}\big\|_{\infty}+8\log2\sqrt{\frac{2\log\big(\frac{4|\mathcal{S}|\log\frac{e}{1-\gamma}}{\delta}\big)}{(1-\gamma)^{3}N}}\,\|\bm{r}\|_{\infty}\\
 & \leq6\sqrt{\frac{2\log\big(\frac{4|\mathcal{S}|\log\frac{e}{1-\gamma}}{\delta}\big)}{(1-\gamma)^{3}N}}\,\|\bm{r}\|_{\infty},
\end{align*}
with the proviso that $N\geq\frac{32e^{2}}{1-\gamma}\log\big(\frac{4|\mathcal{S}|\log\frac{e}{1-\gamma}}{\delta}\big).$

%
%
%
%

\subsubsection{Proof of Lemma \ref{lemma:VarP-V-bound}}\label{sec:proof_lemma_varp-vbound}
To begin with, we make the observation that 
\begin{align}
\mathsf{Var}_{\bm{P}_{\pi}}(\bm{V}) & =\bm{P}_{\pi}(\bm{V}\circ\bm{V})-(\bm{P}_{\pi}\bm{V})\circ(\bm{P}_{\pi}\bm{V})\nonumber \\
 & =\bm{P}_{\pi}(\bm{V}\circ\bm{V})-\frac{1}{\gamma^{2}}(\bm{V}-\bm{r}_{\pi})\circ(\bm{V}-\bm{r}_{\pi})\label{eq:VarP-V-expand1}\\
 & =\bm{P}_{\pi}(\bm{V}\circ\bm{V})-\frac{1}{\gamma^{2}}\bm{V}\circ\bm{V}+\frac{2}{\gamma^{2}}\bm{V}\circ\bm{r}_{\pi}-\frac{1}{\gamma^{2}}\bm{r}_{\pi}\circ\bm{r}_{\pi}\nonumber \\
 & \leq\frac{1}{\gamma^{2}}(\gamma^{2}\bm{P}_{\pi}-\bm{I})(\bm{V}\circ\bm{V})+\frac{2}{\gamma^{2}}\bm{V}\circ\bm{r}_{\pi},\label{eq:Var-V-expand-4}
\end{align}
where the identity (\ref{eq:VarP-V-expand1}) makes use of the relation
 $\bm{V}=\bm{r}_{\pi}+\gamma\bm{P}_{\pi}\bm{V}$. In addition, one can
deduce that
\begin{align}
 \left\Vert (\bm{I}-\gamma\bm{P}_{\pi})^{-1}\sqrt{\mathsf{Var}_{\bm{P}_{\pi}}(\bm{V})}\right\Vert _{\infty} & =\frac{1}{1-\gamma}\left\| (1-\gamma)(\bm{I}-\gamma\bm{P}_{\pi})^{-1}\sqrt{  \mathsf{Var}_{\bm{P}_{\pi}}(\bm{V}) }\right\| _{\infty} \nonumber \\
 & \overset{(\mathrm{i})}{\leq}\frac{1}{\sqrt{1-\gamma}}\left\Vert \sqrt{(\bm{I}-\gamma\bm{P}_{\pi})^{-1}\mathsf{Var}_{\bm{P}_{\pi}}(\bm{V})}\right\Vert _{\infty}\nonumber \\
 & \overset{(\mathrm{ii})}{\leq}\frac{1}{\sqrt{1-\gamma}}\sqrt{\left\Vert 2(\bm{I}-\gamma^{2}\bm{P}_{\pi})^{-1}\mathsf{Var}_{\bm{P}_{\pi}}(\bm{V})\right\Vert _{\infty}}\nonumber.
\end{align}
Here, (i) comes from Jensen's inequality (so that $\mathbb{E}[\sqrt{v}]\leq\sqrt{\mathbb{E}[v]}$) recognizing that each row of \mbox{$(1-\gamma)(\bm{I}-\gamma\bm{P}_{\pi})^{-1}$} is a probability distribution,
and Lemma \ref{lemma:basic-properties-I-gammaP}(d), (ii) is an elementary
fact established in \cite[Lemma 4]{agarwal2019optimality}.
Combining Lemma \ref{lemma:basic-properties-I-gammaP}(e) and the inequality \eqref{eq:Var-V-expand-4} 
further yields 
\begin{align}
\left\Vert (\bm{I}-\gamma\bm{P}_{\pi})^{-1}\sqrt{\mathsf{Var}_{\bm{P}_{\pi}}(\bm{V})}\right\Vert _{\infty} & \leq \frac{1}{\sqrt{1-\gamma}}\sqrt{\left\Vert 2(\bm{I}-\gamma^{2}\bm{P}_{\pi})^{-1}\left(\frac{1}{\gamma^{2}}\Big(\gamma^{2}\bm{P}_{\pi}-\bm{I}\Big)(\bm{V}\circ\bm{V})+\frac{2}{\gamma^{2}}\bm{V}\circ\bm{r}_{\pi}\right)\right\Vert _{\infty}}\nonumber \\
 & \leq\frac{1}{\gamma\sqrt{1-\gamma}}\sqrt{2\left\Vert \bm{V}\circ\bm{V}\right\Vert _{\infty}}+\frac{2}{\gamma\sqrt{1-\gamma}}\sqrt{\left\Vert (\bm{I}-\gamma^{2}\bm{P}_{\pi})^{-1}\left(\bm{V}\circ\bm{r}_{\pi}\right)\right\Vert _{\infty}}.\label{eq:VarP-V-expand2}
\end{align}
where the last step arises from the triangle inequality and
$\sqrt{a+b}\leq\sqrt{a}+\sqrt{b}$. This leaves us with two terms
to deal with. 

Regarding the first term of (\ref{eq:VarP-V-expand2}), we observe
that $\left\Vert (\bm{V}\circ\bm{V})\right\Vert _{\infty}=\left\Vert \bm{V}\right\Vert _{\infty}^{2}$.
When it comes to the second term of (\ref{eq:VarP-V-expand2}), it
is seen that 
\begin{align*}
\left\Vert (\bm{I}-\gamma^{2}\bm{P}_{\pi})^{-1}\left(\bm{V}\circ\bm{r}_{\pi}\right)\right\Vert _{\infty} & \leq\left\Vert (\bm{I}-\gamma\bm{P}_{\pi})^{-1}\left(\bm{V}\circ\bm{r}_{\pi}\right)\right\Vert _{\infty}\\
 & \leq\|\bm{V}\|_{\infty}\left\Vert (\bm{I}-\gamma\bm{P}_{\pi})^{-1}\bm{r}_{\pi}\right\Vert _{\infty}\\
 & =\|\bm{V}\|_{\infty}^{2}.
\end{align*}
Here, the first inequality holds true since $(\bm{I}-\gamma^{2}\bm{P}_{\pi})^{-1}=\sum_{i=0}^{\infty}\gamma^{2i}\bm{P}_{\pi}^{i}\leq\sum_{i=0}^{\infty}\gamma^{i}\bm{P}_{\pi}^{i}=(\bm{I}-\gamma\bm{P}_{\pi})^{-1}$,
while the second line holds since $\bm{V}$, $\bm{r}$ and $(\bm{I}-\gamma\bm{P}_{\pi})^{-1}$
are all non-negative. Substitution into (\ref{eq:VarP-V-expand2})
thus yields 
\begin{align*}
\left\Vert (\bm{I}-\gamma\bm{P}_{\pi})^{-1}\sqrt{\mathsf{Var}_{\bm{P}_{\pi}}(\bm{V})}\right\Vert _{\infty} & \leq\frac{\sqrt{2}}{\gamma\sqrt{1-\gamma}}\left\Vert \bm{V}\right\Vert _{\infty}+\frac{2}{\gamma\sqrt{1-\gamma}}\left\Vert \bm{V}\right\Vert _{\infty}\leq\frac{4}{\gamma\sqrt{1-\gamma}}\left\Vert \bm{V}\right\Vert _{\infty}
\end{align*}
as claimed.

\subsection{Proof of Lemma~\ref{LemAbsToNormalQV}}
\label{sec:proof-LemAbsToNormalQV}

To establish Lemma~\ref{LemAbsToNormalQV}, it suffices to check that $\Vstar$ and $\Qstar$ satisfy the Bellman optimality equations underlying $\MDP_{s,a,u^{\star}}$. 
	Towards this end, we study the absorbing state-action pair $(s,a)$ and other pairs separately. For notational simplicity, we shall let $\abPbm$ and $\rewabs(\cdot,\cdot)$ denote respectively the probability transition matrix and the reward function associated with $\MDP_{s,a, u^{\star}}$. 

First, we observe that, by construction,  
\begin{align*}
	\rewabs (s,a) + \gamma (\abPbm \Vstar)_{s,a} = u^{\star} + \gamma V^{\star}(s).
\end{align*}
%
Recall that $\Vstar$ satisfies the Bellman optimality equation w.r.t.~the original MDP, namely, 
$Q^\star(s,a) = r(s,a) + \gamma (\bP \Vstar)_{s,a}$.
	This together with our choice of $u^{\star}$ gives
\begin{align*}
	u^{\star} + \gamma V^{\star}(s) &=Q^\star(s,a)- \gamma V^{\star}(s) +  \gamma V^{\star}(s)   = Q^\star(s,a). 
\end{align*}
Putting the above identities together, we arrive at  
\begin{align}
	\label{eqn:bellman-sa}
	Q^\star(s,a) = \rewabs (s,a) + \gamma (\abPbm \Vstar)_{s,a}.
\end{align}
Next, consider any state-action pair $(s',a')\neq(s,a)$.
Recalling again the properties of $\MDP_{s,a,u^{\star}}$, we reach 
\begin{align}
\label{eqn:bellman-notsa}
	\rewabs (s',a') + \gamma (\abPbm \Vstar)_{s',a'}
	= 
	  r (s',a') +  \gamma (\bP \Vstar)_{s',a'} = Q^\star(s',a').
\end{align}
Here the last identity is due to the Bellman equation w.r.t.~the original MDP.
Combining \eqref{eqn:bellman-sa} and \eqref{eqn:bellman-notsa} implies that 
$\Vstar$ and $\Qstar$ satisfy Bellman's optimality equations in  $\MDP_{s,a,u^{\star}}$, thus concluding the proof. 

\subsection{Proof of Lemma~\ref{LemUnique}}
\label{sec:proof-LemUnique}

Our first observation is that 
$\Qhatstar_{s,a,u}$ satisfies Lipschitz continuity w.r.t.~$u$ in the sense that
\begin{align}
\label{eqn:lipschitz}
	\big\|\Qhatstar_{s,a,u} - \Qhatstar_{s,a,u'}\|_{\infty} \leq \frac{1}{1- \gamma} |u - u'|.
\end{align}
The proof of this relation is identical to that of \cite[Lemma 8]{agarwal2019optimality}; we omit here for brevity. 
	In view of Lemma~\ref{LemAbsToNormalQV}, if we set $\widehat{u}^{\star} \defn r(s,a) + \gamma (\Phat \Vhatstar)_{s,a} - \gamma \widehat{V}^\star(s)$, then one has 
\begin{align*}
	\Qhatstar_{s,a,\widehat{u}^{\star}} = \Qhatstar \qquad \text{and }~~\Vhatstar_{s,a,\widehat{u}^{\star}} = \Vhatstar.
\end{align*}
	In addition, there exists a point $u_0$ in the epsilon-net $\net{(1-\gamma)\gap/4}$ such that $|\widehat{u}^{\star} - u_0| \leq (1-\gamma)\gap/4$, which combined with
the Lipschitz continuity property~\eqref{eqn:lipschitz} gives
\begin{align}
\label{eqn:watermelon}
	\big\|\Qhatstar  - \Qhatstar_{s,a,u_0}\|_{\infty} = \big\|\Qhatstar_{s,a,\widehat{u}^{\star}}  - \Qhatstar_{s,a,u_0}\|_{\infty} \leq \frac{1}{1- \gamma}  |\widehat{u}^{\star} - u_0| \leq  \frac{\gap}{4}.
\end{align}

Additionally, for any $s' \in \mathcal{S}$ and any $a_{1}, a_{2} \in \mathcal{A}$ with $a_1\neq a_2$, we have the following decomposition 
\begin{align}
\notag	&\widehat{Q}^\star_{s,a,u_0}(s',a_1) - \widehat{Q}^\star_{s,a,u_0}(s',a_2) \nonumber\\
\notag	&  \qquad = \widehat{Q}^\star(s',a_1) - \widehat{Q}^\star(s',a_2) + 
	\widehat{Q}^\star_{s,a,u_0}(s',a_1) - \widehat{Q}^\star(s',a_1) - \left(\widehat{Q}^\star_{s,a,u_0}(s',a_2) - \widehat{Q}^\star(s',a_2) \right)
	 \nonumber\\
	 & \qquad \geq \widehat{Q}^\star(s',a_1) - \widehat{Q}^\star(s',a_2) - 2 \big\|\Qhatstar  - \Qhatstar_{s,a,u_0}\|_{\infty}  \nonumber\\
	& \qquad \geq \widehat{Q}^\star(s',a_1) - \widehat{Q}^\star(s',a_2) - \frac{\gap}{2},  \label{eqn:uzerotou}
\end{align}
where the last inequality invokes the inequality~\eqref{eqn:watermelon}. Moreover, our separation condition defined in \eqref{eq:defn-separation-event} requires that: for any $s' \in \mathcal{S}$ and any $a_{2} \neq \widehat{\pi}^\star(s')$, one has $\widehat{Q}^\star(s',\widehat{\pi}^\star(s')) -\widehat{Q}^\star(s',a_2) \geq \gap$, which together with \eqref{eqn:uzerotou} reveals that 
\begin{align}
	\widehat{Q}^\star_{s,a,u_0}(s',\widehat{\pi}^\star(s')) - \widehat{Q}^\star_{s,a,u_0}(s',a_2) 
	 \geq \widehat{Q}^\star(s',\widehat{\pi}^\star(s')) - \widehat{Q}^\star(s',a_2)- \frac{\gap}{2}
	 \geq \frac{\gap}{2} . 
	 \label{eq:separation-Q-uas}
\end{align}
Given that $\pihatstar_{s,a,u_0}(s') \defn \arg\max_{a'} \widehat{Q}^\star_{s,a,u_0}(s',a')$, 
it is seen from \eqref{eq:separation-Q-uas} that 
\begin{align*}
	\widehat{\pi}^\star_{s,a,u_0}(s') = \widehat{\pi}^\star(s'),
\end{align*}
which holds true for all $s' \in \mathcal{S}$. This concludes the proof.






\subsection{Proof of Lemma~\ref{lem:gap-separation-condition}}
\label{sec:proof-lem:gap-separation-condition}

 We start by bounding $\big\|\widehat{\bm{V}}^{\pihatstar}-\bm{V}^{\pihatstar}\big\|_{\infty}$. 
Recall the definition of the series $\{\bm{V}^{(l)}\}$ in \eqref{eqn:superV}. 
Throughout this proof, we shall write $\bm{V}_{\pi}^{(l)}$ instead in order to make apparent the dependency on
the policy $\pi$.

For each state-action pair $(s,a)$, let us construct the epsilon-net 
$\net{(1-\gamma)\gap/4}$ as in the expression~\eqref{eqn:net}. 
For every $u \in \net{(1-\gamma)\gap/4}$, recall that $\pihatstar_{s,a,u}$ is defined as the optimal policy with respect to the $(s,a)$-absorbing MDP $\widehat{\MDP}_{s,a,u}$. 
By construction, the set of policies $\pihatstar_{s,a,u}$ ($u \in \net{(1-\gamma)\gap/4}$) is independent of $\Phat_{s,a}$. 
%
The Bernstein inequality (e.g.~\cite[Lemma 6]{agarwal2019optimality}) taken together with the union bound thus guarantees that with probability at least $1 - \delta$, 
\begin{align}
\label{eqn:bernstein-union}
	\Big|(\Phat - \bP)_{s,a} \bm{V}_{\pihatstar_{s,a,u}}^{(l)}\Big| \leq 
	\sqrt{\frac{\beta_1}{N}}\sqrt{ \big( \mathsf{Var}_{\bP}\big[\bm{V}^{(l)}_{\pihatstar_{s,a,u}}\big] \big)_{s,a} }+\frac{\big\|\bm{V}_{\pihatstar_{s,a,u}}^{(l)}\big\|_{\infty}\beta_1}{N}
\end{align}
holds uniformly over all $0\leq l\leq  \log \frac{e}{1-\gamma}$, $u \in \net{(1-\gamma)\gap/4}$, $(s,a) \in \mathcal{S} \times \mathcal{A}$. Here, $\beta_1$ is given by
\[
\beta_{1}:=2\log\Big(\frac{4\log\big(\frac{e}{1-\gamma}\big)|\net{(1-\gamma)\omega/4}||\mathcal{S}||\mathcal{A}|}{\delta}\Big) 
\leq 2\log\Big(\frac{32}{(1-\gamma)^{2}\omega\delta}|\mathcal{S}||\mathcal{A}|\log\big(\frac{e}{1-\gamma}\big)\Big),
\]
where we have used the fact $|\net{(1-\gamma)\omega/4}| \leq \frac{8}{(1-\gamma)^2\omega}$. 
In addition, for any $0<\gap <1$,  Lemma~\ref{LemUnique}
guarantees that for each state-action pair $(s,a) \in \mathcal{S} \times \mathcal{A}$, there exists a point $u_0 \in \net{(1-\gamma)\gap/4}$ such that 
$\pihatstar = \pihatstar_{s,a,u_0}$. 
Invoking this important fact, we obtain 
\begin{align}
\notag  \left| (\Phat - \bP)_{s,a} \bm{V}_{\pihatstar}^{(l)} \right| 
&  =  \left|(\Phat - \bP)_{s,a} \bm{V}^{(l)}_{\pihatstar_{s,a,u_0}}
	\right| \\
	&\leq \sqrt{\frac{\beta_1}{N}}\sqrt{ \big( \mathsf{Var}_{\bP}\big[\bm{V}^{(l)}_{\pihatstar_{s,a,u_0}}\big] \big)_{s,a} }+\frac{\big\|\bm{V}_{\pihatstar_{s,a,u_0}}^{(l)}\big\|_{\infty}\beta_1}{N} \notag\\
	& = \sqrt{\frac{\beta_1}{N}}\sqrt{ \big( \mathsf{Var}_{\bP}\big[\bm{V}^{(l)}_{\pihatstar}\big] \big)_{s,a} }+\frac{\big\|\bm{V}^{(l)}_{\pihatstar}\big\|_{\infty}\beta_1}{N}. \notag
\end{align}
The above inequality further allows us to deduce that, with probability $1 - \delta$, 
\begin{align*}
	\Big| \big(\Phat_{\pihatstar}-\bP_{\pihatstar}\big)\bm{V}_{\pihatstar}^{(l)} \Big| &= \Big| \bPi^{\pihatstar} (\Phat - \bP) \bm{V}_{\pihatstar}^{(l)} \Big| \\
	&\leq \sqrt{\frac{\beta_1}{N}} \sqrt{ \bPi^{\pihatstar} \mathsf{Var}_{\bP}\big[\bm{V}^{(l)}_{\pihatstar}\big]}+\frac{\big\|\bm{V}^{(l)}_{\pihatstar}\big\|_{\infty}\beta_1}{N} \one  \notag\\
	&= \sqrt{\frac{\beta_1}{N}}\sqrt{\mathsf{Var}_{\bP_{\pihatstar}}\big[\bm{V}^{(l)}_{\pihatstar}\big]}+\frac{\big\|\bm{V}^{(l)}_{\pihatstar}\big\|_{\infty}\beta_1}{N} \one .
\end{align*}
The above derivation validates the assumption required for Lemma~\ref{lemma:dependent-policy-error}.
As a result, if $N > \frac{16e^{2}}{1-\gamma}\beta_{1}$ and $0<
\omega < 1$, then Lemma~\ref{lemma:dependent-policy-error} leads to the advertised bound
\begin{align}
\big\|\widehat{\bm{V}}^{\pihatstar}-\bm{V}^{\pihatstar}\big\|_{\infty} & \leq\frac{6}{1-\gamma}\sqrt{\frac{\beta_{1}}{N(1-\gamma)}}\leq\frac{6}{1-\gamma}\sqrt{\frac{2\log\Big(\frac{32}{(1-\gamma)^{2}\omega\delta}|\mathcal{S}||\mathcal{A}|\log\big(\frac{e}{1-\gamma}\big)\Big)}{N(1-\gamma)}} \notag\\
 & \leq 6 \sqrt{\frac{2\log\Big(\frac{32|\mathcal{S}||\mathcal{A}|}{(1-\gamma)^{3}\omega\delta}\Big)}{N(1-\gamma)^3}} .  \label{eq:Vhat-pihat-bound-proof}
\end{align}

Finally, we move on to the term $\Vstar - \Vhat^{\pistar} $. Given that  $\pi^{\star}$ is independent of $\widehat{\bm{P}}$, invoke Lemma~\ref{lemma:fixed-policy-error} to reach
\begin{align}
	\left\| \Vhat^{\pistar} - \Vstar \right\|_{\infty} = \big\| \Vhat^{\pistar} - \bm{V}^{\pistar} \big\|_{\infty} &  \leq 6\sqrt{2}  \sqrt{\frac{\log\big(\frac{4|\mathcal{S}|}{\delta}\big)+\log\log\big(\frac{e}{1-\gamma}\big)}{N(1-\gamma)^3}} \nonumber \\
&	\leq 6 \sqrt{\frac{2\log\big(\frac{32 |\mathcal{S}||\mathcal{A}|}{(1-\gamma)^2\omega\delta} \big)}{N(1-\gamma)^3}}
	\label{eq:Vhat-Vstar-pistar}
\end{align}
with probability at least $1 - \delta$. This together with \eqref{eq:Vhat-Vstar-pistar-tot} and \eqref{eq:Vhat-pihat-bound-proof}  immediately establishes Lemma~\ref{lem:gap-separation-condition}.

\subsection{Proof of Lemma~\ref{LemTieBreaking}}
\label{sec:proof-lem-tie-breaking}

The proofs for $Q_{{\mathrm{p}}}^{\star}$ and
$\widehat{Q}_{{\mathrm{p}}}^{\star}$ are exactly the same;
for the sake of conciseness, we shall only provide the proof for $Q_{{\mathrm{p}}}^{\star}$.
Here we aim to prove a more general result than Lemma~\ref{LemTieBreaking}, namely, with probability at least $1-\delta$, 
\begin{align*}
	\forall s\in\mathcal{S}  \text{ and } a_{1},a_{2}\in\mathcal{A}  \text{ with }a_1\neq a_2: 
	\qquad 
	\big| Q_{{\mathrm{p}}}^{\star}(s,a_1)-Q_{{\mathrm{p}}}^{\star}(s,a_2)  \big|
	> \frac{\xi\delta(1-\gamma)}{4|\mathcal{S}||\mathcal{A}|^{2}}.
\end{align*}

Consider any state $s$ and any actions $a_1\neq a_2$. 
In what follows, we allow $r_{\mathrm{p}}(s,a_{1})=\tau$
to vary, while \emph{freezing} the values of all other rewards $\{ r_{\mathrm{p}}(\tilde{s},a) \mid 
(\tilde{s},a)\neq(s,a_{1})\}$. To streamline notation, we define
\begin{itemize}
\item $\bm{r}_{\tau}=[r_{\tau}(s,a)]_{(s,a)\in\mathcal{S}\times\mathcal{A}}$:
the reward vector obeying 
\[
	r_{\tau}(s,a_{1})=\tau \qquad\quad \text{and} \qquad\quad r_{\tau}(\tilde{s},a)=r_{\mathrm{p}}(\tilde{s},a) \quad \text{for all }(\tilde{s},a)\neq(s,a_{1});
\]
\item $Q_{\tau}^{\star}$: the optimal Q-function when the reward vector
is $\bm{r}_{\tau}$; 
\item $V_{\tau}^{\star}$: the optimal value function when the reward vector
is $\bm{r}_{\tau}$; 
\item $\pi_{\tau}^{\star}$: the optimal policy when the reward vector
is $\bm{r}_{\tau}$. 
\end{itemize}
Additionally, we claim for the moment that there exists a phase transition
boundary $\tau_{\mathrm{th}}$ such that
\begin{subequations}
\label{eq:pi-phase-transition}
\begin{align}
\pi_{\tau}^{\star}(s)\neq a_{1},\quad & \text{for all }\tau<\tau_{\mathrm{th}};\label{eq:pi-phase-transition-1}\\
\pi_{\tau}^{\star}(s)=a_{1},\quad & \text{for all }\tau>\tau_{\mathrm{th}}.\label{eq:pi-phase-transition-2}
\end{align}
\end{subequations}
The proof of this claim is deferred to the end of this section. 
To establish Lemma~\ref{LemTieBreaking}, the idea is to control the size of the set
\begin{align}
\label{eqn:setIzero}
\mathcal{I}_{0,\omega}:=\left\{ \tau\mid\big|Q_{\tau}^{\star}(s,a_{1})-Q_{\tau}^{\star}(s,a_{2})\big|<\omega\right\} 
\end{align}
for some $\omega > 0$ to be specified shortly. 
As motivated by \eqref{eq:pi-phase-transition}, we further break down this set into two parts $\mathcal{I}_{0,\omega} = \mathcal{I}_{1,\omega}\cup\mathcal{I}_{2,\omega}$, where
\begin{subequations}
\begin{align}
 \mathcal{I}_{1,\omega} &:=\left\{ \tau\mid\tau<\tau_{\mathrm{th}},\,\big|Q_{\tau}^{\star}(s,a_{1})-Q_{\tau}^{\star}(s,a_{2})\big|<\omega\right\}, \label{eq:defn-I1-omega}	\\
\mathcal{I}_{2,\omega} &:=\left\{ \tau\mid\tau \geq \tau_{\mathrm{th}},\,\big|Q_{\tau}^{\star}(s,a_{1})-Q_{\tau}^{\star}(s,a_{2})\big|<\omega\right\} \label{eq:defn-I2-omega}.
 \end{align} 
\end{subequations}
In what follows, we first control the size of each set separately, and then demonstrate that the probability of these events happening is very small. 

\paragraph{Step 1.} We begin with $ \mathcal{I}_{1,\omega} $ associated  with the range $\tau<\tau_{\mathrm{th}}$.
In this case, the value function $V_{\tau}^{\star}$ does not vary with $\tau$,
since the reward $r_{\mathrm{\tau}}(s,a_{1})=\tau$ is never active when calculating $V_{\tau}^{\star}$ (by virtue of \eqref{eq:pi-phase-transition-1}). Thus, the Bellman equation allows us to write
\[
Q_{\tau}^{\star}(s,a_{1})=\tau+B_{1}\qquad\text{and}\qquad Q_{\tau}^{\star}(s,a_{2})=B_{2}
\]
for some quantities $B_{1}$ and $B_{2}$, 
where neither $B_{1}$ nor $B_{2}$ depends on the value of $\tau$. 
Armed with this observation, we can easily show
that: for any $\omega>0$, the interval $\mathcal{I}_{1,\omega}$ (cf.~\eqref{eq:defn-I1-omega}) obeys
\begin{align*}
\mathcal{I}_{1,\omega}\subseteq\left\{ \tau\mid|\tau+B_{1}-B_{2}| < \omega\right\} ,
\end{align*}
and hence has length (or Lebesgue measure) at most $2\omega$.

\paragraph{Step 2.} We then move on to $ \mathcal{I}_{2,\omega} $ associated with the range $\tau>\tau_{\mathrm{th}}$ in which case $\pi_{\tau}^{\star}(s) = a_{1}$. 
Towards this, we first make some useful observations. 

\begin{itemize}
	\item To begin with, given the relation $\bm Q_{\tau}^{\star} = \bm r_{\tau} + \gamma \bm P \bm V_{\tau}^{\star}$, it is easily seen that for any $\tau_2 > \tau_1 > \tau_{\mathrm{th}}$,
	\begin{align}
		\bm 0 \le \bm Q_{\tau_2}^{\star} - \bm Q_{\tau_1}^{\star} \le \bm r_{\tau_2} - \bm r_{\tau_1} + \gamma \|\bm V_{\tau_2}^{\star} - \bm V_{\tau_1}^{\star}\|_{\infty}.
		\label{eq:Q-tau2-tau1-bound-V}
	\end{align}
	In addition, for any state-action pair $(\tilde{s},a)\neq(s,a_{1})$, by construction we have $r_{\tau_2}(\tilde{s},a) - r_{\tau_1}(\tilde{s},a) = 0$,
	which together with \eqref{eq:Q-tau2-tau1-bound-V} indicates that
	\begin{align} 
		\label{eq:Q_diff}
		\forall (\tilde{s},a)\neq(s,a_{1}):\quad 
		0\leq Q_{\tau_2}^{\star}(\tilde{s},a) - Q_{\tau_1}^{\star}(\tilde{s},a) \le \gamma \|\bm V_{\tau_2}^{\star} - \bm V_{\tau_1}^{\star}\|_{\infty}.
	\end{align}

	\item Next, observe that for any $\tau_2 > \tau_1 > \tau_{\mathrm{th}}$, 
	\begin{align}
		\forall s'\in \cS: \quad 
		0\leq V_{\tau_{2}}^{\star}(s')-V_{\tau_{1}}^{\star}(s')=\max_{a}Q_{\tau_{2}}^{\star}(s',a)-\max_{a}Q_{\tau_{1}}^{\star}(s',a)
		\leq \big\|\bm{Q}_{\tau_{2}}^{\star}-\bm{Q}_{\tau_{1}}^{\star} \big\|_{\infty}
	\end{align}
	and hence $\|\bm Q_{\tau_2}^{\star} - \bm Q_{\tau_1}^{\star}\|_{\infty} \ge \|\bm V_{\tau_2}^{\star} - \bm V_{\tau_1}^{\star}\|_{\infty}$. This combined with  \eqref{eq:Q_diff} and the fact $\gamma<1$ implies that
	\begin{align}
		Q_{\tau_2}^{\star}(s,a_{1}) - Q_{\tau_1}^{\star}(s,a_{1}) &\ge \|\bm V_{\tau_2}^{\star} - \bm V_{\tau_1}^{\star}\|_{\infty} ,
		\label{eq:Q-tau2-Q-tau1-gap-LB}
	\end{align}
	which together with the facts $V_{\tau_1}^{\star}(s) = Q_{\tau_1}^{\star}(s,a_{1})$ and  $V_{\tau_2}^{\star}(s) = Q_{\tau_2}^{\star}(s,a_{1})$ (by virtue of \eqref{eq:pi-phase-transition-2}) yields
	\begin{align*}
		V_{\tau_2}^{\star}(s) &- V_{\tau_1}^{\star}(s) = Q_{\tau_2}^{\star}(s,a_{1}) - Q_{\tau_1}^{\star}(s,a_{1}) \ge \|\bm V_{\tau_2}^{\star} - \bm V_{\tau_1}^{\star}\|_{\infty} \\
		& \Longrightarrow \quad Q_{\tau_2}^{\star}(s,a_{1}) - Q_{\tau_1}^{\star}(s,a_{1}) = \|\bm V_{\tau_2}^{\star} - \bm V_{\tau_1}^{\star}\|_{\infty}. 
	\end{align*}
	Invoke the Bellman equation to further derive 
	\begin{align}
		Q_{\tau_2}^{\star}(s,a_{1}) - Q_{\tau_1}^{\star}(s,a_{1}) 
		& =  \|\bm V_{\tau_2}^{\star} - \bm V_{\tau_1}^{\star}\|_{\infty} 
		 = 	\big\|\bm r_{\tau_2}  + \gamma \bm{P} \bm V_{\tau_2}^{\star} - \bm r_{\tau_1} - \gamma \bm{P} \bm V_{\tau_1}^{\star} \big\|_{\infty}  \notag\\
		&\geq \big\|\bm r_{\tau_2} - \bm{r}_{\tau_1} \big\|_{\infty} = \tau_2 - \tau_1,
		\label{eq:Q_diff_1}
	\end{align}
		where the last inequality holds since $\bm r_{\tau_2}  + \gamma \bm{P} \bm V_{\tau_2}^{\star} - \bm r_{\tau_1} - \gamma \bm{P} \bm V_{\tau_1}^{\star}\geq \bm r_{\tau_2}-\bm r_{\tau_1}\geq \bm{0}$ (due to the monotonicity properties $\bm r_{\tau_2}\geq \bm r_{\tau_1}$ and  $\bm V_{\tau_2}^{\star}\geq \bm V_{\tau_1}^{\star}$), and the last identity follows from the definition of $\bm{r}_{\tau}$. 


\end{itemize}

With the above two properties \eqref{eq:Q_diff} and \eqref{eq:Q_diff_1} in mind, 
we are ready to locate $\mathcal{I}_{2,\omega}$ by showing that 
\begin{align}
	\mathcal{I}_{2,\omega}\subseteq \Big[\tau_{\mathrm{th}},\tau_{\mathrm{th}}+\frac{\omega}{1-\gamma}\Big]. \label{eq:I2-omega-range}
\end{align}
Given that $Q_{\tau_{\mathrm{th}}}^{\star}(s,a_{1}) \ge Q_{\tau_{\mathrm{th}}}^{\star}(s,a_{2})$ (in view of \eqref{eq:pi-phase-transition}),
we have for any $\tau\geq \tau_{\mathrm{th}}$ and any $a_2\neq a_1$ that
\begin{align}
	Q_{\tau}^{\star}(s,a_{1}) - Q_{\tau}^{\star}(s,a_{2}) 
	&\ge \big(Q_{\tau}^{\star}(s,a_{1}) - Q_{\tau_{\mathrm{th}}}^{\star}(s,a_{1}) \big) - \big( Q_{\tau}^{\star}(s,a_{2}) - Q_{\tau_{\mathrm{th}}}^{\star}(s,a_{2}) \big) \notag\\
	&\ge (1-\gamma) \big(Q_{\tau}^{\star}(s,a_{1}) - Q_{\tau_{\mathrm{th}}}^{\star}(s,a_{1})\big) .
	\label{eq:Q-s-a1-a2-gap-LB}
\end{align}
Here, the last inequality holds since 
\begin{align*}
	Q_{\tau}^{\star}(s,a_{2}) - Q_{\tau_{\mathrm{th}}}^{\star}(s,a_{2}) 
	\overset{\mathrm{(i)}}{\le} \gamma \big\| \bm{V}_{\tau}^{\star} - \bm{V}_{\tau_{\mathrm{th}}}^{\star} \big\|_{\infty}
	\overset{\mathrm{(ii)}}{\le} \gamma \big( Q_{\tau}^{\star}(s,a_{1}) - Q_{\tau_{\mathrm{th}}}^{\star}(s,a_{1}) \big),
\end{align*}
where (i) follows from \eqref{eq:Q_diff} and (ii) is due to \eqref{eq:Q-tau2-Q-tau1-gap-LB}. 
As a result, for any $\tau > \tau_{\mathrm{th}}+\frac{\omega}{1-\gamma}$, one can invoke \eqref{eq:Q-s-a1-a2-gap-LB} and \eqref{eq:Q_diff_1} to see that
\begin{align*}
	\forall a_2\neq a_1:\quad 
	Q_{\tau}^{\star}(s,a_{1}) - Q_{\tau}^{\star}(s,a_{2}) \ge (1-\gamma) \big( Q_{\tau}^{\star}(s,a_{1}) - Q_{\tau_{\mathrm{th}}}^{\star}(s,a_{1}) \big)
	\geq (1-\gamma)(\tau - \tau_{\mathrm{th}}) > \omega,
\end{align*}
which necessarily implies that such a $\tau$ does not lie within the interval $\mathcal{I}_{2,\omega}$ as defined in 
\eqref{eq:defn-I2-omega}. This establishes the claimed relation \eqref{eq:I2-omega-range}.

\paragraph{Step 3.} Putting the results in the above two steps together, we see the set $\mathcal{I}_{0,\omega}$ (cf.~\eqref{eqn:setIzero})
has total length (or Lebesgue measure) at most $\frac{3\omega}{1-\gamma}$.
Given that $r_{\mathrm{p}}(s,a)=r(s,a)+\zeta(s,a)$ with $\zeta(s,a)\sim\mathsf{Unif}(0,\xi)$,
one has
\begin{align*}
\mathbb{P}\left\{ \big|Q_{{\mathrm{p}}}^{\star}(s,a_{1})-Q_{{\mathrm{p}}}^{\star}(s,a_{2})\big|<\omega\right\}  & \leq\mathbb{P}\left\{ r(s,a)+\zeta(s,a)\in\mathcal{I}_{0,\omega}\right\} 
  \leq\frac{3\omega}{\xi(1-\gamma)}.
\end{align*}
By setting $\omega=\frac{\delta(1-\gamma)\xi}{3|\mathcal{S}||\mathcal{A}|^{2}}$,
we arrive at
\begin{align*}
\mathbb{P}\left\{ \big|Q_{{\mathrm{p}}}^{\star}(s,a_{1})-Q_{{\mathrm{p}}}^{\star}(s,a_{2})\big|<\frac{\delta(1-\gamma)\xi}{3|\mathcal{S}||\mathcal{A}|^{2}}\right\}  & \leq\frac{\delta}{|\mathcal{S}||\mathcal{A}|^{2}}.
\end{align*}
Finally, taking the union bound over all $s,a_{1},a_{2}$, we conclude
that
\begin{align*}
\mathbb{P}\left\{ \exists s,a_{1}\neq a_{2}:\text{ }\big|Q_{{\mathrm{p}}}^{\star}(s,a_{1})-Q_{{\mathrm{p}}}^{\star}(s,a_{2})\big|<\frac{\delta(1-\gamma)\xi}{3|\mathcal{S}||\mathcal{A}|^{2}}\right\}  & \leq\delta ,
\end{align*}
thus establishing Lemma~\ref{LemTieBreaking} as long as the claim \eqref{eq:pi-phase-transition} is valid. 

\begin{proof}[Proof of the claim \eqref{eq:pi-phase-transition}]To
	establish the claim, it suffices to take
\begin{equation}
\tau_{\mathrm{th}}=\sup\left\{ u\mid\pi_{\tau}^{\star}(s)\neq a_{1}\text{ for all }\tau<u\right\} .\label{eq:defn-tau-th}
\end{equation}
%
	It thus suffices to verify \eqref{eq:pi-phase-transition-2} for
our choice \eqref{eq:defn-tau-th}. Towards this, suppose instead
that there exist some $\tau_{3} < \tau_{\mathrm{th}} \leq \tau_{2} < \tau_{1}$
such that 
\[
\pi_{\tau_{3}}^{\star}(s)\neq a_{1},\qquad\pi_{\tau_{2}}^{\star}(s)=a_{1},\qquad\text{and}\qquad\pi_{\tau_{1}}^{\star}(s)\neq a_{1}.
\]
It is straightforward to see that $V_{\tau_{1}}^{\star}=V_{\tau_{3}}^{\star}$,
since in both cases, the reward $r_{\mathrm{\tau}}(s,a_{1})$ does
	not enter the calculation of the optimal value function (while the rewards in other state-action pairs 
	are identical in both cases). 
In view of the monotonicity of the value function w.r.t.~the reward vector,
we have
\[
V_{\tau_{1}}^{\star}=V_{\tau_{2}}^{\star} = V_{\tau_{3}}^{\star}.
\]
However, this contradicts our assumption that $a_{1}$ is
	the optimal action for state $s$ at $\tau_{2}$ but not at $\tau_3$,
since enlarging $\tau_{2}$ to $\tau_{3}$ otherwise will enlarge the optimal value function
$V_{\tau_{3}}^{\star}$. 
We have thus established \eqref{eq:pi-phase-transition}. \end{proof}

\subsection{Proof of Lemma~\ref{lem:upper-bound-picon-Vhat-Vstar}}
\label{sec:proof-inequality-upper-bound-picon-Vhat-Vstar}
To begin, 
we find it helpful to introduce a modified reward function $\widetilde{\bm r} \in \mathbb{R}^{|\mathcal{S}||\mathcal{A}|}$ as follows
\begin{equation}
\widetilde{r}(s, a) \defn 
\begin{cases}
r(s, a) + \widehat{V}^{\star}(s) - \widehat{Q}^{\star}(s, a), & \text{if}~a = \picon(s), \\
r(s, a), & \text{otherwise}.
\end{cases}
\label{eq:new-r-tilde-definition}
\end{equation}
Armed with this new reward function, we subsequently define a vector $ \widetilde{\bm Q}= [\widetilde{Q}(s,a)]$ as follows
\begin{align}
	\widetilde{\bm Q} \defn \widetilde{\bm r} + \gamma \widehat{\bm P}\widehat{\bm V}^{\star}.
	\label{eq:construction-Q-widetilde-approx}
\end{align}
In view of the Bellman optimality equation $\widehat{\bm Q}^{\star} = \bm r + \gamma \widehat{\bm P}\widehat{\bm V}^{\star}$, 
we see that $\widetilde{\bm Q} $ satisfies $\widetilde{\bm{Q}}=\widetilde{\bm{r}}+\widehat{\bm{Q}}^{\star}-\bm{r}$, which combined with the construction \eqref{eq:new-r-tilde-definition} gives
\[
\forall (s,a) \in \cS \times \cA: \qquad 
\widetilde{Q}(s,a)=\begin{cases}
\widehat{V}^{\star}(s), & \text{if }a=\picon(s),\\
\widehat{Q}^{\star}(s,a), \qquad & \text{else}.
\end{cases}
\]
As a consequence, it is easily seen that
\[
	\forall s\in\cS:\qquad\max_{a}\widetilde{Q}(s,a)=\widehat{V}^{\star}(s)
	\qquad\text{and}\qquad
	\widetilde{Q}\big(s,\picon(s)\big)=\widehat{V}^{\star}(s)
\]
This taken collectively with \eqref{eq:construction-Q-widetilde-approx} demonstrates that $\widetilde{\bm Q}$ and $\widehat{\bm V}^{\star}$ 
are respectively the optimal Q-function and optimal value function of the MDP $\widetilde{\mathcal{M}}=(\cS,\cA, \widetilde{r}, \widehat{P}, \gamma)$, since they satisfy the Bellman optimality condition w.r.t.~$\widetilde{\mathcal{M}}$. 
	In addition, if we let $\widetilde{\bm{V}}^{\picon}$ represent the value function of the policy $\picon$ in $\widetilde{\mathcal{M}}$, then the preceding relation clearly implies that $\widetilde{\bm{V}}^{\picon} = \widehat{\bm{V}}^{\star}$.

Using the above properties, one can deduce that
\begin{align*}
	\Vhat^{\star}-\Vhat^{\picon} &=  \widetilde{\bm{V}}^{\picon}-\Vhat^{\picon} \stackrel{\mathrm{(i)}}{=} 
	\big(\bm{I}-\gamma\widehat{\bm{P}}_{\picon}\big)^{-1}\widetilde{\bm{r}} 
	- \big(\bm{I}-\gamma\widehat{\bm{P}}_{\picon}\big)^{-1} \bm{r}
	= \big(\bm{I}-\gamma\widehat{\bm{P}}_{\picon}\big)^{-1}\big(\widetilde{\bm{r}}-\bm{r}\big)\\
	& \leq  \big\| \big(\bm{I}-\gamma\widehat{\bm{P}}_{\picon}\big)^{-1} \big\|_1 \|\widetilde{\bm{r}}-\bm{r} \|_{\infty} \one
	\stackrel{\mathrm{(ii)}}{=}  \frac{1}{1-\gamma}\|\widetilde{\bm{r}}-\bm{r}\|_{\infty}\one
	\stackrel{\mathrm{(iii)}}{\leq}~\frac{\newnoise}{1-\gamma}\one
	\stackrel{\mathrm{(iv)}}{\leq}~\frac{\xi}{1-\gamma}\one. 
\end{align*}
Here, (i) is due to the Bellman equation, (ii) relies on the fact $\| \big(\bm{I}-\gamma\widehat{\bm{P}}_{\picon}\big)^{-1} \|_1 = \frac{1}{1-\gamma}$, (iii) arises since $\widehat{V}^{\star}(s) - \widehat{Q}^{\star}(s, \picon(s)) \leq \newnoise $ by construction of $\picon$,
whereas (iv) is valid since $\newnoise\in [0, \xi]$. The lemma then follows by recognizing that $ \Vhat^{\pistar}\leq 	\Vhat^{\star}$ due to the optimality of $\Vhat^{\star}$.
%

\subsection{Proof of Lemma~\ref{lem:picon-is-stable}}
\label{sec:proof-lem:picon-is-stable}

We first make the key observation that, with probability at least $1 - \delta$, the following event holds true:
\begin{align}
	\label{eqn:two-side-relation}
	\forall (s,a) \in \mathcal{S} \times \mathcal{A}, \qquad 
	\big| \widehat{V}(s) - \widehat{Q}(s,a) - \newnoise \big| > \frac{\xi \delta}{2|\mathcal{S}||\mathcal{A}|}. 
\end{align}
To justify this claim, note that its complementary event satisfies
\begin{align*}
	\mprob \left(\exists(s,a) \in \mathcal{S} \times \mathcal{A}, ~|\widehat{V}(s) - \widehat{Q}(s,a) - \newnoise | 
	\leq \frac{\xi \delta}{2|\mathcal{S}||\mathcal{A}|}\right)
	&\leq 
	\sum_{s,a}\mprob \left(\newnoise\in\Big[\widehat{V}(s)-\widehat{Q}(s,a)\pm\frac{\xi\delta}{2|\mathcal{S}||\mathcal{A}|}\Big]\right) \\
	& \leq 	
	|\mathcal{S}||\mathcal{A}| \cdot \frac{\xi \delta}{\xi|\mathcal{S}||\mathcal{A}|} = \delta,
\end{align*}
where the first inequality applies the union bound, and the last line follows since $\newnoise  \sim \mathsf{Unif}(0,\xi)$. Here, we abbreviate $[a\pm b] := [a-b,a+b]$. 
Combining \eqref{eqn:two-side-relation} with the assumption \eqref{eq:Q-Qhat-distance-bound}, we further reach
\begin{align*}
	\big|V(s)-Q(s,a)-\newnoise\big|\geq\big|\widehat{V}(s)-\widehat{Q}(s,a)-\newnoise\big|-2\max_{s,a}\big|Q(s,a)-\widehat{Q}(s,a)\big|
	> \frac{\xi\delta}{4|\mathcal{S}||\mathcal{A}|}	 
\end{align*}
for all $(s,a)\in \cS\times \cA$, or equivalently, 
\begin{align}
	\label{eqn:two-side-relation-Q-Qhat}
	V(s)-Q(s,a)< \newnoise - \frac{\xi\delta}{4|\mathcal{S}||\mathcal{A}|}	 
	~~ \text{ or }~~
	V(s)-Q(s,a)> \newnoise + \frac{\xi\delta}{4|\mathcal{S}||\mathcal{A}|}
\end{align}
holds for all $(s,a)\in \cS\times \cA$.

We are now prepared to justify the claim of this lemma. 
To begin with,  consider any action $\widehat{a} \in \{a \in \mathcal{A}: Q(s,a) > V(s) - \newnoise \}$. Comparing this with the condition \eqref{eqn:two-side-relation-Q-Qhat}, we can easily see that the only possibility is 
\[
	V(s)-Q(s, \widehat{a}) <\newnoise-\frac{\xi\delta}{4|\mathcal{S}||\mathcal{A}|}.
\]
Therefore, by invoking a basic decomposition, we ensure that 
\begin{align*}
	\notag \widehat{V}(s) - \widehat{Q}(s,\ahat) 
	& = \widehat{V}(s) - V(s) + V(s) - 	Q(s,\ahat) + Q(s,\ahat) - \widehat{Q}(s, \ahat) \\
	\notag &\leq V(s) - Q(s,\ahat) + \big| V(s) - \widehat{V}(s) \big| + \big| \widehat{Q}(s,\ahat) - Q(s, \ahat) \big| \\
	\notag &< \newnoise  - \frac{\xi\delta}{4|\mathcal{S}||\mathcal{A}|} + 2\max_{s,a}\big|Q(s,a)-\widehat{Q}(s,a)\big| \\
	& \leq \newnoise  - \frac{\xi\delta}{4|\mathcal{S}||\mathcal{A}|} + 2\cdot \frac{\xi\delta}{8|\mathcal{S}||\mathcal{A}|} 
	= \newnoise .
\end{align*}
This essentially implies that
\begin{align}
\label{eqn:subset-I} 
	\big\{a \in \mathcal{A}: Q(s,a) > V(s) - \newnoise \big\} 
	~\subseteq~ 
	\big\{a \in \mathcal{A}: \widehat{Q}(s,a) > \widehat{V}(s) - \newnoise \big\}.
\end{align}
%
Applying exactly the same argument for any action $\widehat{a} \in \{a \in \mathcal{A}: Q(s,a) \leq V(s) - \newnoise \}$, we can also derive
%
\begin{align}
\label{eqn:subset-I-reverse} 
	\big\{a \in \mathcal{A}: \widehat{Q}(s,a) > \widehat{V}(s) - \newnoise \big\} 
	~\subseteq~ 
	\big\{a \in \mathcal{A}: Q(s,a) > V(s) - \newnoise \big\}.
\end{align}
These two set inequalities taken collectively establish the lemma. 

\section{Proofs of auxiliary lemmas: finite-horizon MDPs}
\label{sec:proof-auxiliary-finite}

To begin, we find it helpful to control the entrywise magnitudes of $\{\bm{V}^{(l)}_h\}$. 
This is accomplished via the following lemma, 
with the proof postponed to Section~\ref{sec:proof-lem:Vhl-inf-norm-UB-finite}. 
\begin{lemma}
\label{lem:Vhl-inf-norm-UB-finite}
The vectors $\big\{ \bm{V}^{(l)}_h \big\}$ defined in \eqref{defn-r-v-iterate-finite} obey 
\begin{align} \label{eq:Vl-finite}
\big\|\bm{V}^{(l)}_h\big\|_{\infty} \le \big(\sqrt{3H}\big)^l H 
\end{align}
for all $1 \le h \le H$ and all $l\geq 0$. 
\end{lemma}

\subsection{Proof of Lemma~\ref{lemma:dependent-policy-error-finite}}
\label{sec:proof-lemma:dependent-policy-error-finite}
Consider any $l$ obeying $1 \le l \le m \defn \log_2 H$.
By construction (cf.~\eqref{defn-r-v-iterate-finite}), we see that
\begin{align*}
\widehat{\bm{V}}^{(l)}_h - \bm{V}^{(l)}_h &= \widehat{\bm P}_{h,\pi} \widehat{\bm V}^{(l)}_{h+1} - {\bm P}_{h,\pi} \bm V^{(l)}_{h+1} \\
&= \widehat{\bm P}_{h,\pi} \big(\widehat{\bm V}^{(l)}_{h+1} - \bm V^{(l)}_{h+1}\big) + \big(\widehat{\bm P}_{h,\pi} - {\bm P}_{h,\pi}\big) \bm V^{(l)}_{h+1},
\end{align*}
which connects $\bm V^{(l)}_{h}$ and $\widehat{\bm V}^{(l)}_{h}$ with $\bm V^{(l)}_{h+1}$ and $\widehat{\bm V}^{(l)}_{h+1}$. 
Apply the above relation recursively and make use of the conditions \eqref{defn-r-v-iterate-finite-Hplus1} to arrive at
\begin{align*}
\widehat{\bm{V}}^{(l)}_h - \bm{V}^{(l)}_h = \sum_{j = h}^{H-1} \prod_{i = h}^{j-1} \widehat{\bm P}_{i,\pi} \big(\widehat{\bm P}_{j,\pi} - \bm P_{j,\pi}\big) \bm V^{(l)}_{j+1} ,
\end{align*}
where we adopt the convenient notation and let $\prod_{i = h}^{h-1} \widehat{\bm P}_{i,\pi} = \bm{I}.$
%
	
%
According to the triangle inequality, we can further deduce that
\begin{align}
\big|\widehat{\bm{V}}^{(l)}_h - \bm{V}^{(l)}_h\big| &\le \sum_{j = h}^{H-1} \prod_{i = h}^{j-1} \widehat{\bm P}_{i, \pi} \Big|\big(\widehat{\bm P}_{j,\pi} - \bm P_{j,\pi}\big) \bm V^{(l)}_{j+1}\Big| \notag\\
&\le \sum_{j = h}^{H-1} \prod_{i = h}^{j-1} \widehat{\bm P}_{i,\pi}\bigg(\sqrt{\frac{\beta_1}{N}}\sqrt{\mathsf{Var}_{\bm P_{j,\pi}}\big[\bm{V}^{(l)}_{j+1}\big]} + \frac{\beta_1 \big\|\bm{V}^{(l)}_{j+1}\big\|_{\infty}}{N}\bm{1}\bigg) \notag\\
& = \sqrt{\frac{\beta_{1}}{N}}\bigg(\sum_{j=h}^{H-1}\prod_{i=h}^{j-1}\widehat{\bm{P}}_{i,\pi}\bm{r}_{j}^{(l+1)}\bigg)+\sum_{j=h}^{H-1}\frac{\beta_{1}\big\|\bm{V}_{j+1}^{(l)}\big\|_{\infty}}{N}\bm{1} \notag\\
& \leq \sqrt{\frac{\beta_{1}}{N}}\bigg(\sum_{j=h}^{H-1}\prod_{i=h}^{j-1}\widehat{\bm{P}}_{i,\pi}\bm{r}_{j}^{(l+1)}\bigg)+
	\frac{\beta_{1} H \max_{j}\big\|\bm{V}_{j}^{(l)}\big\|_{\infty}}{N}\bm{1} ,
\label{eq:Vl-diff-UB-123}
\end{align}
where the second line follows from the assumption \eqref{eq:Bernstein-Vl-general-finite}, 
and the third line makes use of the definition \eqref{defn-r-v-iterate-finite} of $\bm{r}_h^{(l)}$ and the elementary identity $\prod_{i = h}^{j-1} \widehat{\bm P}_{i,\pi}\bm{1} = \bm{1}$ (since each $\widehat{\bm P}_{i,\pi}$ is a probability transition matrix). 
In view of the construction \eqref{defn-r-v-iterate-finite}, we can also derive recursively that
\begin{align}
\widehat{\bm{V}}_{h}^{(l)}=\sum_{j=h}^{H-1}\prod_{i=h}^{j-1}\widehat{\bm{P}}_{i,\pi}\bm{r}_{j}^{(l)},
\end{align}	
which combined with \eqref{eq:Vl-diff-UB-123} yields
\begin{align}
\big|\widehat{\bm{V}}^{(l)}_h - \bm{V}^{(l)}_h\big| &\le 
\sqrt{\frac{\beta_1}{N}}\widehat{\bm{V}}^{(l+1)}_h + \frac{\beta_1 H\max_{j}\big\|\bm{V}_{j}^{(l)}\big\|_{\infty}}{N}\bm{1}. 
\label{eq:Vl-diff-UB-456}
\end{align}
%


Further, the above inequality together with the triangle inequality immediately results in the following recursive relation
\begin{align*}
\big\|\widehat{\bm{V}}_{h}^{(l)}-\bm{V}_{h}^{(l)}\big\|_{\infty} & \le\sqrt{\frac{\beta_{1}}{N}}\big\|\widehat{\bm{V}}_{h}^{(l+1)}\big\|_{\infty}+\frac{\beta_{1}H}{N}\max_{j}\big\|\bm{V}_{j}^{(l)}\big\|_{\infty}\\
 & \le\sqrt{\frac{\beta_{1}}{N}}\big\|\widehat{\bm{V}}_{h}^{(l+1)}-\bm{V}_{h}^{(l+1)}\big\|_{\infty}+\sqrt{\frac{\beta_{1}}{N}}\big\|\bm{V}_{h}^{(l+1)}\big\|_{\infty}+\frac{\beta_{1}H}{N}\max_{j}\big\|\bm{V}_{j}^{(l)}\big\|_{\infty},
\end{align*}
thus revealing a useful connection between $\|\widehat{\bm{V}}_{h}^{(l)}-\bm{V}_{h}^{(l)}\|_{\infty}$
and $\|\widehat{\bm{V}}_{h}^{(l+1)}-\bm{V}_{h}^{(l+1)}\|_{\infty}$. 
Applying this relation recursively with a little algebra leads to 
\begin{align}
\big\|\widehat{\bm{V}}^{(0)}_h - \bm{V}^{(0)}_h\big\|_{\infty} &\le \bigg(\sqrt{\frac{\beta_1}{N}}\bigg)^m\big\|\widehat{\bm{V}}^{(m)}_h - \bm{V}^{(m)}_h\big\|_{\infty} \notag\\
& \qquad + \sum_{l = 1}^m \bigg(\sqrt{\frac{\beta_1}{N}}\bigg)^{l}\big\|\bm{V}^{(l)}_h\|_{\infty} + \frac{\beta_1 H}{N}\sum_{l = 1}^{m} \bigg(\sqrt{\frac{\beta_1}{N}}\bigg)^{l-1}\max_{j}\big\|\bm{V}_{j}^{(l-1)}\big\|_{\infty} .
	\label{eq:V0-hat-diff-UB123}
\end{align}
Additionally, it is easily seen from the definition \eqref{defn-r-v-iterate-finite} that
\[
	\bm{r}_{h}^{(l+1)}\leq\bm{V}_{h+1}^{(l)} \leq \big \| \bm{V}_{h+1}^{(l)}  \big\|_{\infty} \bm{1}, 
\]
which taken together with \eqref{eq:Vl-diff-UB-123} and the elementary identity $\prod_{i = h}^{j-1} \widehat{\bm P}_{i,\pi}\bm{1} = \bm{1}$ implies that
\begin{align*}
\big|\widehat{\bm{V}}_{h}^{(m)}-\bm{V}_{h}^{(m)}\big| & \leq\sqrt{\frac{\beta_{1}}{N}}\max_{j}\big\|\bm{r}_{j}^{(m+1)}\big\|_{\infty}\bigg(\sum_{j=h}^{H-1}\prod_{i=h}^{j-1}\widehat{\bm{P}}_{i,\pi}\bm{1}\bigg)+\frac{\beta_{1}H\max_{j}\big\|\bm{V}_{j}^{(m)}\big\|_{\infty}}{N}\bm{1}\\
 & \leq\sqrt{\frac{\beta_{1}}{N}}\max_{j}\big\|\bm{V}_{j}^{(m)}\big\|_{\infty}\bigg(\sum_{j=h}^{H-1}\bm{1}\bigg)+\frac{\beta_{1}H\max_{j}\big\|\bm{V}_{j}^{(m)}\big\|_{\infty}}{N}\bm{1}\\
 & \leq H\left(\sqrt{\frac{\beta_{1}}{N}}+\frac{\beta_{1}}{N}\right)\max_{j}\big\|\bm{V}_{j}^{(m)}\big\|_{\infty}\bm{1}.
\end{align*}
Substitution into \eqref{eq:V0-hat-diff-UB123} results in
\begin{align}
\big\|\widehat{\bm{V}}^{(0)}_h - \bm{V}^{(0)}_h\big\|_{\infty} 
&\le H\bigg(\sqrt{\frac{\beta_1}{N}}\bigg)^m\bigg(\sqrt{\frac{\beta_1}{N}} + \frac{\beta_1 }{N}\bigg)\max_j\big\|\bm{V}^{(m)}_{j}\big\|_{\infty} \notag\\
	& \qquad + \sum_{l = 1}^m \bigg(\sqrt{\frac{\beta_1}{N}}\bigg)^{l}\big\|\bm{V}^{(l)}_h\big\|_{\infty} + \frac{\beta_1 H}{N}\sum_{l = 1}^{m} \bigg(\sqrt{\frac{\beta_1}{N}}\bigg)^{l-1}\max_j\big\|\bm{V}^{(l-1)}_{j}\big\|_{\infty}. 
	\label{eq:V0-hat-diff-UB456}
\end{align}
%
%

To finish up, it remains to control the terms $\big\{ \|\bm{V}^{(l)}_h\|_{\infty} \big\}$.   
Towards this, combining Lemma~\ref{lem:Vhl-inf-norm-UB-finite} with the above inequality \eqref{eq:V0-hat-diff-UB456} yields
\begin{align}
\big\|\widehat{\bm{V}}_{h}^{(0)}-\bm{V}_{h}^{(0)}\big\|_{\infty} & \le H^{2}\bigg(\sqrt{\frac{3\beta_{1}H}{N}}\bigg)^{m}\Big(\sqrt{\frac{\beta_{1}}{N}}+\frac{\beta_{1}}{N}\Big)+\sum_{l=1}^{m}H\bigg(\sqrt{\frac{3\beta_{1}H}{N}}\bigg)^{l}+\sum_{l=1}^{m}\frac{\beta_{1}H^{2}}{N}\bigg(\sqrt{\frac{3\beta_{1}H}{N}}\bigg)^{l-1} \notag\\
 & =H^{2}\bigg(\sqrt{\frac{3\beta_{1}H}{N}}\bigg)^{m}\Big(\sqrt{\frac{\beta_{1}}{N}}+\frac{\beta_{1}}{N}\Big)+\left\{ H\sqrt{\frac{3\beta_{1}H}{N}}+\frac{\beta_{1}H^{2}}{N}\right\} \sum_{l=1}^{m}\bigg(\sqrt{\frac{3\beta_{1}H}{N}}\bigg)^{l-1} \notag\\
 & \leq H^{2}\bigg(\sqrt{\frac{3\beta_{1}H}{N}}\bigg)^{m}\Big(\sqrt{\frac{\beta_{1}}{N}}+\frac{\beta_{1}}{N}\Big)+2\left\{ H\sqrt{\frac{3\beta_{1}H}{N}}+\frac{\beta_{1}H^{2}}{N}\right\} ,
	\label{eq:V0-hat-diff-UB789}
\end{align}
Here, the last inequality holds true since
\[
\sum_{l=1}^{m}\bigg(\sqrt{\frac{3\beta_{1}H}{N}}\bigg)^{l-1}\leq\frac{1}{1-\sqrt{\frac{3\beta_{1}H}{N}}}\leq2,
\]
provided that $\sqrt{\frac{3\beta_{1}H}{N}}\leq1/2$. Invoking the assumption $N\geq 12 H \beta_1$ again and taking $m=\log_2 H$, 
we have $\big(\sqrt{\frac{3\beta_{1}H}{N}}\big)^{m} \leq 1/H$. This combined with \eqref{eq:V0-hat-diff-UB789} immediately leads to
\begin{align}
\big\|\widehat{\bm{V}}_{h}^{(0)}-\bm{V}_{h}^{(0)}\big\|_{\infty} 
 & \leq 3\left\{ H\sqrt{\frac{3\beta_{1}H}{N}}+\frac{\beta_{1}H^{2}}{N}\right\} \leq 6 H\sqrt{\frac{3\beta_{1}H}{N}},
	\label{eq:V0-hat-diff-UB789}
\end{align}
with the proviso that $N\geq 3\beta_1 H$. This concludes the proof. 
%
%
%

\subsection{Proof of Lemma~\ref{lem:Vhl-inf-norm-UB-finite}}
\label{sec:proof-lem:Vhl-inf-norm-UB-finite}
To begin with, it is seen from the notation \eqref{eq:matrix-VarP-V-expression} that for all $1 \le j \le H$,
\begin{align}
\mathsf{Var}_{\bm P_{j, \pi}}\big(\bm{V}^{(l)}_{j+1}\big) &= \bm P_{j,\pi}\big(\bm{V}^{(l)}_{j+1} \circ \bm{V}^{(l)}_{j+1}\big) - \big(\bm P_{j,\pi}\bm{V}^{(l)}_{j+1}\big) \circ \big(\bm P_{j,\pi}\bm{V}^{(l)}_{j+1}\big) \nonumber\\
&= \bm P_{j,\pi}\big(\bm{V}^{(l)}_{j+1} \circ \bm{V}^{(l)}_{j+1}\big) - \big(\bm{V}^{(l)}_{j}-\bm r_{j}^{(l)}\big) \circ \big(\bm{V}^{(l)}_{j}-\bm r_{j}^{(l)}\big) \nonumber\\
&= \bm P_{j,\pi}\big(\bm{V}^{(l)}_{j+1} \circ \bm{V}^{(l)}_{j+1}\big) - \bm{V}^{(l)}_{j} \circ \bm{V}^{(l)}_{j} + 2 \bm{V}^{(l)}_{j} \circ \bm r_{j}^{(l)} - \bm r_{j}^{(l)} \circ \bm r_{j}^{(l)} \notag\\
&\le \bm P_{j,\pi}\big(\bm{V}^{(l)}_{j+1} \circ \bm{V}^{(l)}_{j+1}\big) - \bm{V}^{(l)}_{j} \circ \bm{V}^{(l)}_{j} + 2 \bm{V}^{(l)}_{j} \circ \bm r_{j}^{(l)} \label{eq:var-finite}, 
\end{align}
where the second identity makes use of the fact that $\bm{V}^{(l)}_{j}=\bm r_{j}^{(l)}+\bm P_{j,\pi}\bm{V}^{(l)}_{j+1}$ (cf.~\eqref{defn-r-v-iterate-finite}).
Moreover, 
from the construction \eqref{defn-r-v-iterate-finite} we can easily derive 
\begin{align}
	{\bm{V}}_{h}^{(l)}=\sum_{j=h}^{H-1}\prod_{i=h}^{j-1}{\bm{P}}_{i,\pi}\bm{r}_{j}^{(l)}
	= \sum_{j=h}^{H-1}\prod_{i=h}^{j-1}{\bm{P}}_{i,\pi}\sqrt{\mathsf{Var}_{\bm P_{j,\pi}}\big[\bm{V}^{(l-1)}_{j+1}\big]} .
	\label{eq:Vh-l-expression-long}
\end{align}	
The above two results taken collectively give
\begin{align*}
\big|\bm{V}^{(l+1)}_h\big| &= \sum_{j = h}^{H-1} \prod_{i = h}^{j-1} \bm P_{i,\pi}\sqrt{\mathsf{Var}_{\bm P_{j,\pi}}\big[\bm{V}^{(l)}_{j+1}\big]} 
\overset{\mathrm{(i)}}{\leq} \sum_{j = h}^{H-1} \sqrt{\prod_{i = h}^{j-1} \bm P_{i,\pi} \mathsf{Var}_{\bm P_{j,\pi}}\big[\bm{V}^{(l)}_{j+1}\big]} \\
&\overset{\mathrm{(ii)}}{\leq} \sqrt{H\sum_{j = h}^{H-1} \prod_{i = h}^{j-1} \bm P_{i,\pi}\mathsf{Var}_{\bm P_{j,\pi}}\big[\bm{V}^{(l)}_{j+1}\big]} \\
&\overset{\mathrm{(iii)}}{\leq} \sqrt{H\sum_{j = h}^{H-1} \prod_{i = h}^{j-1} \bm P_{i,\pi}\bigg[\bm P_{j,\pi}\big(\bm{V}^{(l)}_{j+1} \circ \bm{V}^{(l)}_{j+1}\big) - \bm{V}^{(l)}_{j} \circ \bm{V}^{(l)}_{j} + 2 \bm{V}^{(l)}_{j} \circ \bm r_{j}^{(l)}\bigg]} \\
& =  \sqrt{H\sum_{j = h}^{H-1} \Big( \prod_{i = h}^{j} \bm P_{i,\pi}  \big(\bm{V}^{(l)}_{j+1} \circ \bm{V}^{(l)}_{j+1}\big) - \prod_{i = h}^{j-1} \bm P_{i,\pi}  \big( \bm{V}^{(l)}_{j} \circ \bm{V}^{(l)}_{j} \big)  \Big) + 2 H \sum_{j = h}^{H-1} \prod_{i = h}^{j-1} \bm P_{i,\pi}  \big (\bm{V}^{(l)}_{j} \circ \bm r_{j}^{(l)} \big) }.
\end{align*}
Here,  
	(i) arises from Jensen's inequality; 
	(ii) holds true due to the Cauchy-Schwarz inequality; (iii) follows from \eqref{eq:var-finite}.
By telescoping summation, one further arrives at 
\begin{align*}
\big|\bm{V}^{(l+1)}_h\big|	&\leq \sqrt{H\bigg[\prod_{i = h}^{H-1} \bm P_{i,\pi}\big(\bm{V}^{(l)}_{H} \circ \bm{V}^{(l)}_{H}\big) - \bm{V}^{(l)}_{h} \circ \bm{V}^{(l)}_{h}\bigg] + 2H\max_{j} \big\|\bm{V}^{(l)}_{j} \big\|_{\infty}\sum_{j = h}^{H-1} \prod_{i = h}^{j-1} \bm P_{i,\pi} \bm r_{j}^{(l)}} \\
	&\overset{\mathrm{(iv)}}{=} \sqrt{H\bigg[\prod_{i = h}^{H-1} \bm P_{i,\pi}\big(\bm{V}^{(l)}_{H} \circ \bm{V}^{(l)}_{H}\big) - \bm{V}^{(l)}_{h} \circ \bm{V}^{(l)}_{h}\bigg] + 2H\max_{j} \big\|\bm{V}^{(l)}_{j} \big\|_{\infty} \bm{V}_h^{(l)} } \\
	&\leq \sqrt{H \prod_{i = h}^{H-1} \bm P_{i,\pi}\big(\bm{V}^{(l)}_{H} \circ \bm{V}^{(l)}_{H}\big)  + 2H \Big( \max_{j} \big\|\bm{V}^{(l)}_{j} \big\|_{\infty} \Big)^2 \bm{1} } \\
	&\overset{\mathrm{(v)}}{\leq} \sqrt{3H}\max_{j} \big\| \bm{V}^{(l)}_{j} \big\|_{\infty} \bm{1}. 
\end{align*}
Here (iv) invokes the relation $\bm{V}^{(l)}_h = \sum_{j = h}^{H-1} \prod_{i = h}^{j-1} \bm P_{i,\pi} \bm r_{j}^{(l)}$ (see \eqref{eq:Vh-l-expression-long});
	and (v) holds true since 
	\[
		\bigg \| \prod_{i = h}^{H-1} \bm P_{i,\pi}\big(\bm{V}^{(l)}_{H} \circ \bm{V}^{(l)}_{H}\big) \bigg\|_{\infty}
		\leq \bigg \| \prod_{i = h}^{H-1} \bm P_{i,\pi} \bigg\|_1  \big\| \big(\bm{V}^{(l)}_{H} \circ \bm{V}^{(l)}_{H}\big) \big\|_{\infty}
		\leq \max_j \big\| \bm{V}^{(l)}_{j} \big\|_{\infty}^2. 
	\]

As a consequence, the above inequality allows one to deduce that
\begin{align*}
	\max_j \big\|\bm{V}_j^{(l+1)}\big\|_{\infty}
	\leq \sqrt{3H}\max_{j} \big\| \bm{V}^{(l)}_{j} \big\|_{\infty} ,
\end{align*}
and therefore, 
\begin{align*}
	\max_j \big\|\bm{V}_j^{(l)}\big\|_{\infty}
	\leq \big( \sqrt{3H} \, \big)^l \max_{j} \big\| \bm{V}^{(0)}_{j} \big\|_{\infty} \leq \big( \sqrt{3H} \, \big)^l H,
\end{align*}
where the last inequality arises from the trivial upper bound $\max_{j} \|\bm{V}^{(0)}_{j}\|_{\infty} \le H$.

\bibliography{bibfileRL}
\bibliographystyle{apalike}

\end{document}